\documentclass[10pt,twocolumn,letterpaper]{article}

\usepackage{ifdraft}

\usepackage[final]{iccv}
\usepackage{times}
\usepackage{epsfig}

\usepackage{appendix}

\usepackage[utf8]{inputenc}             
\usepackage[T1]{fontenc}               	
\usepackage[english]{babel}                   	

\usepackage{amsmath}					
\usepackage{amssymb}					
\usepackage{amsthm}
\usepackage{amsbsy} 
\usepackage{mathtools}
\usepackage{listings}					
\usepackage[dvipsnames]{xcolor}			
\usepackage{colorblind} 
\usepackage{tabularx}
\usepackage{enumitem}					
\usepackage{tikz}
\usepackage{tikz-cd}
\usetikzlibrary{calc}

\definecolor{iccvblue}{rgb}{0.21,0.49,0.74}
\usepackage[pagebackref,breaklinks,colorlinks,allcolors=iccvblue]{hyperref}
\ifdraft{\hypersetup{final}}{}

\def\paperID{5196} 
\def\confName{ICCV}
\def\confYear{2025}

\theoremstyle{plain}
\newtheorem{theorem}{Theorem}[section]
\newtheorem{lemma}[theorem]{Lemma}

\newtheorem{prop}[theorem]{Proposition}
\newtheorem*{crit}{Non-Minimality Criterion}

\theoremstyle{definition}
\newtheorem{definition}[theorem]{Definition}
\newtheorem{example}[theorem]{Example}

\theoremstyle{remark}
\newtheorem{remark}[theorem]{Remark}

\newcommand{\tx}[1]{\text{#1}}

\newcommand{\rank}{\operatorname{rank}}
\newcommand{\im}{\operatorname{im}}

\newcommand{\Stab}{\operatorname{Stab}}

\newcommand{\PP}{\mathbb{P}}

\newcommand{\RR}{\mathbb{R}}
\newcommand{\GG}{\mathbb{G}}

\newcommand{\PGL}[1]{\mathrm{PGL}_{#1}}

\newcommand{\II}{\mathcal{I}} 

\newcommand{\XX}{\mathcal{X}} 
\newcommand{\Xs}{\XX_{p, l, \II}} 

\newcommand{\YY}{\mathcal{Y}} 
\newcommand{\Ys}{\YY_{p, l, \II}} 

\newcommand{\CC}{\mathcal{C}} 
\newcommand{\Cs}{\CC^{m}} 

\newcommand{\FF}{\Phi} 
\newcommand{\Fm}{\FF_{p, l, \II, m}} 

\makeatletter
\newcommand{\dashedrightarrow}[1][2pt]{%
  \settowidth{\@tempdima}{$\rightarrow$}\rightarrow
  \makebox[-\@tempdima]{\hskip-1.2ex\color{white}\rule[0.410ex]{#1}{4pt}}
  \phantom{\rightarrow}
}
\newcommand{\longdashedrightarrow}[1][2pt]{%
  \settowidth{\@tempdima}{$\longrightarrow$}\longrightarrow
  \makebox[-\@tempdima]{\hskip-1.2ex\color{white}\rule[0.410ex]{#1}{5pt}\phantom{\rule{4pt}{5pt}}\rule[0.410ex]{#1}{5pt}}
  \phantom{\longrightarrow}
}
\makeatother

\newcommand\blfootnote[1]{
    \begingroup
    \renewcommand\thefootnote{}\footnote{#1}
    \addtocounter{footnote}{-1}
    \endgroup
}

\title{PLMP -- Point-Line Minimal Problems for Projective SfM}

\author{Kim Kiehn\\
KTH, Department of Mathematics\\
Lindstedtsvägen 25, 10044 Stockholm\\
{\tt\small kiehn@kth.se}
\and
Albin Ahlbäck%
\\
CNRS, LIX (UMR 7161), École Polytechnique\\
1 Rue Honoré d'Estienne d'Orves, 91120 Palaiseau\\
{\tt\small ahlback@lix.polytechnique.fr}
\and
Kathlén Kohn\\
KTH, Department of Mathematics\\
Lindstedtsvägen 25, 10044 Stockholm\\
{\tt\small kathlen@kth.se}
}

\begin{document}
\maketitle

\begin{abstract}
    We completely classify all minimal problems for Structure-from-Motion (SfM) where arrangements of points and lines are fully observed by multiple uncalibrated pinhole cameras. 
    We find 291 minimal problems, 73 of which have unique solutions and can thus be solved linearly.
    Two of the linear problems allow an arbitrary number of views, while all other minimal problems have at most 9 cameras.
    All minimal problems have at most 7 points and at most 12 lines. 
    We compute the number of solutions of each minimal problem, as this gives a measurement of the problem's intrinsic difficulty, and find that these number are relatively low (e.g., when comparing with minimal problems for calibrated cameras).
    Finally, by exploring stabilizer subgroups of subarrangements, we develop a geometric and systematic way to 1) factorize minimal problems into smaller problems, 2) identify minimal problems in underconstrained problems, and 3) formally prove non-minimality.
    \blfootnote{KK and KK were supported by the Wallenberg AI, Autonomous Systems and Software Program (WASP) funded by the Knut and Alice Wallenberg Foundation. AA was supported by an ERC-2023-ADG grant for the ODELIX project (number 101142171).}
    \end{abstract}

\section{Introduction}
\ifdraft{\textcolor{red}{This text will disappear when removing \texttt{draft} from class option.}}{}
Minimal problems are families of algebraic inverse problems that generically have a finite number of solutions. 
They are a key ingredient in RANSAC schemes for reconstruction problems, where one selects a minimal amount of random data samples to obtain subproblems with finitely many solutions \cite{schoenberger2016sfm,ransac}. 
The solution to the original problem is then obtained by finding a consensus among the solutions of many minimal subproblems and optimization refinement. 
Particularly if the data is compromised by outliers, a large number of minimal problems has to be solved in this scheme.
Therefore, it is crucial to work with minimal problems that can be solved efficiently and stably.

A given reconstruction problem has in general many minimal subproblems. 
Thus, it is desirable to identify those minimal problems that can be solved most efficiently. 

In this article, we provide a \emph{complete catalog of minimal problems} (including the most efficient ones) for the following Structure-from-Motion (SfM) problem: Multiple uncalibrated pinhole cameras observe an arrangement of points and lines with prescribed incidences. We assume complete visibility in our minimal problems, meaning that every point and line appears in each view. 
The reconstruction task is to find the camera matrices and the 3D coordinates of the points and lines, up to projective transformations.

We then measure the intrinsic difficulty of solving each minimal problem by its generic number of solutions. This number is known as the \emph{degree} of the minimal problem. 
Of course, we re-discover the famous 7-point problem for 2 views, which has degree 3, and its degenerations.
We also find 2 point-line arrangements where 7 points lie on 3 lines that admit \emph{unique} reconstruction (via linear solvers) from \emph{arbitrarily} many views (these are the 2 right-most arrangements in Table \ref{tab:7Pts}).
We show that there are precisely 285 further minimal problems (illustrated in SM Section \ref{sm:minimal}). 
These have 3-9 views and their degrees range from 1 through 114.
Some problems are well-known  (such as 6 points in 3 views \cite{schaffalitzky2000six,quan1995invariants,heyden1995geometry}, or 9 lines in 3 views \cite{larsson2017efficient}, see also \cite{oskarsson2004minimal} for all incidence-free 3-view problems), but most are new. 
Almost all minimal problems in our catalog have very low degrees.
In particular, 73 of them have degree 1, and so admit unique reconstruction which can be found by linear solvers.

In addition to this catalog of minimal problems, we provide formal proofs for the non-minimality of a given problem, which is in contrast to all previous works where non-minimality is established via strong numerical evidence (see Section \ref{subsec:related}).
The strategy we explain here gives new formal certificates for non-minimality that are applicable for many algebraic inverse problems beyond the setting of this paper.
At the same time, our strategy brings practical advantages: First, it can be used to factorize minimal problems into minimal subproblems of smaller degrees or with less unknowns. Second, in underconstrained problems, our strategy can identify minimal subproblems such that some of the unknown parameters can be reconstructed. 

\medskip
For projective reconstruction from pinhole-camera views with point- and line-matches, our catalog allows to identify the most efficient minimal subproblems. This is particularly interesting when not enough point-matches have been found to use standard 7-point solvers (as discussed e.g. in \cite{fabbri2020trplp}) but line-matches are given. Our catalog suggests many efficient minimal problems involving line-matches (potentially with point incidences as studied e.g. in \cite{ventura2024absolute,fabbri2012camera}).

In addition to the minimal problems found in this article, we expect there to be many more minimal problems with partial visibility, i.e., where some  points/lines do not appear in all views. Since classifying all such minimal problems is a less tractable endeavour, we will address it in future work. 
The complete-visibility assumption in this article is reasonable as complete point-line incidence matches arise for instance from SIFT point features together with their orientation \cite{sift}, simultaneous point and line detections as in \cite{miraldo2018minimal}, or lines constructed from matched affine frames \cite{matas2002local}.

\subsection{Related work} \label{subsec:related}
There is a large zoo of minimal problems described in the computer vision literature
\cite{Nister-5pt-PAMI-2004,
Stewenius-ISPRS-2006,
kukelova2008automatic,
Byrod-ECCV-2008,
DBLP:conf/cvpr/RamalingamS08,
Elqursh-CVPR-2011,
mirzaei2011optimal,
DBLP:conf/eccv/KneipSP12,
Hartley-PAMI-2012,
kuang-astrom-2espc2-13,
Kuang-ICCV-2013,
saurer2015minimal,
ventura2015efficient,
DBLP:conf/eccv/CamposecoSP16,
SalaunMM-ECCV-2016,
larsson2017efficient,
Larsson-Saturated-ICCV-2017,
larsson2017making,
kukelova2017clever,
Larsson-CVPR-2018,
AgarwalLST17,
Barath-CVPR-2017,
Barath-CVPR-2018,
Barath-TIP-2018,
miraldo2018minimal,
Larsson_2019_ICCV,
Bhayani_2020_CVPR,
Mateus_2020_CVPR,
Ding_2020_CVPR}
and new minimal problems are appearing constantly. 
This work follows a systematic approach to enumerate all minimal problems in a given setting, similar to \mbox{\cite{kileel2017minimal,plmp,duff2024pl,hahn2024order}}.

The most relevant work is \cite{plmp}, which classifies all minimal problems for SfM from calibrated cameras observing point-line arrangements under the assumption of complete visibility.
The paper at hand provides the analogous extension to uncalibrated cameras.
In comparison to \cite{plmp}, we find many more minimal problems (291 instead of 30).
The degrees of the uncalibrated minimal problems are overall smaller than for the calibrated ones.
This can be explained by that, while the space of uncalibrated cameras is much larger than the space of calibrated cameras (their dimensions are 11 and 6, respectively), the latter space is nonlinear and thus has a more complicated geometry than the linear (projective) space of uncalibrated cameras. 
In addition, for uncalibrated cameras, we find minimal problems with \emph{arbitrarily} many views, while all minimal problems described in \cite{plmp} have at most 6 calibrated views.

The article \cite{duff2024pl} extends \cite{plmp} to partial visibility, but to make the classification of minimal problems tractable the authors restrict themselves to 3 calibrated views and lines that are incident to at most one point each.
Also \cite{kileel2017minimal} focuses on 3 calibrated views, classifying all minimal problems that reconstruct the relative poses of the cameras from linear constraints on the trifocal tensor.
The article \cite{hahn2024order} enumerates all minimal problems for SfM with rolling-shutter cameras that move with constant speed along a line, do not rotate, and see each world point exactly once.

Finally, a general approach for identifying whether a minimal problem factorizes into simpler minimal subproblems is described in \cite{duff2022galois}. 
However, that approach does not yield a systematic way to actually find such a factorization. 
In contrast, the strategy we describe here to formally prove non-minimality provides -- when applied to minimal problems -- one  concrete factorization strategy.

\subsection{Notation}
We write $\PP^n$ for $n$-dimensional real projective space.
We denote by $\GG_{1,n}$ the set of lines in $\PP^n$, known as the \emph{Grassmannian of lines}.
A \emph{projective variety} in $\PP^n$ is the zero set of a system of homogeneous polynomial equations in $n+1$ variables. Grassmannians are examples of projective varieties.
$\GG_{1,2}$ is isomorphic to $\PP^2$ (the line $\{x \in \PP^2 \mid \sum_{i=1}^3 a_ix_i=0 \} $  corresponds to the point $(a_1:a_2:a_3)$).
The Grassmannian $\GG_{1,3}$ of world lines can be seen as a subvariety of $\PP^5$ via its \emph{Plücker coordinates}.
A variety is called \emph{irreducible} if it cannot be written as the union of two proper subvarieties.
An \emph{algebraic map} between projective varieties is a map that can be locally represented by polynomials.
A \emph{fiber} of a map is the preimage of a singleton.

\section{Main results}
We study \emph{point-line problems (PLPs)}, as defined in \cite{plmp}. A PLP is a tuple $(p,l,\II,m)$ specifying that $m$ cameras observe $p$ points and $l$ lines in $3$-space that satisfy the incidences prescribed by $\II \subseteq \{ 1, \ldots, p \} \times \{ 1, \ldots, l \}$.
Here, $(i,j) \in \II$ means that point $i$ lies on line $j$.
Intersecting lines are modeled by requiring that their intersection is one of the $p$ points in space.
We only consider incidences $\II$ that are realizable (e.g., two distinct lines cannot be incident to the same two distinct points) and complete (i.e., any incidence relation that is automatically implied by the relations present in $\II$ is already contained in $\II$).

In Structure-from-Motion (SfM), we aim to solve the following reconstruction problem associated with a PLP:
Given $m$ images, each showing $p$ points and $l$ lines satisfying the incidences in $\II$, find $m$ cameras and a 3D-arrangement of points and lines such that the pictures taken by the cameras of the arrangement agree with the given images.
A PLP is said to be \emph{minimal} if, given $m$ generic images, the reconstruction problem has finitely many (and at least one) solutions, up to coordinate changes in $3$-space. More formally, we consider \emph{projective cameras}, i.e., full-rank $3 \times 4$ matrices up to scaling:\vspace*{-2mm}
\begin{gather*}
    \CC := \PP (\{ P \in \RR^{3 \times 4} \mid \rank P = 3\}).
\end{gather*}\vspace*{-2mm}
Such a camera observes point-line arrangements in
\begin{align*}
    \Xs
:=
    \{ (X,L) \in (\PP^3)^p \!\times\! \GG_{1,3}^l \mid
    \forall (i,j) \!\in\!\II: X_i \in L_j \}
\tx{.}
\end{align*}
The resulting image is an element in \vspace*{-2mm}
\begin{gather*}
    \Ys :=
    \{ (x,\ell) \in (\PP^2)^p \times \GG_{1,2}^l \mid \forall (i,j) \in \II: x_i \in \ell_j \}.
\end{gather*}
Given $m$ such images, we hope to recover the $m$ cameras and a consistent 3D-arrangement in $\Xs$.
Since acting with the group $\PGL{4}$ of projective linear transformations in $\PP^3$ simultaneously on the cameras and 3D-scene does not affect the resulting images, we can only hope to do reconstruction modulo $\PGL{4}$.
Thus, formally, reconstruction means:
Given an element in $\Ys^m$, compute its preimage under the \emph{joint camera map}\vspace*{-2mm}
\begin{align*}
  \Fm : (\Cs \times \Xs) / \PGL{4} \to \Ys^m
\tx{.}
\end{align*}
\begin{definition} \label{def:minimal}
    A PLP $(p,l,\II,m)$ is \emph{minimal} if the preimage of a generic element in $\Ys^m$ under the joint camera map $\Fm$ is non-empty and finite.
\end{definition}
\begin{remark}
    The image of the joint camera map, known as the \emph{joint image} \cite{handbook}, is generally a proper subset of $\Ys^m$. We use the latter in Definition~\ref{def:minimal} to allow for noise in the image measurements. \hfill $\diamondsuit$
\end{remark}
We observe that the joint camera map $\Fm$ is an algebraic map between algebraic varieties. Hence, over the complex numbers, almost all elements in $\Ys^m$ have the same number of elements in their preimage under $\Fm$. For a minimal PLP, this finite cardinality is known as the \emph{degree} of the minimal problem.
It is an upper bound for the number of real solutions and measures the intrinsic difficulty of the problem (without reformulations).

\begin{theorem}
    \label{thm:main}
    All minimal PLPs for at least two projective cameras are:
    \begin{enumerate}[label={\alph*)}, ref={(\alph*)}]
    \item\label{item:thm:main:a}
      2 cameras viewing one of the point-line arrangements in Table~\ref{tab:7Pts}, plus arbitrarily many additional~lines;
    \item\label{item:thm:main:b}
      arbitrarily many cameras observing one of the 2 right-most point-line arrangements in Table~\ref{tab:7Pts};
    \item\label{item:thm:main:c}
      one of the 285 PLPs  in SM Section~\ref{sm:minimal} (with 3--9 views).
    \end{enumerate}
    Their degrees are given in Table~\ref{tab:7Pts} and SM Section~\ref{sm:minimal}.
\end{theorem}

\begin{table}[!ht]
    \centering
    \input{tab_7pt_v2.tex}%
    \caption{7-point arrangements in minimal PLPs with their degrees.}
    \label{tab:7Pts}
\end{table} 

The PLPs in part \ref{item:thm:main:a} of this theorem are simply degenerations of the classical 7-point minimal problem  (note that arbitrarily many lines can be added since world lines can be uniquely reconstructed from 2 views).
Theorem~\ref{thm:main} \ref{item:thm:main:b} describes 2 infinite families of minimal PLPs with \emph{linear solvers}: They are infinite since the number of views is arbitrary. This is in contrast to the classification of minimal PLPs for calibrated cameras in \cite{plmp} where the maximum numbers of views is 6.
The degrees of the minimal problems reported in Theorem~\ref{thm:main} \ref{item:thm:main:c} are rather low: The largest degree -- by far -- is $114$, and $70$ of the problems have degree 1 which allows their unique reconstruction to be found by linear solvers.
Note that there are additional minimal problems for a single view; see SM Section \ref{sec:singleView}.

To prove Theorem~\ref{thm:main}, we proceed as in \cite{plmp}.
A necessary condition for a PLP to be minimal is that domain and codomain of its joint camera map need to have the same dimension (by the Fiber-Dimension Theorem \ref{thm:fiberDim}).
We say that a PLP satisfying this condition is \emph{balanced}.
Section~\ref{sec:balanced} classifies all balanced PLPs.
Section~\ref{sec:minimal} determines which balanced PLPs are actually minimal.
Section~\ref{sec:degree} computes the degrees of the minimal PLPs.

What is new in comparison to \cite{plmp} and related papers, is that we develop \textbf{new strategies to formally prove non-minimality} of reconstruction problems.
To show minimality of a balanced problem, it is sufficient to compute that the Jacobian matrix of the joint camera map at \emph{one} point in the domain has full rank (see Section~\ref{sec:minimal}).
However, to establish non-minimality, one would have to show that the Jacobian drops rank at \emph{every} point.
For that, previous papers only compute the Jacobian rank at a few random points, which is not a full proof.
Here, we provide a new algebraic proof strategy based on stabilizer groups of subarrangements.

This strategy has practical applications.
They enable a definition of minimality for subproblems that are too small to completely reconstruct the cameras, but instead allow partial reconstruction.
For underconstrained problems, this can help us in reconstructing as many camera parameters as possible.
For minimal problems, this strategy identifies subproblems of smaller (often linear) complexity. These can then be solved first to reduce the algorithmic complexity of the original reconstruction problem by eliminating several variables (see Remark \ref{rem:positiveEffect} and Example \ref{ex:positiveEffect}), thus leading to an effective factorization of minimal problems into smaller subproblems.

\section{Balanced problems} \label{sec:balanced}
We call a PLP $(m,p,l,\II)$ balanced if domain and codomain of the joint camera map $\Fm$ have the same dimension. The dimension of the camera space $\CC$ is $11$.
The dimensions of $\Xs$ and $\Ys$ have been computed in \cite{plmp}.
We consider a generic point-line arrangement in $\Xs$. Some points in the arrangement may be \emph{dependent} on other points, i.e., lie on a line spanned by two other points. Every minimal set of \emph{independent} (also called \emph{free}) points has the same cardinality, which we denote by $p^f$.
The number of dependent points is then $p^d := p - p^f$.
We denote by $l^f$ the number of lines not incident to any point, while $l^a$ counts the lines adjacent to precisely one point.
The lines meeting two or more points are already determined by only the points, so $l = l^f+l^a$.
Then, as in \cite{plmp}, \vspace*{-2mm}
\begin{align}\label{eq:dimXY}
    \begin{split}
        \dim \Xs &= 3p^f + p^d + 4l^4 + 2l^a, \\  \dim \Ys &= 2p^f + p^d + 2l^4 + l^a.
    \end{split}
\end{align}
The group $\PGL{4}$ has dimension 15.
If $m \geq 2$, the generic orbits of the group's action on the camera space $\Cs$ have dimension $15$, and so  \vspace*{-2mm}
\begin{align}\label{eq:dimDomain}
    \dim \left(\Cs \!\times\! \Xs   \right) / \PGL{4} = 11m + \dim \Xs - 15. \!
\end{align}
(Concretely, given two generic cameras, after modding out their scaling and $\PGL{4}$ we can assume that the cameras are
$\begin{bsmallmatrix}
        1 & 0 & 0 & 0 \\
        0 & 1 & 0 & 0 \\
        0 & 0 & 1 & 0
\end{bsmallmatrix}
$
 and 
 $\begin{bsmallmatrix}
           0 &    0 &    0 &    1 \\
        \ast & \ast & \ast & \ast \\
        \ast & \ast &    1 & \ast
    \end{bsmallmatrix}
$,
leaving 7 degrees of freedom.)

Combining \eqref{eq:dimXY} and \eqref{eq:dimDomain}, we see that a PLP $(m,p,l,\II)$ with projective cameras is balanced if and only if  \vspace*{-2mm}
\begin{multline}\label{Dimeq}
    11m-15 = (2m-3)p^f+(m-1)p^d \\ +2(m-2)l^f+(m-2)l^a
\tx{.}
\end{multline}
\begin{theorem}\label{thm:balanced}
    All balanced PLPs for at least two projective cameras  are:
    \begin{enumerate}[label={\alph*)}, ref={(\alph*)}]
        \item\label{item:thm:balanced:m_2-pt_7} for $m=2$, all PLPs with 7 points,
        \item\label{item:thm:balanced:settled_cases}
        for  arbitrary $m$, all PLPs satisfying  $(p^f, p^d,l^f, l^a) \in \lbrace(4,3,0,0),(3,4,0,1),(2,5,1,0),(2,5,0,2) \rbrace$,
        \item\label{item:thm:balanced:3gemge9}
        for $3 \leq m \leq 9$, all PLPs such that $(p^f, p^d,l^f, l^a)$ is one of the 124 tuples  in Table~\ref{tab:balanced}.
    \end{enumerate}
\end{theorem}

\begin{table*}[!ht]
\centering
\renewcommand\tabularxcolumn[1]{m{#1}}
\newcolumntype{Y}{>{\small\centering\arraybackslash}X}
\begin{tabularx}{\linewidth}{|c|Y|}
\hline
  $m$ & \multicolumn{1}{c|}{$(p^f,p^d,l^f,l^a)$}%
  \vphantom{\raisebox{-3pt}{a}\rule{0pt}{10pt}}%
\\\hline
  $9$ & $(0,0,6,0)$
\\\hline
  $8$ & $(1,0,5,0)$, $(1,0,4,2)$, $(1,0,3,4)$, $(1,0,2,6)$, $(1,0,1,8)$, $(1,0,0,10)$
\\\hline
  $7$ & $(2,0,4,0)$, $(2,0,3,2)$, $(2,0,2,4)$, $(2,0,1,6)$, $(2,0,0,8)$
\\\hline $6$ & $(3,0,3,0)$, $(3,0,2,2)$, $(3,0,1,4)$, $(3,0,0,6)$, $(2,1,3,1)$, $(2,1,2,3)$, $(2,1,1,5)$, $(2,1,0,7)$
\\\hline
  $5$ &  $(4,0,2,0)$, $(4,0,1,2)$, $(4,0,0,4)$, $(3,1,2,1)$, $(3,1,1,3)$, $(3,1,0,5)$, $(2,2,3,0)$, $(2,2,2,2)$, $(2,2,1,4)$, $(2,2,0,6)$, $(1,0,5,1)$, $(1,0,4,3)$, $(1,0,3,5)$, $(1,0,2,7)$, $(1,0,1,9)$, $(1,0,0,11)$
\\\hline
  $4$ & $(5,0,1,0)$, $(5,0,0,2)$, $(4,1,1,1)$, $(4,1,0,3)$, $(3,2,2,0)$, $(3,2,1,2)$, $(3,2,0,4)$, $(3,0,3,1)$, $(3,0,2,3)$, $(3,0,1,5)$, $(3,0,0,7)$, $(2,3,2,1)$, $(2,3,1,3)$, $(2,3,0,5)$, $(2,1,4,0)$, $(2,1,3,2)$, $(2,1,2,4)$, $(2,1,1,6)$, $(2,1,0,8)$ $(1,0,6,0)$, $(1,0,5,2)$, $(1,0,4,4)$, $(1,0,3,6)$, $(1,0,2,8)$, $(1,0,1,10)$, $(1,0,0,12)$
\\\hline
  $3$ & $(6,0,0,0)$, $(5,1,0,1)$,$(5,0,1,1)$, $(5,0,0,3)$, $(4,2,1,0)$, $(4,2,0,2)$, $(4,1,2,0)$, $(4,1,1,2)$, $(4,1,0,4)$, $(4,0,3,0)$, $(4,0,2,2)$, $(4,0,1,4)$, $(4,0,0,6)$, $(3,3,1,1)$, $(3,3,0,3)$, $(3,2,2,1)$, $(3,2,1,3)$, $(3,2,0,5)$, $(3,1,3,1)$, $(3,1,2,3)$, $(3,1,1,5)$, $(3,1,0,7)$, $(3,0,4,1)$, $(3,0,3,3)$, $(3,0,2,5)$, $(3,0,1,7)$, $(3,0,0,9)$, $(2,6,0,0)$, $(2,4,2,0)$, $(2,4,1,2)$, $(2,4,0,4)$, $(2,3,3,0)$, $(2,3,2,2)$, $(2,3,1,4)$, $(2,3,0,6)$, $(2,2,4,0)$, $(2,2,3,2)$, $(2,2,2,4)$, $(2,2,1,6)$, $(2,2,0,8)$, $(2,1,5,0)$, $(2,1,4,2)$, $(2,1,3,4)$, $(2,1,2,6)$, $(2,1,1,8)$, $(2,1,0,10)$, $(2,0,6,0)$, $(2,0,5,2)$, $(2,0,4,4)$, $(2,0,3,6)$, $(2,0,2,8)$, $(2,0,1,10)$, $(2,0,0,12)$, $(1,0,7,1)$, $(1,0,6,3)$, $(1,0,5,5)$, $(1,0,4,7)$, $(1,0,3,9)$, $(1,0,2,11)$, $(1,0,1,13)$, $(1,0,0,15)$, $(0,0,9,0)$
\\\hline
\end{tabularx}%
\caption{The classes of balanced PLPs from Theorem~\ref{thm:balanced}~\ref{item:thm:balanced:3gemge9}.  \vspace*{-2mm}}%
\label{tab:balanced}%
\end{table*}

We note that, in general, a tuple  $(m, p^f, p^d,l^f, l^a)$ does not uniquely determine a PLP.
For instance, $(m, p^f, p^d,l^f, l^a)=(3,5,1,0,1)$ encodes two non-equivalent PLPs: Both have 6 points, out of which 3 are collinear. The PLPs are distinguished by whether the line is adjacent to one of the collinear points or to one of the other 3 points. For more details on distinguishing PLPs with the same tuple $(m, p^f, p^d,l^f, l^a)$, see Section~\ref{sec:minimal}.

We split the proof of Theorem~\ref{thm:balanced} into 2 key lemmas.

\begin{lemma}\label{sevpoint}
    The only balanced PLPs with $m\geq 3$ and $p \geq 7$ are the infinite families in Theorem~\ref{thm:balanced} $b)$ 
    and the PLP consisting of $8$ collinear points and $3$ cameras.
\end{lemma}
\begin{proof}
Since $p \geq 7$, we have $p^d\geq 7-p^f$. Inserting this and  $l^f,l^a \geq 0$ into  \eqref{Dimeq}, we obtain
$11m-15\geq (2m-3)p^f+(m-1)(7-p^f)$, i.e. $2(p^f-4)\geq m(p^f-4)$.
This inequality shows that  $p^f>4$ would imply $2\geq m$, contradicting our assumption of $m\geq 3$. Thus, $p^f \leq 4$.
Since among $7$ or more points we must have at least two independent points,  we conclude $p^f\in \lbrace 2,3,4  \rbrace$.

If \underline{$p^f=4$}, then \eqref{Dimeq}  simplifies to  $
3(m-1)=(m-1)p^d+2(m-2)l^f+(m-2)l^a.$
Since $p^d = p - p^f \geq 3,m\geq 3$ and $l^a,l^f\geq 0$, the only solution of this equation is $p^d=3, l^f=l^a=0$, which holds for all $m$.

If \underline{$p^f=3$}, then  \eqref{Dimeq} becomes \vspace*{-2mm}
\begin{gather*}
    5m-6=(m-1)p^d+2(m-2)l^f+(m-2)l^a.
\end{gather*}
This equation shows that $p^d \geq 5$ is not possible, as it would require $2(m-2)l^f+(m-2)l^a\leq -1$.
Thus, $p^d \leq 4$, and due to $p \geq 7$ we have in fact $p^d=4$. Now, the equation reduces further to $
    m-2= 2(m-2)l^f+(m-2)l^a.$
Dividing by $m-2$, we obtain $1=2l^f+l^a$, which implies $l^f=0$ and $l^a=1$. This solution is valid for all $m$.

Inserting \underline{$p^f=2$} into \eqref{Dimeq},
 we observe similarly to the above that $p^d\geq7$ leads to a contradiction. Therefore, $p^d\in \lbrace 5,6\rbrace$.
For $p^d=5$, \eqref{Dimeq} reduces to $
2(m-2)=2(m-2)l^f+(m-2)l^a$, which has two solutions $(l^f,l^a)\in \lbrace(1,0),(0,2)\rbrace$, and both are independent of $m$. Finally,
$p^d=6$ simplifies  \eqref{Dimeq} to $
    m-3=2(m-2)l^f+(m-2)l^a$, which has the unique solution $m=3$ and $ l^f=l^a=0$.
\end{proof}

\begin{lemma}\label{cambound}
    A balanced PLP with  $m\geq 3$ and $p \leq 6$ satisfies $m\leq 9$.
\end{lemma}
\begin{proof}
    We rewrite  \eqref{Dimeq} as $
        11(m-2)+7=
        (m-2)(2p^f+p^d+2\ell^f+\ell^a) +p^f+p^d.$
    Taking this  modulo $m-2$ yields  \vspace*{-2mm}\begin{gather}\label{Dimmod}
        7 \equiv p\mod  (m-2).
    \end{gather}
    So $m\geq 10$  yields $p\geq 7$, contradicting our assumption.
\end{proof}
\begin{proof} [Proof of Theorem~\ref{thm:balanced}]
    We start by considering the case $m=2$. Then,  \eqref{Dimeq} reduces to $7=p^f+p^d$ and there are no restrictions on the lines in the arrangements.

    Hence, we assume from now on $m\geq 3$. If $p\geq 7$, we  invoke Lemma~\ref{sevpoint} as a classification for all balanced PLPs in this case.
    Thus, we further assume $p\leq 6$. Then, we have  $3\leq m\leq 9$ by Lemma~\ref{cambound}. In this range, \eqref{Dimeq} has  finitely many solutions.
    In fact, it is straightforward to enumerate all solutions\footnote{\label{code} see code (lines 122-175) at \url{https://github.com/albinahlback/minimal_plps}} satisfying the conditions $p \leq 6$, $(p^d \geq 1 \Rightarrow p^f \geq 2)$,  $(l^a \geq 1 \Rightarrow p^f \geq 1)$: There are 123 solutions, listed in Table~\ref{tab:balanced} together with the single solution $(m,p^f,p^d,l^d,l^a)=(3,2,6,0,0)$ from Lemma~\ref{sevpoint}.
\end{proof}

\section{Checking minimality} \label{sec:minimal}

Here, we determine which of the balanced PLPs are actually minimal.
To reduce the number of PLPs to consider, we invoke the following standard homography-based certificates for non-minimality.
\begin{lemma}\label{lem:homography}
A minimal PLP with $\geq 2$ views cannot contain
\begin{itemize}
\item 
    4 or more collinear points, or
\item
    3 free points and 3 or more dependent points in the plane spanned by the free points.
\end{itemize}
\end{lemma}
\begin{proof}
    Projecting collinear world points yields collinear points in each view and, when comparing 2 views, it induces a homography between the line spanned by the collinear points on the first view and that line on the second view.
    Since a homography between projective lines is uniquely determined by 3 points on each line, 4 or more collinear image points generally do not permit a reconstruction (for that, they would need to have the same cross ratios).
    Analogously, for 2 views showing 3 free points $x_1,x_2,x_3$ and 2 dependent points on the lines through $x_1,x_2$ and  $x_1,x_3$, respectively, the homography between the 2 image planes is uniquely determined. Thus, further dependent points generically forbid a reconstruction. We give a more formal proof using our stabilizer techniques in the SM.
\end{proof}  \vspace*{-1mm}
This lemma directly excludes the balanced PLPs from Theorem~\ref{thm:balanced}~\ref{item:thm:balanced:settled_cases} with $(p^f,p^d)\!\in\! \{ (3,4), \, (2,5) \}$ to be minimal.
In fact, the only 7-point configurations not excluded by the lemma are shown in Table~\ref{tab:7Pts} (for a proof, see Lemma~\ref{lem:minimalPointArrangement}).
Theorem~\ref{thm:main} \ref{item:thm:main:a} and \ref{item:thm:main:b} say that all of those remaining configurations indeed yield  minimal PLPs.
We prove this using the Jacobian check, which we explain~next.

\subsection{Fiber-dimension theorem \& Jacobian check} \label{subsec:Jacobian}
First, we formally explain the relation between minimal and balanced PLPs.
For that, we  use  the fiber-dimension theorem \cite[Theorem 1.25]{ShafarevichHirsch94}), which extends the rank-nullity theorem from linear algebra to polynomial maps.
\begin{theorem}[Fiber-Dimension] \label{thm:fiberDim}
    Let $\Phi: \mathcal{Z} \to \mathcal{Y}$ be an algebraic map of irreducible varieties. Over~$\mathbb{C}$, almost  all $z \in \mathcal{Z}$ satisfy
    \mbox{$\dim \mathcal{Z} = \dim \im(\Phi) + \dim \Phi^{-1} (\Phi(z))$}.
\end{theorem}

\begin{lemma} \label{lem:balancedVsMinimal}
    Every minimal PLP is balanced. Moreover, a balanced PLP $(p,l,\II,m)$ is minimal if and only if
    $\dim \im (\Fm) = \dim \Ys^{m}$.
\end{lemma}
\begin{proof}
We shortly write $\mathcal{Z}$ and $\mathcal{Y}$ for domain and codomain of the joint camera map $\Phi := \Fm$.
Both varieties $\mathcal{Z}$ and $\mathcal{Y}$ are irreducible.
The PLP being minimal means that:\\
1) almost all points in $\mathcal{Y}$ are in $\im(\Phi)$, which is equivalent to $\dim \im(\Phi) = \dim \mathcal{Y}$ due to the irreducibility of $\mathcal{Y}$, and\\
2) almost all points $y$ in the image have a finite fiber $\Phi^{-1}(y)$, which is equivalent to $\dim \Phi^{-1} (\Phi(z)) = 0$ for almost all $z \in \mathcal{Z}$.
\quad
Plugging both conditions of minimality into the fiber-dimension theorem, we obtain $\dim \mathcal{Z} = \dim \mathcal{Y}$, i.e., balancedness.

For the second part of the claim, we assume that the given PLP is balanced, i.e., $\dim \mathcal{Z} = \dim \mathcal{Y}$.
If furthermore $\dim \im(\Phi) = \dim \mathcal{Y}$ holds (which is the first condition of minimality), then the fiber-dimension theorem implies that $\dim \Phi^{-1} (\Phi(z)) = 0$ for almost all $z \in \mathcal{Z}$ (which is the second condition of minimality).
\end{proof}

Hence, to prove our Main Theorem~\ref{thm:main}, it is enough to compute the dimension of the image of the joint camera map for all balanced PLPs  found in Theorem~\ref{thm:balanced}.
This can be done via Jacobian matrices \cite[Chapter 6, Lemma 2.4]{ShafarevichHirsch94}.

\begin{prop}[Jacobian check]
    Let $\Phi: \mathcal{Z} \to \mathcal{Y}$ be an algebraic map of irreducible varieties, and let $J_z \Phi$ be its Jacobian matrix at $z \in Z$. Then,
    $ \rank J_z \Phi \leq \dim \im(\Phi)$, and equality holds for  almost all $z \in \mathcal{Z}$.
\end{prop}
\begin{lemma} \label{lem:jacobianMinimal}
    A balanced PLP $(p,l,\II,m)$ is minimal if and only if
    there exists a point in the domain of  $\Fm$ where 
    its Jacobian matrix is invertible.
\end{lemma}
\begin{proof}
    For a minimal PLP, Lemma~\ref{lem:balancedVsMinimal} and the Jacobian check together state that the rank of the Jacobian matrix equals $ \dim \Ys^{m}$ at almost every point in the domain. Those matrices are invertible due to the balancedness assumption.
    In fact, if one finds a single point
    with an invertible Jacobian matrix, then balancedness and the Jacobian check yield $ \dim \Ys^{m} \leq \dim \im(\Fm)$. This has to be an equality since the latter variety is contained in the former, and so Lemma~\ref{lem:balancedVsMinimal} shows minimality.
\end{proof}

\subsection{Proof of main theorem}
\label{subsec:mainProof}
Here, we explain the proof of Theorem~\ref{thm:main}, except for the degree computations that are discussed in Section~\ref{sec:degree}.
We consider the list of balanced PLPs described in Theorem~\ref{thm:balanced} and determine their minimality as follows.

\smallskip \noindent \textbf{a)}
As mentioned above, Lemma~\ref{lem:homography} excludes all balanced  PLPs from Theorem~\ref{thm:balanced} \ref{item:thm:balanced:m_2-pt_7}, except those in Table~\ref{tab:7Pts}.
For any camera pair, lines in their 2 views yield a unique world line, so  minimality (and degree) of a PLP with 2 cameras is not affected by the presence of additional lines.
For each of the 2-view PLPs in Table~\ref{tab:7Pts}, we verify computationally (see code\footref{code} lines 1699-1727) that  some Jacobian matrix of the joint camera map is invertible, showing minimality by Lemma~\ref{lem:jacobianMinimal}.

\begin{remark}
    To speed up computations, we  compute the rank of the Jacobian matrices over finite fields of large prime order. Indeed, if the Jacobian matrix at a rational point is invertible over a finite field, then it is also invertible over $\mathbb{Q}$; see \cite[Section 13.2]{duff2024pl}. \hfill $\diamondsuit$
\end{remark}

\smallskip \noindent \textbf{b)}
Lemma~\ref{lem:homography} eliminates all balanced PLPs from Theorem~\ref{thm:balanced} \ref{item:thm:balanced:settled_cases} except the right-most two in Table~\ref{tab:7Pts}. We can prove their minimality for $m \geq 2$ cameras directly.
 Generically, the four free world points form a basis. Depending on whether the configuration corresponds to the last or penultimate entry of Table~\ref{tab:7Pts}, there is a transformation in $ \PGL{4}$ that maps the free points to the standard unit vectors $e_1, \ldots, e_4$ and the dependent points to either $e_1+e_2, e_2+e_3,e_3+e_4$, or $e_1+e_2,e_1+e_3,e_1+e_4$. This transformatiom is unique, thus the joint camera map $(\Cs \times \XX) / \PGL{4} \to \YY$ can be seen as the linear map $(P_k)\mapsto ((P_k\cdot e_i), (P_k\cdot e_i+e_j))$, turning the reconstruction problem into a linear one. Furthermore, we see that the matrix defining this linear map has full rank implying the existence of a unique solution.

\smallskip \noindent \textbf{c)}
Recall that the tuples in Table~\ref{tab:balanced} generally encode several balanced PLPs.
Hence, as a first step, we need to distinguish the different PLPs encoded by a tuple $(p^f, p^d, l^f, l^a)$.
Lemma~\ref{lem:homography} disqualifies many PLPs from being minimal, and we only have to distinguish the remaining ones.
The $p^f + p^d$ many points have to be arranged as shown in Table~\ref{tab:atMost6Pts} (see Lemma~\ref{lem:minimalPointArrangement}).
Then, there are generally several (non-equivalent) ways how the $l^a$ adjacent lines can be attached to the points.
This process turns the 124 tuples from Table~\ref{tab:balanced} into 434 balanced PLPs that are candidates for being minimal (see code\footref{code} lines 394-548).

For each of the 434, we pick a random point in the domain of the joint camera map and check whether the Jacobian is invertible over a large finite field (see code\footref{code} lines 1026-1050). After repeating this for several random points, we verified that 285 of the balanced PLPs are indeed minimal by Lemma~\ref{lem:jacobianMinimal}.
This strongly suggests that the remaining 149 balanced PLPs are not minimal,
 but it does \emph{not} provide a formal proof for non-minimality.
Therefore, in contrast to previous works on minimal-problem classification \cite{plmp,duff2024pl} which have been satisfied with this computational evidence, we provide formal strategies in the following.

\subsection{Non-minimality criteria}
\label{subsec:nonMinimal}
In a square system of polynomial equations (i.e.,  same number of unknows as equations) that generically does not have any solution, one can often identify an overconstrained subsystem to show that the original  system does indeed not have any solution generically.
In our setting, this means that in a given balanced PLP $(p,l,\II,m)$ that is non-minimal, we hope to find a subarrangement of its point-line arrangement  giving rise to a PLP $(p',l',\II',m)$ with fewer unknowns and fewer constraints that overconstrain \emph{some} of the camera parameters.
To find out which exact camera parameters those are,  we normalize -- by modding out the $\PGL{4}$-action -- as many of the unknowns in the points and lines in the subarrangement as possible.
\emph{If} then all normalized subarrangements $A \in \XX_{p',l',\II'}$ have the same stabilizer subgroup $\Stab(A) := \{ H \in \PGL{4} \mid H \cdot A = A \} $,  we can mod out this stabilizer from \emph{each} of the cameras \emph{independently}. 
We will see that the remaining camera parameters are often those that are overconstrained in the subproblem. 
Formally, we consider subproblems satisfying the following:

\begin{definition}\label{def:reduced}
We call a pair of varieties $(\CC',\XX')$ a \emph{reduced problem} for the PLP $(p',l',\II',m)$ if:
\begin{enumerate}
\item $\XX'\subset \XX_{p',l',\II'}$  and $\CC'\subset \PP(\RR^{3 \times 4})$ are  subvarieties,
\item for all $ A_1,A_2\in \XX'$, we have $\Stab(A_1)=\Stab(A_2)$,
\\
(we denote this group by $\Stab(\XX')$)
\item $\dim (\PGL{4}\cdot \XX') = \dim (\XX_{p',l',\II'})$, and
\item $\dim (\CC'\cdot \Stab(\XX')) = 11$.
\end{enumerate}
    The \emph{reduced joint camera map} associated to the reduced problem $(\CC',\XX')$ is defined by \vspace*{-2mm}
    \begin{align*}
      \FF^{\mathrm{red}}_{p',l',\II',m}:  (\CC')^m\times \XX' & \,\dashrightarrow
       \YY_{p',l',\II'}^m\\
        (P_1,...,P_m,A) &\,\mapsto (P_1\cdot A,...,P_m\cdot A).
    \end{align*}
\end{definition}

\begin{example}\label{ex:running1}
    We consider the balanced PLP where 8 cameras observe 2 free lines and 6 lines intersecting at a single point. We will show in Example \ref{ex:running2} that it is not minimal.
    For that, we consider its subarrangements obtained by removing the 2 free lines, i.e.,
    $(p',l',\II',m) = (1,6,\{ 1 \} \times \{ 1,\ldots,6 \},8)$.
    To define $\XX'$, we let $e_i \in \mathbb{P}^3$ be the standard unit vectors, set $e := e_2+e_3+e_4$, and let $\overline{Q_1Q_2}$ be the line spanned by the points $Q_1,Q_2 \in \PP^3$.
   Then, $\XX' := \{ (e_1, \overline{e_1 e_2}, \overline{e_1 e_3}, \overline{e_1 e_4}, \overline{e_1 e}, \overline{e_1 Q_1}, \overline{e_1 Q_2}) \mid Q_i \!\neq\! e_1 \} $.

    First, we compute the stabilizer of elements in $\XX'$. Let $H \in \PGL{4}$ be in the stabilizer. Fixing the projective point $e_1$ means $He_1=\lambda_1e_1$ for some scalar $\lambda_1\neq 0$. Similarly, fixing the first three lines implies that $e_i$ (for $i\geq 2$) gets mapped to  a linear combination of $e_1$ and $e_i$. Thus, 
    \vspace*{-2mm}
    \begin{gather*}
    H=\left[\begin{smallmatrix}
        \lambda_1 & \lambda_2 &\lambda_3 & \lambda_4\\
    0 &\mu_1 &0&0\\
    0&0&\mu_2&0\\
    0&0&0&\mu_3
    \end{smallmatrix} \right].
\end{gather*} 
Fixing the fourth line enforces that $He = \lambda e_1 + \mu e$. Combining this with $He=He_1+He_2+He_3$, we see that  $\mu = \mu_1=\mu_2=\mu_3$. 
Now a straightforward computation shows that the line passing through $e_1$ and an arbitrary point $Q$ is already fixed by any such $H$. Thus, all elements in $\XX'$ have the same stabilizer subgroup:\vspace*{-2mm}
\begin{gather*}
    \Stab(\XX')=\left\lbrace \begin{bsmallmatrix}
        \lambda_1 & \lambda_2 &\lambda_3 & \lambda_4\\
    0 &\mu &0&0\\
    0&0&\mu&0\\
    0&0&0&\mu
    \end{bsmallmatrix} \right\rbrace \subset \PGL{4}.
\end{gather*}

Second, we verify that the orbit $\PGL{4} \cdot \XX'$ has the same dimension as $\XX_{p',l',\II'}$. For this it suffices to show that almost every arrangement  $A \in \XX_{p',l',\II'}$ can be brought into the form in $\XX'$ by a projective transformation. 
The point in the arrangement $A$ can be mapped linearly to $e_1$. Then, we are left with an arrangement of lines attached to $e_1$.
These lines intersect the plane spanned by $e_2,e_3,e_4$. Generically, the intersections of the first four lines lie in general position in this plane. Now we can choose a homography of this plane to map the intersections to $e_2,e_3,e_4$ and $e$, showing that $A$ can indeed be mapped linearly to an element of~$\XX'$.

Lastly, we define \vspace*{-2mm}
\begin{gather*}
  \CC'
:=
  \left\{
    \begin{bsmallmatrix}
      1 & 1 & 1 &1\\
      c_1 & c_2 & c_3 &c_4\\
      c_5 & c_6 & c_7 &1
    \end{bsmallmatrix}
  \mathrel{\Big\vert}
    c_i \in \mathbb{R}
  \right\}
\subset
  \mathbb{P}(\RR^{3 \times 4})
\end{gather*}
and show that the dimension of the orbit $\CC'\cdot\Stab(\XX')$ is $11$. We compute how elements in the orbit look like:
\vspace*{-2mm}
\begin{multline}\label{ReintroStab}
  \begin{bsmallmatrix}
    1 & 1 & 1 &1\\
    c_1 & c_2 & c_3 &c_4\\
    c_5 & c_6 & c_7 &1
  \end{bsmallmatrix}
  \begin{bsmallmatrix}
    \lambda_1 & \lambda_2 &\lambda_3 & \lambda_4\\
    0 &1 &0&0\\
    0&0&1&0\\
    0&0&0&1
  \end{bsmallmatrix}
\\
=
  \begin{bsmallmatrix}
    \lambda_1 & \lambda_2+1 & \lambda_3+1 &\lambda_4+1\\
    \lambda_1c_1 & \lambda_2c_1+c_2 & \lambda_3c_1+c_3 &\lambda_4c_1+c_4\\
    \lambda_1c_5 & \lambda_2c_5+c_6 & \lambda_3c_5+c_7 &\lambda_4c_5+1
  \end{bsmallmatrix}.
\end{multline}
Indeed, we see that the latter family of matrices has $11$ degrees of freedom. 
Thus, we have shown that $(\mathcal{C}',\mathcal{X}')$ is a reduced problem for the PLP $(p',l',\mathcal{I}',m)$. 
\hfill $\diamondsuit$
\end{example}

   The property $\dim (\CC'\cdot\Stab(\XX'))=11$ shows that the reduced joint camera map observes as much as the original joint camera map. This is because it implies that, for a generic camera $P$, there is a matrix $H\in \Stab(\mathcal{X}')$ and a reduced camera $P'\in \mathcal{C}'$ with $P=P'H$. For any  arrangement $A\in \mathcal{X}'$, we then have by definition of the stabilizer 
   \vspace*{-2mm}
   \begin{gather} \label{eq:stabCameraImpact}
       P\cdot A=P'\cdot (H\cdot A)=P'\cdot A.
   \end{gather}
   Therefore, the images of the reduced camera agree with the images of the full camera. 
   This allows us to mod out the stabilizer from \emph{every} camera, which is in contrast to the standard approach of modding out the global $\mathrm{PGL}_4$-action by fixing some coordinates of the first two cameras only.
   (Also note that, by property (iii) in Definition \ref{def:reduced}, almost any point-line subarrangement can be transformed into reduced form, and so it is not too restrictive to only consider arrangements in $\mathcal{X'}$ in \eqref{eq:stabCameraImpact}.)
   
   In particular, we get as a consequence that we can reconstruct the subproblem with the full camera model if and only if we can reconstruct it when using the reduced cameras. Hence, if the original PLP were minimal, the dimension of the domain of the reduced joint camera map must be greater or equal than the dimension of its codomain. Formally:

\begin{prop}\label{Schlüssellemma}
Assume that the joint camera map has full-dimensional image:     $\dim \im (\Fm) = \dim \Ys^{m}$.
Then, for any sub-PLP $(p'\!,l',\!\mathcal{I}',\!m)$ with associated reduced problem $(\CC',\XX')$,  the reduced joint camera map 
    has full-dimensional  \mbox{image: 
   $\dim \im(\Phi^{\mathrm{red}}_{p',l',\II',m}) = \dim \YY_{p',l',\II'}^m$.}

    In particular, the following inequality holds 
    \vspace*{-2mm}
\begin{gather}
\label{prop:Ineq}
        m\cdot \dim\left(\CC'\right)+\dim(\XX')\geq m \cdot \dim\left(\YY_{p',l',\II'}\right).
    \end{gather}
\end{prop}

We give a more formal proof (than the argument above) of this proposition in the SM. 
We now illustrate at the previously discussed example how the proposition can be used as a criterium for non-minimality.

\begin{example}
    \label{ex:running2}
    We consider the balanced PLP from Example \ref{ex:running1}.
    If it were minimal, the image of its joint camera map were full-dimensional by Lemma \ref{lem:balancedVsMinimal}.
    But then Proposition \ref{Schlüssellemma} would tell us that the reduced subproblem $(\CC',\XX')$ described in Example \ref{ex:running1} would need to satisfy the inequality in \eqref{prop:Ineq}. 
    However, this is absurd, since $m=8$,
     $\dim(\mathcal{C}')=7$, $ \dim(\mathcal{X}')=4$, and $\dim(\mathcal{Y}_{p',\ell',\mathcal{I}'})= 8$.
     Thus, the balanced PLP we started from is not minimal.
     \hfill $\diamondsuit$
\end{example}

To summarize, we have obtained the following strategy:

\begin{crit}
    In a balanced PLP, identify a reduced subproblem violating the inequality \eqref{prop:Ineq}.
\end{crit}

In other words, in a square system of polynomial equations coming from a non-minimal problem, this criterion gives one geometric way to identify an overconstrained subsystem.
This criterion gives a formal proof for the non-minimality of 130 out of the 149 balanced PLPs that were computationally found to be non-minimal at the end of Section \ref{subsec:mainProof}. 
For the remaining 19 problems, we study their equations explicitely after eliminating variables. We explain this in details in the SM Sections \ref{appendix:nonMinimal} and \ref{appendix:nonMinimal19}.

\begin{remark}  \label{rem:positiveEffect}
The ideas described above also have  positive effects on 1) underconstrained problems and 2) minimal problems. 
First, applying the strategy above to underconstrained problems can systematically find minimal subproblems that allow us to reconstruct some of the unknown parameters.
Similarly, in minimal problems, our method can systematically find easier (often linear) subproblems that can be solved as a first step, potentially leading to drastic computational improvements as in SM Section \ref{sm:factorisations}. To do so, we consider a minimal PLP with a reduced subproblem that is  balanced, i.e., where the inequality \eqref{prop:Ineq} is an equality.
 Proposition \ref{Schlüssellemma} implies that also the subproblem is minimal. 

After having solved the subproblem, one needs to bring the stabilizer subgroups back into use, as it is done in the above example in \eqref{ReintroStab}.
Note that now the entries of the matrices in $\CC'$ (in the example above these are $c_1,...,c_7$) are no longer unknowns but explicit numbers that are the solutions of the reduced reconstruction problem. 
Also, in general, we do not need to bring the stabilizer subgroups in all cameras back into use since we did not mod out all of the $\PGL{4}$ action but only the part needed to normalize the subarrangement to the form required by $\XX'$ (see Example \ref{ex:positiveEffect}). That way, we re-introduce the missing number of unknowns in the cameras to solve the original reconstruction problem. \hfill $\diamondsuit$
\end{remark}

\begin{example} \label{ex:positiveEffect}
    We consider the minimal PLP where 3 cameras observe 4 free lines and 7 lines intersecting at a single point (see SM Section~\ref{sm:minimal}). It has 51 unknowns (after modding out $\PGL{4}$). We study its subarrangements obtained by removing the 4 free lines, i.e. $(p',l',\II',m)=(1,7,\lbrace 1\rbrace\times \lbrace 1,\dots,7\rbrace,3)$. We use the same normalization as in Example \ref{ex:running1}, that is, \linebreak[4] $\XX' := \{ (e_1, \overline{e_1 e_2}, \overline{e_1 e_3}, \overline{e_1 e_4}, \overline{e_1 e}, \overline{e_1 Q_1}, \overline{e_1 Q_2}, \overline{e_1 Q_3}) \mid Q_i \neq e_1\rbrace$, with the only difference being an additional third non-normalized line. The stabilizer subgroup $\Stab(\XX')$ and the reduced cameras $\CC'$ are also the same as in Example \ref{ex:running1}. 
    The reduced subproblem has only 27 unknowns. 
    
    The methods from the next chapter yield a degree of $6$ for the original minimal PLP. Performing the same computation for the reduced subproblem (code\footref{code} \texttt{example\_4-12.jl}), 
    we obtain a degree of $2$. Thus, we factorize our degree-6 PLP into two minimal problems of degrees 2 and 3.
    
    We now describe the degree-3 problem explicitly. From the reduced subproblem, we obtain 2 tuples of matrices and adjacent lines, where each matrix is an element of $\CC'$ in Example \ref{ex:running1}. Fixing for now one of these 2 solutions, we bring the stabilizers back into use by computing (for $i=1,2,3$) 
    \begin{gather}\label{eq:reintroVars}
         \begin{bsmallmatrix}
        1 & 1 & 1 &1\\
        c_1^i & c_2^i & c_3^i &c_4^i\\
        c_5^i & c_6^i & c_7^i &1
    \end{bsmallmatrix}
    \begin{bsmallmatrix}
        \lambda_1^i & \lambda_2^i &\lambda_3^i & \lambda_4^i\\
    0 &1 &0&0\\
    0&0&1&0\\
    0&0&0&1
    \end{bsmallmatrix},
    \end{gather}
    where the $c_j^i$ are numbers and no longer variables but the $\lambda_j^i$ are unknowns instead. To mod out the rest of the $\PGL{4}$-action, we  set the first stabilizer matrix to the identity, i.e., $\lambda_1^1=1$ and $\lambda_j^1=0$ for $j=2,3,4$. Now the problem with these cameras and the 4 free lines we omitted at first has in total 24 unknowns.
    As anticipated before, it has 3 solutions, 
    for both of the solutions to the minimal reduced subproblem. Hence, we have factorized  the original degree-6 PLP into
    2 problems with less unknowns and smaller degrees.
    \hfill $\diamondsuit$
\end{example}

\section{Degree computations} \label{sec:degree}
We computed the degrees of all minimal problems described in Theorem \ref{thm:main} using the  monodromy technique based on numerical homotopy continuation. Our implementation is similar to previous works \cite{plmp,duff2024pl}. 
It uses \texttt{Oscar.jl} \cite{OSCAR} for the computer algebra system and computations of Gröbner bases, and \texttt{HomotopyContinuation.jl} \cite{HomotopyContinuation.jl} for monodromy computations; see code\footref{code} lines 2102-2195 and 1610-1659.

Monodromy roughly works as follows: 
A minimal problem is expressed as a parametrized system of polynomial equations that, for generic parameters, has finitely many solutions. 
The camera entries and coordinates of the point-line arrangements in 3-space are the unknowns of the system, while the resulting images are its parameters. 
To perform monodromy, 
one first picks random cameras and a random point-line arrangement. 
Then, one computes the resulting images.
Now, one moves in a loop through the parameter space, starting and ending at the images computed in the previous step.
While doing that, the original solution (i.e., the chosen cameras and point-line arrangement) is tracked using homotopy continuation. 
When the loop closes, one is expected to find a new solution. 
Continuing this process, one finds -- in theory -- all solutions \cite{monodromy}.

For  problems of low degree (say, $<1000$), monodromy is also in practice expected to find all solutions. Thus, we are confident about the degrees reported in SM Section \ref{sm:minimal}.
However, a priori, monodromy only provides a lower bound on the degree of a minimal problem. 
Therefore, we verified our monodromy results with symbolic Gröbner basis methods for all minimal problems with up to 4 views.
For that, we reduced the number of variables in our equation systems by directly encoding the constraints that corresponding image lines impose on cameras. That way, the only unknowns are world points and cameras. This approach has been  explained in \cite[Sec. 6]{plmp} (which eliminates both world points and lines, while we only \mbox{eliminate the latter}).

\section{Conclusion}
We described all minimal problems and computed their algebraic degrees for reconstructing point-line arrangements and uncalibrated cameras from images in the setting of complete multi-view visibility. We discovered many new minimal problems of a small degree with a small number of image features, calling for the implementation of efficient and practical solvers.
By studying stabilizer subgroups of subarrangements in non-minimal problems, we describe a first systematic strategy to formally disprove minimality. 
The same strategy can be applied to, on the one hand, minimal problems in order to factorize them into smaller subproblems or, on the other hand, underconstrained problems in order to identify reconstructable parameters.

{
    \small
    \bibliographystyle{ieeenat_fullname}
    \bibliography{biblio.bib}

\begin{thebibliography}{50}
\providecommand{\natexlab}[1]{#1}
\providecommand{\url}[1]{\texttt{#1}}
\expandafter\ifx\csname urlstyle\endcsname\relax
  \providecommand{\doi}[1]{doi: #1}\else
  \providecommand{\doi}{doi: \begingroup \urlstyle{rm}\Url}\fi

\bibitem[Agarwal et~al.(2017)Agarwal, Lee, Sturmfels, and Thomas]{AgarwalLST17}
Sameer Agarwal, Hon{-}Leung Lee, Bernd Sturmfels, and Rekha~R. Thomas.
\newblock On the existence of epipolar matrices.
\newblock \emph{International Journal of Computer Vision}, 121\penalty0 (3):\penalty0 403--415, 2017.

\bibitem[Barath(2018)]{Barath-CVPR-2018}
Daniel Barath.
\newblock Five-point fundamental matrix estimation for uncalibrated cameras.
\newblock In \emph{Conference on Computer Vision and Pattern Recognition (CVPR)}, pages 235--243, 2018.

\bibitem[Barath and Hajder(2018)]{Barath-TIP-2018}
Daniel Barath and Levente Hajder.
\newblock Efficient recovery of essential matrix from two affine correspondences.
\newblock \emph{IEEE Transactions on Image Processing}, 27\penalty0 (11):\penalty0 5328--5337, 2018.

\bibitem[Barath et~al.(2017)Barath, Toth, and Hajder]{Barath-CVPR-2017}
Daniel Barath, Tekla Toth, and Levente Hajder.
\newblock A minimal solution for two-view focal-length estimation using two affine correspondences.
\newblock In \emph{Conference on Computer Vision and Pattern Recognition (CVPR)}, pages 2557--2565, 2017.

\bibitem[Bhayani et~al.(2020)Bhayani, Kukelova, and Heikkilä]{Bhayani_2020_CVPR}
Snehal Bhayani, Zuzana Kukelova, and Janne Heikkilä.
\newblock A sparse resultant based method for efficient minimal solvers.
\newblock In \emph{Conference on Computer Vision and Pattern Recognition (CVPR)}, pages 1767--1776, 2020.

\bibitem[Breiding and Timme(2018)]{HomotopyContinuation.jl}
Paul Breiding and Sascha Timme.
\newblock {H}omotopy{C}ontinuation.jl: {A} {P}ackage for {H}omotopy {C}ontinuation in {J}ulia.
\newblock In \emph{International Congress on Mathematical Software}, pages 458--465. Springer, 2018.

\bibitem[Byr{\"o}d et~al.(2008)Byr{\"o}d, Josephson, and {\AA}str{\"o}m]{Byrod-ECCV-2008}
Martin Byr{\"o}d, Klas Josephson, and Kalle {\AA}str{\"o}m.
\newblock A column-pivoting based strategy for monomial ordering in numerical gr{\"o}bner basis calculations.
\newblock In \emph{European Conference on Computer Vision (ECCV)}, pages 130--143. 2008.

\bibitem[Camposeco et~al.(2016)Camposeco, Sattler, and Pollefeys]{DBLP:conf/eccv/CamposecoSP16}
Federico Camposeco, Torsten Sattler, and Marc Pollefeys.
\newblock Minimal solvers for generalized pose and scale estimation from two rays and one point.
\newblock In \emph{European Conference on Computer Vision (ECCV)}, pages 202--218, 2016.

\bibitem[Ding et~al.(2020)Ding, Yang, Ponce, and Kong]{Ding_2020_CVPR}
Yaqing Ding, Jian Yang, Jean Ponce, and Hui Kong.
\newblock Minimal solutions to relative pose estimation from two views sharing a common direction with unknown focal length.
\newblock In \emph{Conference on Computer Vision and Pattern Recognition (CVPR)}, 2020.

\bibitem[Duff et~al.(2018)Duff, Hill, Jensen, Lee, Leykin, and Sommars]{monodromy}
Timothy Duff, Cvetelina Hill, Anders Jensen, Kisun Lee, Anton Leykin, and Jeff Sommars.
\newblock Solving polynomial systems via homotopy continuation and monodromy.
\newblock \emph{IMA Journal of Numerical Analysis}, 39\penalty0 (3):\penalty0 1421--1446, 2018.

\bibitem[Duff et~al.(2022)Duff, Korotynskiy, Pajdla, and Regan]{duff2022galois}
Timothy Duff, Viktor Korotynskiy, Tomas Pajdla, and Margaret~H. Regan.
\newblock Galois/monodromy groups for decomposing minimal problems in {3D} reconstruction.
\newblock \emph{SIAM Journal on Applied Algebra and Geometry}, 6\penalty0 (4):\penalty0 740--772, 2022.

\bibitem[Duff et~al.(2023)Duff, Kohn, Leykin, and Pajdla]{plmp}
Timothy Duff, Kathl{\'e}n Kohn, Anton Leykin, and Tomas Pajdla.
\newblock {PLMP}: Point-line minimal problems in complete multi-view visibility.
\newblock \emph{IEEE Transactions on Pattern Analysis and Machine Intelligence}, 2023.

\bibitem[Duff et~al.(2024)Duff, Kohn, Leykin, and Pajdla]{duff2024pl}
Timothy Duff, Kathl{\'e}n Kohn, Anton Leykin, and Tomas Pajdla.
\newblock {PLMP}: Point-line minimal problems under partial visibility in three views.
\newblock \emph{International Journal of Computer Vision}, pages 1--22, 2024.

\bibitem[Elqursh and Elgammal(2011)]{Elqursh-CVPR-2011}
Ali Elqursh and Ahmed~M. Elgammal.
\newblock Line-based relative pose estimation.
\newblock In \emph{Conference on Computer Vision and Pattern Recognition (CVPR)}, 2011.

\bibitem[Fabbri et~al.(2012)Fabbri, Kimia, and Giblin]{fabbri2012camera}
Ricardo Fabbri, Benjamin Kimia, and Peter Giblin.
\newblock Camera pose estimation using first-order curve differential geometry.
\newblock In \emph{European Conference on Computer Vision (ECCV)}, pages 231--244, 2012.

\bibitem[Fabbri et~al.(2020)Fabbri, Duff, Fan, Regan, Pinho, Tsigaridas, Wampler, Hauenstein, Giblin, Kimia, Leykin, and Pajdla]{fabbri2020trplp}
Ricardo Fabbri, Timothy Duff, Hongyi Fan, Margaret~H. Regan, David da Costa~de Pinho, Elias Tsigaridas, Charles~W. Wampler, Jonathan~D. Hauenstein, Peter~J. Giblin, Benjamin Kimia, Anton Leykin, and Tomas Pajdla.
\newblock {TRPLP} - {T}rifocal relative pose from lines at points.
\newblock In \emph{Conference on Computer Vision and Pattern Recognition (CVPR)}, 2020.

\bibitem[Fischler and Bolles(1981)]{ransac}
Martin~A. Fischler and Robert~C. Bolles.
\newblock Random sample consensus: a paradigm for model fitting with applications to image analysis and automated cartography.
\newblock \emph{Commun. ACM}, 24\penalty0 (6):\penalty0 381–395, 1981.

\bibitem[Hahn et~al.(2025)Hahn, Kohn, Marigliano, and Pajdla]{hahn2024order}
Marvin~Anas Hahn, Kathl{\'e}n Kohn, Orlando Marigliano, and Tomas Pajdla.
\newblock Order-one rolling shutter cameras.
\newblock In \emph{Conference on Computer Vision and Pattern Recognition (CVPR)}, 2025.

\bibitem[Hartley and Li(2012)]{Hartley-PAMI-2012}
Richard Hartley and Hongdong Li.
\newblock An efficient hidden variable approach to minimal-case camera motion estimation.
\newblock \emph{IEEE Transactions on Pattern Analysis and Machine Intelligence}, 34\penalty0 (12):\penalty0 2303--2314, 2012.

\bibitem[Heyden(1995)]{heyden1995geometry}
Anders Heyden.
\newblock \emph{Geometry and algebra of multiple projective transformations}.
\newblock PhD thesis, Lund University, 1995.

\bibitem[Kileel(2017)]{kileel2017minimal}
Joe Kileel.
\newblock Minimal problems for the calibrated trifocal variety.
\newblock \emph{SIAM Journal on Applied Algebra and Geometry}, 1\penalty0 (1):\penalty0 575--598, 2017.

\bibitem[Kneip et~al.(2012)Kneip, Siegwart, and Pollefeys]{DBLP:conf/eccv/KneipSP12}
Laurent Kneip, Roland Siegwart, and Marc Pollefeys.
\newblock Finding the exact rotation between two images independently of the translation.
\newblock In \emph{European Conference on Computer Vision (ECCV)}, pages 696--709, 2012.

\bibitem[Kuang and {\AA}str{\"{o}}m(2013)]{Kuang-ICCV-2013}
Yubin Kuang and Kalle {\AA}str{\"{o}}m.
\newblock Pose estimation with unknown focal length using points, directions and lines.
\newblock In \emph{International Conference on Computer Vision (ICCV)}, pages 529--536, 2013.

\bibitem[Kuang and Åström(2013)]{kuang-astrom-2espc2-13}
Yubin Kuang and Kalle Åström.
\newblock Stratified sensor network self-calibration from {TDOA} measurements.
\newblock In \emph{21st European Signal Processing Conference (EUSIPCO 2013)}, pages 1--5, 2013.

\bibitem[Kukelova et~al.(2008)Kukelova, Bujnak, and Pajdla]{kukelova2008automatic}
Zuzana Kukelova, Martin Bujnak, and Tomas Pajdla.
\newblock Automatic generator of minimal problem solvers.
\newblock In \emph{European Conference on Computer Vision (ECCV)}, 2008.

\bibitem[Kukelova et~al.(2017)Kukelova, Kileel, Sturmfels, and Pajdla]{kukelova2017clever}
Zuzana Kukelova, Joe Kileel, Bernd Sturmfels, and Tomas Pajdla.
\newblock A clever elimination strategy for efficient minimal solvers.
\newblock In \emph{Conference on Computer Vision and Pattern Recognition (CVPR)}. IEEE, 2017.

\bibitem[Larsson et~al.(2017{\natexlab{a}})Larsson, {\AA}str{\"{o}}m, and Oskarsson]{Larsson-Saturated-ICCV-2017}
Viktor Larsson, Kalle {\AA}str{\"{o}}m, and Magnus Oskarsson.
\newblock Polynomial solvers for saturated ideals.
\newblock In \emph{International Conference on Computer Vision (ICCV)}, pages 2307--2316, 2017{\natexlab{a}}.

\bibitem[Larsson et~al.(2017{\natexlab{b}})Larsson, {\AA}str{\"o}m, and Oskarsson]{larsson2017efficient}
Viktor Larsson, Kalle {\AA}str{\"o}m, and Magnus Oskarsson.
\newblock Efficient solvers for minimal problems by syzygy-based reduction.
\newblock In \emph{Conference on Computer Vision and Pattern Recognition (CVPR)}, 2017{\natexlab{b}}.

\bibitem[Larsson et~al.(2017{\natexlab{c}})Larsson, Kukelova, and Zheng]{larsson2017making}
Viktor Larsson, Zuzana Kukelova, and Yinqiang Zheng.
\newblock Making minimal solvers for absolute pose estimation compact and robust.
\newblock In \emph{International Conference on Computer Vision (ICCV)}, 2017{\natexlab{c}}.

\bibitem[Larsson et~al.(2018)Larsson, Oskarsson, {\AA}str{\"{o}}m, Wallis, Kukelova, and Pajdla]{Larsson-CVPR-2018}
Viktor Larsson, Magnus Oskarsson, Kalle {\AA}str{\"{o}}m, Alge Wallis, Zuzana Kukelova, and Tomas Pajdla.
\newblock Beyond grobner bases: Basis selection for minimal solvers.
\newblock In \emph{Conference on Computer Vision and Pattern Recognition (CVPR)}, pages 3945--3954, 2018.

\bibitem[Larsson et~al.(2019)Larsson, Sattler, Kukelova, and Pollefeys]{Larsson_2019_ICCV}
Viktor Larsson, Torsten Sattler, Zuzana Kukelova, and Marc Pollefeys.
\newblock Revisiting radial distortion absolute pose.
\newblock In \emph{International Conference on Computer Vision (ICCV)}, 2019.

\bibitem[Lowe(2004)]{sift}
David~G. Lowe.
\newblock Distinctive image features from scale-invariant keypoints.
\newblock \emph{International Journal of Computer Vision}, 60:\penalty0 91--110, 2004.

\bibitem[Matas et~al.(2002)Matas, Obdrzalek, and Chum]{matas2002local}
Jiri Matas, Stepan Obdrzalek, and Ondrej Chum.
\newblock Local affine frames for wide-baseline stereo.
\newblock In \emph{International Conference on Pattern Recognition}, pages 363--366. IEEE, 2002.

\bibitem[Mateus et~al.(2020)Mateus, Ramalingam, and Miraldo]{Mateus_2020_CVPR}
Andre Mateus, Srikumar Ramalingam, and Pedro Miraldo.
\newblock Minimal solvers for {3D} scan alignment with pairs of intersecting lines.
\newblock In \emph{Conference on Computer Vision and Pattern Recognition (CVPR)}, 2020.

\bibitem[Miraldo et~al.(2018)Miraldo, Dias, and Ramalingam]{miraldo2018minimal}
Pedro Miraldo, Tiago Dias, and Srikumar Ramalingam.
\newblock A minimal closed-form solution for multi-perspective pose estimation using points and lines.
\newblock In \emph{European Conference on Computer Vision (ECCV)}, pages 474--490, 2018.

\bibitem[Mirzaei and Roumeliotis(2011)]{mirzaei2011optimal}
Faraz Mirzaei and Stergios Roumeliotis.
\newblock Optimal estimation of vanishing points in a manhattan world.
\newblock In \emph{International Conference on Computer Vision (ICCV)}, 2011.

\bibitem[Nist\'er(2004)]{Nister-5pt-PAMI-2004}
David Nist\'er.
\newblock An efficient solution to the five-point relative pose problem.
\newblock \emph{IEEE Transactions on Pattern Analysis and Machine Intelligence}, 26\penalty0 (6):\penalty0 756--770, 2004.

\bibitem[OSCAR()]{OSCAR}
OSCAR.
\newblock Oscar -- open source computer algebra research system, version 1.2.2, 2025.

\bibitem[Oskarsson et~al.(2004)Oskarsson, Zisserman, and {\AA}str{\"o}m]{oskarsson2004minimal}
Magnus Oskarsson, Andrew Zisserman, and Kalle {\AA}str{\"o}m.
\newblock Minimal projective reconstruction for combinations of points and lines in three views.
\newblock \emph{Image and Vision Computing}, 22\penalty0 (10):\penalty0 777--785, 2004.

\bibitem[Quan(1995)]{quan1995invariants}
Long Quan.
\newblock Invariants of six points and projective reconstruction from three uncalibrated images.
\newblock \emph{IEEE Transactions on Pattern Analysis and Machine Intelligence}, 17\penalty0 (1):\penalty0 34--46, 1995.

\bibitem[Ramalingam and Sturm(2008)]{DBLP:conf/cvpr/RamalingamS08}
Srikumar Ramalingam and Peter Sturm.
\newblock Minimal solutions for generic imaging models.
\newblock In \emph{Conference on Computer Vision and Pattern Recognition (CVPR)}, pages 1--8, 2008.

\bibitem[Sala{\"{u}}n et~al.(2016)Sala{\"{u}}n, Marlet, and Monasse]{SalaunMM-ECCV-2016}
Yohann Sala{\"{u}}n, Renaud Marlet, and Pascal Monasse.
\newblock Robust and accurate line- and/or point-based pose estimation without manhattan assumptions.
\newblock In \emph{European Conference on Computer Vision (ECCV)}, 2016.

\bibitem[Saurer et~al.(2015)Saurer, Pollefeys, and Lee]{saurer2015minimal}
Olivier Saurer, Marc Pollefeys, and Gim~Hee Lee.
\newblock A minimal solution to the rolling shutter pose estimation problem.
\newblock In \emph{International Conference on Intelligent Robots and Systems (IROS)}, pages 1328--1334, 2015.

\bibitem[Schaffalitzky et~al.(2000)Schaffalitzky, Zisserman, Hartley, and Torr]{schaffalitzky2000six}
Frederik Schaffalitzky, Andrew Zisserman, Richard Hartley, and Philip Torr.
\newblock A six point solution for structure and motion.
\newblock In \emph{European Conference on Computer Vision (ECCV)}, pages 632--648, 2000.

\bibitem[Sch\"{o}nberger and Frahm(2016)]{schoenberger2016sfm}
Johannes~Lutz Sch\"{o}nberger and Jan-Michael Frahm.
\newblock Structure-from-motion revisited.
\newblock In \emph{Conference on Computer Vision and Pattern Recognition (CVPR)}, 2016.

\bibitem[Shafarevich(1994)]{ShafarevichHirsch94}
Igor Shafarevich.
\newblock \emph{Basic algebraic geometry}.
\newblock Springer, 1994.

\bibitem[Stewenius et~al.(2006)Stewenius, Engels, and Nist\'er]{Stewenius-ISPRS-2006}
Henrik Stewenius, Christopher Engels, and David Nist\'er.
\newblock Recent developments on direct relative orientation.
\newblock \emph{ISPRS J. of Photogrammetry and Remote Sensing}, 60:\penalty0 284--294, 2006.

\bibitem[Trager et~al.(2015)Trager, Hebert, and Ponce]{handbook}
Matthew Trager, Martial Hebert, and Jean Ponce.
\newblock The joint image handbook.
\newblock In \emph{International Conference on Computer Vision (ICCV)}, pages 909--917, 2015.

\bibitem[Ventura et~al.(2015)Ventura, Arth, and Lepetit]{ventura2015efficient}
Jonathan Ventura, Clemens Arth, and Vincent Lepetit.
\newblock An efficient minimal solution for multi-camera motion.
\newblock In \emph{International Conference on Computer Vision (ICCV)}, pages 747--755, 2015.

\bibitem[Ventura et~al.(2024)Ventura, Kukelova, Sattler, and Bar{\'a}th]{ventura2024absolute}
Jonathan Ventura, Zuzana Kukelova, Torsten Sattler, and D{\'a}niel Bar{\'a}th.
\newblock Absolute pose from one or two scaled and oriented features.
\newblock In \emph{Conference on Computer Vision and Pattern Recognition (CVPR)}, pages 20870--20880, 2024.

\end{thebibliography}
}

\newpage 
\clearpage
\setcounter{page}{1}
\maketitlesupplementary 
\appendix

\section{Proofs} \label{sec:proofs}
Here, we provide formal proofs for the lemmas and propositions used in the main paper.
\begin{lemma} \label{lem:minimalPointArrangement}
    For any minimal PLP with at least two views, the underlying point arrangement is among those depicted in
    Tables~\ref{tab:7Pts} or~\ref{tab:atMost6Pts}.
\end{lemma}

\begin{proof}
    Since minimal PLPs are balanced, it suffices to restrict to tuples $(p^f,p^d)$ appearing in at least one balanced PLP. For $p^f+p^d\geq 9$, there are no balanced problems. Thus, we can assume $p^f+p^d\leq 8$. Starting with  $p^f+p^d= 8$, we only have to consider one tuple, namely  $(2,6)$ which is not possible by Lemma~\ref{lem:homography}.

     We consider the case $p^f+p^d=7$. The tuples $(7,0)$ and $(6,1)$ give  rise to the unique generic configurations represented by the first two entries of Table~\ref{tab:7Pts}. In case of three or more collinear points, we will refer to a line containing these points as a supporting line. Lemma~\ref{lem:homography} implies that on each supporting line we have precisely three points. For the tuple $(5,2)$ we thus need two distinct supporting lines. These lines can either have precisely one point in common or no point in common. In both cases we obtain a unique configuration.

    Next, we have to consider $p^f=4$ and $p^d=3$, in which case we need $3$ supporting lines by Lemma~\ref{lem:homography}. If a point is on no line, we have three dependent points in the span of three other points which is excluded by Lemma~\ref{lem:homography}. Thus, all points lie on one of the 3 supporting lines.
    First, we  investigate the case of one point lying on all 3 lines. Besides this point each line must contain two further points and no other point is allowed on two of the lines as two points define a line uniquely, so the only possibility is the last entry of Table~\ref{tab:7Pts}.  Second, we  consider the case where each point lies on 1 or 2 lines. As there are in total 9 points on the lines (counted with multiplicity), there are precisely 2 points, say $p_1$ and $p_2$, lying on two of the lines and the other 5 lying only on one of the lines. Exaclty one of the supporting lines contains both $p_1$ and $p_2$. As the 5 other points lie on exactly one line each and we have still 5 positions to fill, there is a unique configuration: see the penultimate entry of Table~\ref{tab:7Pts}.

The tuples $(3,4)$ and $(2,5)$ are not possible by Lemma~\ref{lem:homography} and the tuples $(1,6)$ and $(0,7)$ cannot exist as we need at least two free points to define the first dependent point. Any configuration appearing for $p^f+p^d\leq 6$ is a subconfiguration of a seven-point one. Hence, the proofs are the same and we listed all such problems in Table~\ref{tab:atMost6Pts}.
\end{proof}

\begin{table}[!htb]
    \centering
    \input{tab_lt7pt_v2.tex}%
    \caption{Arrangements of at most 6 points in minimal PLPs.}
    \label{tab:atMost6Pts}
\end{table}

\begin{proof}[\textbf{Proof of Lemma~\ref{lem:homography}}]
As an alternative to standard homography arguments, we illustrate the stabilizer technique carried out in Examples \ref{ex:running1} and \ref{ex:running2} to prove Lemma \ref{lem:homography}.

First, consider $\XX$ to be the variety of two free points and $p^d \geq 1$ dependent points. By applying homographies we know that the orbit of 
     $\XX' := \{(e_1,e_2,e_1+e_2 )\} \times \{ \lambda e_1+\mu e_2 \mid (\lambda:\mu) \in \PP^1 \}^{p^d-1}$ is dense in $\XX$.
     The stabilizer of the first three points is given by \begin{gather*}
      \begin{bmatrix}
            1 & 0 & h_1 &h_2\\
            0 & 1 & h_3& h_4\\
            0 & 0& h_5& h_6\\
            0 & 0& h_7& h_8
        \end{bmatrix}
    \end{gather*}
    and it indeed stabilizes all further dependent points. The camera matrices can generically be normalized to \begin{gather*}
       \CC' = \begin{bmatrix}
            c_1 & 1 & 1 & 0\\
             c_2 & c_3 & 1 & 0\\
              c_4 & c_5 & 1 & 0
        \end{bmatrix}.
    \end{gather*}
    Now, the inequality in \eqref{prop:Ineq} specializes to \begin{gather*}
        5m+p^d-1\geq m(5+p^d-1),
    \end{gather*}
    i.e., $0\geq (m-1)(p^d-1)$.
    Since $m \geq 2$, we get $p^d\leq 1$.
    
   For the second statement, suppose for the sake of contradiction that we have a minimal problem with $p^d>2$ points in a plane spanned by 3 free points.
   We consider the subproblem consisting of the 3 free points and the first 3 dependent points.
   We show now that each of the 3 dependent points has to lie on exactly one of the 3 lines spanned by the 3 free points.
   To prove this formally, we begin with investigating how arrangements with 3 free and only 2 dependent points can look like:
   
   By the first part of this lemma, each of the free points has to lie on at least one supporting line as otherwise both dependent points have to lie on the same supporting line.
   Since we only have 2 supporting lines, 
   containing 3 points each (so in total 6 when counted with multiplicities),
   there is precisely one point $p_1$ on both supporting lines. 
   
   Adding the third dependent point, we see that its supporting line cannot contain $p_1$ as otherwise there would be four points on a line which we proved above to be impossible. Hence, it contains two points $p_2$ and $p_3$ such that $p_1$, $p_2$ and $p_3$ are not collinear.
    Thus, we can view $p_1,p_2,p_3$ as the free points and we can write the dependent points as $p_4=\lambda_1 p_1+\mu_1 p_2$, $p_5=\lambda_2 p_1+\mu_2 p_3$ and $p_6=\lambda_3 p_2+\mu_3 p_3$, where $(\lambda_i:\mu_i)\in \PP^1$. 
    
    We can choose a homography taking $p_1$ to $e_1$, $p_2$ to $e_2$, $p_3$ to $e_3$, $p_4$ to $e_1+e_2$ and $p_5$ to $e_1+e_3$. Hence
     $\XX' := \{(e_1,e_2,e_3,e_1+e_2,e_1+e_3 )\} \times \{ \lambda e_2+\mu e_3 \mid (\lambda:\mu) \in \PP^1 \}$ is dense in the subproblem with precisely 3 dependent points.
     The stabilizer of $\XX'$is given by \begin{gather*}
      \mathrm{Stab}(\XX')=  \begin{bmatrix}
            1 & 0 & 0 &h_1\\
            0 & 1 & 0& h_2\\
            0 & 0& 1& h_3\\
            0 & 0& 0& h_4
        \end{bmatrix}.
    \end{gather*}
    The camera matrices can then generically be normalized to \begin{gather*}
       \CC' = \begin{bmatrix}
            c_1 & c_2 & 1 & 0\\
             c_3 & c_4 & c_5 & 0\\
              c_6 & c_7 & c_8 & 0
        \end{bmatrix}.
    \end{gather*}
    Now, the inequality in \eqref{prop:Ineq} specializes to \begin{gather*}
        8m+1\geq m(3\cdot 2+3),
    \end{gather*}
    i.e., $1\geq m$, contradicting our assumption that $m \geq 2$.
\end{proof}

\begin{proof}[\textbf{Proof of Proposition \ref{Schlüssellemma}}]
First we reduce to the case $p'=p$ and $l'=l$.
    Consider the diagram\[\begin{tikzcd}
        (\Cs \times \Xs) / \PGL{4} && \Ys^{m} \\
	\\
    (\Cs \times \XX_{p',l',\II'}) / \PGL{4} && \YY_{p',l',\II'}^{m}
	\arrow["{\Fm}", dashed, from=1-1, to=1-3]
	\arrow[from=1-1, to=3-1]
	\arrow[from=1-3, to=3-3, two heads]
	\arrow["{\Phi_{p',l',\II',m}}", dashed, from=3-1, to=3-3]
\end{tikzcd}\]
where the unlabled vertical maps are given by projecting onto the points and lines present in the subarrangement. The diagram commutes, and the upper and right morphisms have full dimensional images, so has their composition. Thus, the same holds for $\Phi_{p',l',\II',m}$. Hence, we can assume without loss of generality that $\XX_{p',l',\II'}=\Xs$.  

Now we consider $\Phi: \Cs \times \XX' \to \Ys^{m}$ given by $(P_1,...,P_m,A)\mapsto (P_1A,...,P_mA)$. We claim that this map has also full dimensional image. This follows from the diagram \[\begin{tikzcd}
	\Cs\times \XX' && (\Cs \times \Xs) / \PGL{4} \\
	\\
  &\Ys^{m}&
	\arrow["{\Fm}", dashed, from=1-3, to=3-2]
	\arrow["{\iota}", dashed, from=1-1, to=1-3]
	\arrow["{\Phi}"', dashed, from=1-1, to=3-2]
\end{tikzcd}\]
where $\iota$ takes a tuple of cameras and points and lines to their equivalence class on the right. Since $\PGL{4} \cdot \XX'$ has the same dimension as $\Xs$, the composition has full dimensional image and thus $\Phi$ has by commutativity.

Finally, since $(\CC' \cdot \Stab(\XX'))^m$ has the same dimension as $\Cs$, we can restrict $\Phi$ to this subset without reducing the dimension of its image. By definition of the stabilizer, the image of $\Phi$ on this subset agrees with the image of $\Phi^{\mathrm{red}}_{p',l',\II',m}$ finishing the proof. The inequality \eqref{prop:Ineq} follows directly from the Fiber-Dimension Theorem.
\end{proof}

\section{Non-minimality via subproblem stabilizers} \label{appendix:nonMinimal}

In Section \ref{subsec:nonMinimal}, we explained a non-minimality criterion to formally disprove the minimality of a balanced PLP, focusing at one concrete example problem. 
Here, we explain how that strategy shows the non-minimality for in total 130 balanced PLPs (out of the $149 = 434-285$ balanced problems that are claimed to be non-minimal in part c) of the proof of our Main Theorem \ref{thm:main}). 
For that, we systematically consider sub-arrangements that appear in balanced PLPs, to identify relevant reduced sub-problems $(\mathcal{C}', \mathcal{X}')$ as in Definition \ref{def:reduced}.
All reduced subproblems that we need to consider in this section are listed in Table \ref{tab:reducedSubproblems}. 
There, we use $\ell^a_*$ to denote the number of lines adjacent to a single (fixed) point. 
The diagram in the first column shows at which point these lines are attached (which is relevant for the second to last row where not all points are indistinguishable due to some being collinear and some not). The diagram itself does not depict all adjacent lines but only those that are normalized; this is indicated by the inequality in the description of $\XX'$.
For example, we can consider Example \ref{ex:running1} and Example \ref{ex:positiveEffect} both as instances of the first row of Table \ref{tab:reducedSubproblems}. 

Similarly to Example \ref{ex:running2} and the proof of Lemma \ref{lem:homography} in SM Section \ref{sec:proofs}, we can apply  inequality \eqref{prop:Ineq} to the reduced subproblems in Table \ref{tab:reducedSubproblems} to extract neccessary conditions that minimal problems have to satisfy. 
These conditions are listed in  Table \ref{tab:NonMinCriteria}. 
The table shall be read as follows: Given a minimal PLP with one of the depicted arrangements as a subproblem, then the last three columns specify how many lines can at most be attached to the point specified by the diagram depending on the number of cameras. For example, the first row of Table \ref{tab:NonMinCriteria} says that for three cameras, at most 7 lines can be attached to the point, while for four cameras, there are at most 6 attached lines, and in case of five or more cameras, at most 5 lines can be attached to a single point.

\begin{table*}[!htb]
\centering
\begin{tabular}{|>{\centering\arraybackslash}m{9.5em}|c|c|}
\hline
  $\mathcal{X}'$ & $\Stab(\mathcal{X}')$ & $\mathcal{C}'$
\\
\hline
  $p^f=1, \ell_*^a\geq4$ 
  \begin{tikzpicture}
    \node[fill=black, circle, inner sep=2pt, label=left:$e_1$] (e1) at (0,0) {};

    \node[fill=black, circle, inner sep=1pt, label=right:$e_2$] (e2) at (1,1) {};
    \node[fill=black, circle, inner sep=1pt, label=right:$e_3$] (e3) at (1,-1) {};
    \node[fill=black, circle, inner sep=1pt, label=right:$e_4$] (e4) at (2,0.5) {};
    \node[fill=black, circle, inner sep=1pt, label=right:$e$] (e) at (2,-0.5) {};

    \draw (e1) -- (e2);
    \draw (e1) -- (e3);
    \draw (e1) -- (e4);
    \draw (e1) -- (e);
  \end{tikzpicture}
&
  $\begin{bmatrix}
    \lambda_1 & \lambda_2 &\lambda_3 & \lambda_4\\
    0 &1 &0&0\\
    0&0&1&0\\
    0&0&0&1
  \end{bmatrix}$
&
  $\begin{bmatrix}
    1 & 1 & 1 &1\\
    c_1 & c_2 & c_3 &c_4\\
    c_5 & c_6 & c_7 &1
  \end{bmatrix}$
\\
\hline
  $p^f=2,\ell_*^a\geq3$
  \begin{tikzpicture}
    \node[fill=black, circle, inner sep=2pt, label=left:$e_1$] (e1) at (0,0) {};

    \node[fill=black, circle, inner sep=2pt, label=left:$e_2$] (e2) at (1,1) {};
    \node[fill=black, circle, inner sep=1pt, label=right:$e_3$] (e3) at (1,-1) {};
    \node[fill=black, circle, inner sep=1pt, label=right:$e_4$] (e4) at (2,0.5) {};
    \node[fill=black, circle, inner sep=1pt, label=right:$e$] (e) at (2,-0.5) {};

    \draw (e1) -- (e3);
    \draw (e1) -- (e4);
    \draw (e1) -- (e);
  \end{tikzpicture}
&
  $\begin{bmatrix}
    \lambda_1 & 0 &\lambda_2 & \lambda_3\\
    0 &1 &0&0\\
    0&0&1&0\\
    0&0&0&1
  \end{bmatrix}$
&
  $\begin{bmatrix}
    1 & 1 & 0 &0\\
    c_1 & c_2 & c_3 &c_4\\
    c_5 & c_6 & c_7 &c_8
  \end{bmatrix}$
\\
\hline
  $p^f=3, \ell_*^a\geq2$
  \begin{tikzpicture}
    \node[fill=black, circle, inner sep=2pt, label=left:$e_1$] (e1) at (0,0) {};

    \node[fill=black, circle, inner sep=2pt, label=left:$e_2$] (e2) at (1,1) {};
    \node[fill=black, circle, inner sep=2pt, label=left:$e_3$] (e3) at (1,-1) {};
    \node[fill=black, circle, inner sep=1pt, label=right:$e_4$] (e4) at (2,0.5) {};
    \node[fill=black, circle, inner sep=1pt, label=right:$e$] (e) at (2,-0.5) {};

    \draw (e1) -- (e4);
    \draw (e1) -- (e);
  \end{tikzpicture}
&
  $\begin{bmatrix}
    \lambda_1 & 0 &0 & \lambda_2\\
    0 &1 &0&0\\
    0&0&1&0\\
    0&0&0&1
  \end{bmatrix}$
&
  $\begin{bmatrix}
    1 & 1 & c_1 &0\\
    c_2 & c_3 & c_4 &c_5\\
    c_6 & c_7 & c_8 &c_9
  \end{bmatrix}$
\\
\hline
  $p^f=4, \ell_*^a\geq1$
  \begin{tikzpicture}
    \node[fill=black, circle, inner sep=2pt, label=left:$e_1$] (e1) at (0,0) {};

    \node[fill=black, circle, inner sep=2pt, label=left:$e_2$] (e2) at (1,1) {};
    \node[fill=black, circle, inner sep=2pt, label=left:$e_3$] (e3) at (1,-1) {};
    \node[fill=black, circle, inner sep=2pt, label=right:$e_4$] (e4) at (2,0.5) {};
    \node[fill=black, circle, inner sep=1pt, label=right:$e$] (e) at (2,-0.5) {};

    \draw (e1) -- (e);
  \end{tikzpicture}
&
  $\begin{bmatrix}
    \lambda & 0 &0 & 0\\
    0 &1 &0&0\\
    0&0&1&0\\
    0&0&0&1
  \end{bmatrix}$
&
  $\begin{bmatrix}
    1 & 1 & c_1 &c_2\\
    c_3 & c_4 & c_5 &c_6\\
    c_7 & c_8 & c_9 &c_{10}
  \end{bmatrix}$
\\
    
    
   

\hline
  $p^f=3, p^d=1, \ell_*^a\geq2$
  \begin{tikzpicture}
    \node[fill=black, circle, inner sep=2pt, label=left:$e_1$] (e1) at (0,0) {};

    \node[fill=black, circle, inner sep=2pt, label=left:$e_2$] (e2) at (0.5,0.5) {};
    \node[fill=black, circle, inner sep=2pt, label=right:$e_1+e_2$] (e1e2) at (1.1,1) {};
    \node[fill=black, circle, inner sep=2pt, label=left:$e_3$] (e3) at (1,-1) {};
    \node[fill=black, circle, inner sep=1pt, label=right:$e_4$] (e4) at (2,0.5) {};
    \node[fill=black, circle, inner sep=1pt, label=right:$e$] (e) at (2,-0.5) {};

    \draw (e1) -- (e4);
    \draw (e1) -- (e);
  \end{tikzpicture}
&
  $\begin{bmatrix}
    1 & 0 &0 & \lambda\\
    0 &1 &0&0\\
    0&0&1&0\\
    0&0&0&1
  \end{bmatrix}$
&
  $\begin{bmatrix}
    c_1 & c_2 & 1 &0\\
    c_3 & c_4 & c_5 &c_6\\
    c_7 & c_8 & c_9 &c_{10}
  \end{bmatrix}$
\\
\hline
  $p^f=3, \ell^f=1$
  \begin{tikzpicture}
    \node[fill=black, circle, inner sep=1pt, label=left:$e_1$] (e1) at (0,0) {};

    \node[fill=black, circle, inner sep=2pt, label=left:$e_2$] (e2) at (0.5,0.5) {};

    \node[fill=black, circle, inner sep=2pt, label=left:$e_3$] (e3) at (1,-1) {};
    \node[fill=black, circle, inner sep=2pt, label=right:$e_4$] (e4) at (2,0.5) {};
    \node[fill=black, circle, inner sep=1pt, label=right:$e$] (e) at (2,-0.5) {};

    \draw (e1) -- (e);
  \end{tikzpicture}
&
  $\begin{bmatrix}
    \lambda & 0 & 0 & 0\\
    \mu & 1 & 0 & 0\\
    \mu & 0 & 1 & 0\\
    \mu & 0 & 0 & 1
  \end{bmatrix}$
&
  $\begin{bmatrix}
    1 & c_1 & c_2 & c_3 \\
    1&c_4 & c_5 & c_6 \\
    1&c_8 & c_8 & c_9 
  \end{bmatrix}$
\\
\hline
\end{tabular}
\caption{Reduced sub-PLPs. $e_i$ denote standard basis vectors of $\RR^4$, $e := e_2+e_3+e_4$, and $\ell^a_*$ denotes the number of lines adjacent to $e_1$.}
\label{tab:reducedSubproblems}
\end{table*}

\begin{prop} \label{prop:minimalCriteria}
    Any minimal PLP has to satisfy all criteria listed in Table \ref{tab:NonMinCriteria}. In other words, if a given PLP has a subproblem as shown in the first entry of a row, it has to satisfy the given bound on the number of lines attached to a single point, as specified in the last three columns depending on the number of cameras.
\end{prop}
\begin{proof}
    The criteria $1$ to $4$ follow immediately from the first four rows of Table \ref{tab:reducedSubproblems}. For criterion $5$,  we observe that any matrix in the stabilizer of the subproblem without the dependent point already fixes the dependent point and thus the stabilizer does not change when adding it.
    Thus, in the third row of Table \ref{tab:reducedSubproblems}, we can add an unnormalized dependent point on the line spanned by $e_2$ and $e_3$ to turn $\XX'$ into another reduced subproblem $(\XX'', \CC')$, with the same camera variety $\CC'$ as in row 3 of Table \ref{tab:reducedSubproblems}.
    Thus, we have $\dim(\mathcal{C}')=9$ and $\dim(\mathcal{X}'')=1+ \dim(\mathcal{X}')
    = 1+
    2\cdot (\ell_*^a-2)$, where the summand of $1$ comes from the unnormalized dependent point. Lastly, $\dim(\mathcal{Y}_{p',l',\mathcal{I}'})=2\cdot 3+1+\ell_*^a$. Hence, in this case, inequality \eqref{prop:Ineq} becomes  \begin{gather}
        2\cdot (m-2)+1\geq (m-2)\cdot l_*^a,
    \end{gather}
    which yields the claimed bounds for $m=3$ and $m>3$. Criterion $6$ is obtained similarly using the fourth row of Table \ref{tab:reducedSubproblems}. For criterion $7$, we use the same argument together with the fifth row of Table \ref{tab:reducedSubproblems}.
\end{proof}

\begin{table*}[!ht]
\centering
\begin{tabular}{|>{\centering\arraybackslash}m{11em}|c|c|c|c|}
\hline
  Arrangement & Number & $m = 3$ & $m = 4$ & $m\in \lbrace 5,6,7,8\rbrace$
\\
\hline
  $p^f=1,\ell_*^a\geq 4$
  \begin{tikzpicture}
    \node[fill=black, circle, inner sep=2pt, label=left:$e_1$] (e1) at (0,0) {};

    \node[fill=black, circle, inner sep=1pt, label=right:$e_2$] (e2) at (1,1) {};
    \node[fill=black, circle, inner sep=1pt, label=right:$e_3$] (e3) at (1,-1) {};
    \node[fill=black, circle, inner sep=1pt, label=right:$e_4$] (e4) at (2,0.5) {};
    \node[fill=black, circle, inner sep=1pt, label=right:$e$] (e) at (2,-0.5) {};

    \draw (e1) -- (e2);
    \draw (e1) -- (e3);
    \draw (e1) -- (e4);
    \draw (e1) -- (e);
  \end{tikzpicture}
& 1 & $\ell_*^a \leq 7$ &$\ell_*^a \leq 6$ &$\ell_*^a \leq 5$
\\
\hline
  $p^f=2,\ell_*^a\geq 3$
  \begin{tikzpicture}
    \node[fill=black, circle, inner sep=2pt, label=left:$e_1$] (e1) at (0,0) {};

    \node[fill=black, circle, inner sep=2pt, label=left:$e_2$] (e2) at (1,1) {};
    \node[fill=black, circle, inner sep=1pt, label=right:$e_3$] (e3) at (1,-1) {};
    \node[fill=black, circle, inner sep=1pt, label=right:$e_4$] (e4) at (2,0.5) {};
    \node[fill=black, circle, inner sep=1pt, label=right:$e$] (e) at (2,-0.5) {};

    \draw (e1) -- (e3);
    \draw (e1) -- (e4);
    \draw (e1) -- (e);
  \end{tikzpicture}
& 2 & $\ell_*^a \leq 6$ &$\ell_*^a \leq 5$ &$\ell_*^a \leq 4$
\\
\hline
  $p^f=3, \ell_*^a\geq2$
  \begin{tikzpicture}
    \node[fill=black, circle, inner sep=2pt, label=left:$e_1$] (e1) at (0,0) {};

    \node[fill=black, circle, inner sep=2pt, label=left:$e_2$] (e2) at (1,1) {};
    \node[fill=black, circle, inner sep=2pt, label=left:$e_3$] (e3) at (1,-1) {};
    \node[fill=black, circle, inner sep=1pt, label=right:$e_4$] (e4) at (2,0.5) {};
    \node[fill=black, circle, inner sep=1pt, label=right:$e$] (e) at (2,-0.5) {};

    \draw (e1) -- (e4);
    \draw (e1) -- (e);
  \end{tikzpicture}
& 3 & $\ell_*^a \leq 5$ & $\ell_*^a \leq 4$ &$\ell_*^a \leq 3$
\\
\hline
  $p^f=4, \ell_*^a\geq1$
  \begin{tikzpicture}
    \node[fill=black, circle, inner sep=2pt, label=left:$e_1$] (e1) at (0,0) {};

    \node[fill=black, circle, inner sep=2pt, label=left:$e_2$] (e2) at (1,1) {};
    \node[fill=black, circle, inner sep=2pt, label=left:$e_3$] (e3) at (1,-1) {};
    \node[fill=black, circle, inner sep=2pt, label=right:$e_4$] (e4) at (2,0.5) {};
    \node[fill=black, circle, inner sep=1pt, label=right:$e$] (e) at (2,-0.5) {};

    \draw (e1) -- (e);
  \end{tikzpicture}
& 4 & $\ell_*^a \leq 4$ &$\ell_*^a \leq 3$ & $\ell_*^a \leq 2$
\\
\hline
  $p^f=3, p^d=1, \ell_*^a\geq2$
  \begin{tikzpicture}
    \node[fill=black, circle, inner sep=2pt, label=left:$e_3$] (e1) at (0,0) {};

    \node[fill=black, circle, inner sep=2pt, label=left:$e_2$] (e2) at (0.5,0.5) {};
    \node[fill=black, circle, inner sep=2pt, label=right:$e_1+e_2$] (e1e2) at (1,1) {};
    \node[fill=black, circle, inner sep=2pt, label=above:$e_1$] (e3) at (1,-1) {};
    \node[fill=black, circle, inner sep=1pt, label=right:$e_4$] (e4) at (2,0.5) {};
    \node[fill=black, circle, inner sep=1pt, label=right:$e$] (e) at (2,-0.5) {};

    \draw (e3) -- (e4);
    \draw (e3) -- (e);
  \end{tikzpicture}
&5  & $\ell_*^a \leq 3$ & $\ell_*^a \leq 2$ & $\ell_*^a \leq 2$\\
\hline
  $p^f=4, p^d=1, \ell_*^a\geq1$
  \begin{tikzpicture}
    \node[fill=black, circle, inner sep=2pt, label=left:$e_3$] (e1) at (0,0) {};

    \node[fill=black, circle, inner sep=2pt, label=left:$e_2$] (e2) at (0.5,0.5) {};
    \node[fill=black, circle, inner sep=2pt, label=right:$e_1+e_2$] (e1e2) at (1,1) {};
    \node[fill=black, circle, inner sep=2pt, label=left:$e_1$] (e3) at (1,-1) {};
    \node[fill=black, circle, inner sep=2pt, label=right:$e_4$] (e4) at (2,0.5) {};
    \node[fill=black, circle, inner sep=1pt, label=right:$e$] (e) at (2,-0.5) {};

    \draw (e3) -- (e);
  \end{tikzpicture}
& 6 & $\ell_*^a \leq 2$ & $\ell_*^a \leq 1$ & $\ell_*^a \leq 1$
\\
\hline
  $p^f=3, p^d=2, \ell_*^a\geq2$
  \begin{tikzpicture}
    \node[fill=black, circle, inner sep=2pt, label=left:$e_1$] (e1) at (0,0) {};

    \node[fill=black, circle, inner sep=2pt, label=left:$e_2$] (e2) at (0.5,0.5) {};
    \node[fill=black, circle, inner sep=2pt, label=right:$e_1+e_2$] (e1e2) at (1,1) {};
    \node[fill=black, circle, inner sep=2pt, label=left:$e_3$] (e3) at (1,-0.5) {};
    \node[fill=black, circle, inner sep=2pt, label=left:$e_1+e_3$] (e1e3) at (2,-1) {};
    \node[fill=black, circle, inner sep=1pt, label=right:$e_4$] (e4) at (1.5,0.5) {};
    \node[fill=black, circle, inner sep=1pt, label=right:$e$] (e) at (2,0) {};

    \draw (e1) -- (e4);
    \draw (e1) -- (e);
  \end{tikzpicture}
& 7 & $\ell_*^a \leq 3$ & $\ell_*^a \leq 2$ & $\ell_*^a \leq 2$
\\
\hline
\end{tabular}
\caption{Necessary conditions for minimality. Same notation as in Table \ref{tab:reducedSubproblems}.}
\label{tab:NonMinCriteria}
\end{table*}

To deal with free lines in the case of five points in one plane, we can formulate an $8$th criterion which is independent of the number of cameras as long as this number is greater than $2$.

\begin{lemma}\label{lem:Criterion8}
    Let $(p,\ell,\mathcal{I},m)$ be a minimal PLP with $m\geq 3$ and 5 points contained in a plane, 2 out of which are dependent on the other 3. Then, any line has to be adjacent to one of these 5 points.
\end{lemma}
\begin{proof}
    Suppose there is a line not adjacent to one of the 5 points. We consider the reduced subproblem containing the 3 free points in the plane and the line. Note that the line is free in the subproblem, even if it was not free in the original PLP. A description of the stabilizers and reduced varieties is shown in the last row of Table \ref{tab:reducedSubproblems}. We observe that the stabilizer acts as the identity on the plane spanned by the free points and thus the two dependent points are also fixed. Hence, as in the proof of Proposition \ref{prop:minimalCriteria}, by adding the unnormalized dependent points to $\mathcal{X}'$, we obtain $\mathcal{X}''$ with the same stabilizer. 
    Next, we compute  $\dim(\CC ')=9$, $\dim(\mathcal{X}'')=2$ and $\dim(\mathcal{Y}_{p',l',\mathcal{I}'})=2\cdot 3+2+2$. Therefore, inequality \eqref{prop:Ineq} becomes $  9m+2\geq 10m$, i.e., $ m\leq 2$.
\end{proof}

\begin{remark}
    There is also a geometric proof of this lemma. After normalizing the world points, each camera gives a unique isomorphism between the plane $\Pi$ spanned by the points and the image plane. The intersection of the preimage plane of the line under the camera map with the plane $\Pi$ is given by applying the inverse of the isomorphism to the line in the image. Hence, for each image, we get a unique line in the plane $\Pi$. These three lines have to intersect in one point (namely the intersection of the true world line with the plane $\Pi$), but three generic lines do not intersect and hence the reconstruction is not possible in general. \hfill $\diamondsuit$
\end{remark}

The 130 non-minimal balanced PLPs mentioned above are listed in our code\footref{code}(lines 1396-1544), together with the number of one neccessary minimality-condition they violate.

\section{Non-minimality via elimination}
\label{appendix:nonMinimal19}
In the previous section, we described a method to identify overconstrained subproblems in non-minimal problems. This method covered the majority of the non-minimal balanced problems described in part c) of the proof of our main theorem. However, there are 19 cases left for which we need a more detailed look at the equations. But also here we will make use of the stabilizers. In order to identify contradictory constraints, we start of by finding a minimal subproblem to eliminate some of the variables. 
This factorizes the original non-minimal problem into the minimal subproblem and the `remainder problem', analogously to the factorization in Example \ref{ex:positiveEffect}.
We then find contradictory equations in the `remainder problem'.

\begin{example}\label{ex:(34112)}
    One of the remaining balanced PLPs is the scenario where 3 cameras observe 4 free points, 1 dependent point, 1 free line and 2 lines attached to a point that is not among the collinear points; see first entry of Table \ref{tab:elim}. We consider the subproblem arising from omitting the free line. More explicitely, we  choose the normalized point-line variety to be $\XX'=\lbrace (e_1,e_2,e_3,e_4,\overline{e_1e})\rbrace\times  \{ \lambda e_3+\mu e_4 \mid (\lambda:\mu) \in \PP^1 \}\times \{ \overline{e_1Q} \mid Q=\lambda e_2+\mu e_3+\nu e_4, (\lambda:\mu:\nu) \in \PP^2  \}$. We can compute its stabilizer to be \begin{gather}
      \left\lbrace   \begin{bmatrix}
        \lambda & 0 & 0 & 0\\
        0& 1 & 0& 0\\
        0 &0&1&0\\
        0&0&0&1
    \end{bmatrix} \mathrel{\Bigg\vert} \lambda\neq 0\right\rbrace.
    \end{gather}
    Setting the reduced camera variety to \begin{gather}
        \CC':= \left\lbrace \begin{bmatrix}
    1 & 1& c_{1}& c_2\\
    c_3 & c_4 & c_5 & c_6\\
    c_7 & c_8 & c_9& c_{10}
\end{bmatrix} \mathrel{\Bigg\vert} c_i\in \mathbb{R}\right\rbrace, \label{eq:cameraParamsSectionC}
    \end{gather}
    we can verify that this indeed satisfies the definition of a reduced subproblem. 
    Moreover, it is a balanced subproblem, i.e., it satisfies the inequality in \eqref{prop:Ineq}
with equality.
To see this, we just compute the dimensions: $\dim(\CC')=10$, $\dim(\XX')=1+2$ and $\dim(\mathcal{Y}_{p',\ell',\mathcal{I}'})= 4\cdot 2+ 1+2\cdot 1$, and indeed $30+3=3\cdot 11$.
Hence, if we assume for the sake of contradiction that the original PLP was minimal, then Proposition \ref{Schlüssellemma} (together with Lemma \ref{lem:balancedVsMinimal}) would tell us that the subproblem is minimal as well.
    

Therefore, we could solve the subproblem first and then assume 
that the camera parameters $c_j$ in \eqref{eq:cameraParamsSectionC} are  just numbers. To distinguish between the different cameras, we  decorate them with a superscript $i$. Re-introducing the  parameters that we could not reconstruct from the subproblem, we obtain cameras of the form 
\begin{gather}
    \begin{bmatrix}
        \lambda^i & 1& c_{1}^i& c_2^i\\
    c_3^i\lambda^i & c_4^i & c_5^i & c_6^i\\
    c_7^i\lambda^i & c_8^i & c_9^i& c_{10}^i
    \end{bmatrix}.
\end{gather}
Here, the remaining $\PGL{4}$-action allows us to set $\lambda^1=1$.

 Now, we consider the free line from the original PLP. Generically, it intersects the plane orthogonal to $e_1$ at a unique point of the form $\begin{bmatrix}
    0 &
    1&
    x_1&
    x_2
\end{bmatrix}^\top$.
After applying $C_i$ to it, the resulting image point has to be orthogonal to the coefficient vector representing the free line in the $i$-th image:  $\begin{bmatrix}
    y_1^i&
    y_2^i&
    1
\end{bmatrix}^\top$. Since our point on the line has a $0$ in its first entry, the resulting equation
$$
\begin{bmatrix}
    y_1^i&
    y_2^i&
    1
\end{bmatrix} C_i  \begin{bmatrix}
    0 &
    1&
    x_1&
    x_2
\end{bmatrix}^\top = 0
$$
does not  contain  $\lambda^i$. Hence, for $i=1,2,3$, we obtain three independent equations in two variables, which has no solution generically.  \hfill $\diamondsuit$
\end{example}
The strategy used in this example of eliminating variables and finding an overconstrained subsystem can directly be used for 15 further cases. 
These are listed in Table \ref{tab:elim} together with the respective subarrangements one has to consider and the parts of the 3D arrangement causing the contradictory constraints.


\begin{table*}[htb]
    \centering
    \begin{tabular}{|>{\centering\arraybackslash}m{5em}|>{\centering\arraybackslash}m{5em}|m{10em}||>{\centering\arraybackslash}m{5em}|>{\centering\arraybackslash}m{5em}|m{10em}|}
\hline
  Problem & Subproblem & \multicolumn{1}{c||}{Constraints} & Problem & Subproblem & \multicolumn{1}{c|}{Constraints}
\\
\hline
\begin{tikzpicture}[scale=1.4]
\coordinate (p_1) at (0.275434, 0.814813);
\coordinate (p_2) at (0.886218, 0.58854);
\coordinate (p_3) at (0.608196, 0.162511);
\coordinate (p_4) at (0.114101, 0.340228);
\coordinate (d_1_2) at (${0.5}*(p_1) + {0.5}*(p_2)$);
\coordinate (lf_1_1) at (0.562998875650816, 1.0);
\coordinate (lf_1_2) at (0.0, 0.3226269307829783);
\coordinate (_la_1) at (-0.7491417217476636, 0.66240975289993);
\coordinate (la_1_1) at ($(p_3) + {-0.24533304240849615}*(_la_1)$);
\coordinate (la_1_2) at ($(p_3) + {0.8118570656846437}*(_la_1)$);
\coordinate (_la_2) at (-0.786249542106276, -0.6179091013552651);
\coordinate (la_2_1) at ($(p_3) + {-0.4983201630240703}*(_la_2)$);
\coordinate (la_2_2) at ($(p_3) + {0.26300146679109154}*(_la_2)$);
\draw[dotted, gray] (p_1) -- (p_2);
\draw[yellow] (la_1_1) -- (la_1_2);
\draw[yellow] (la_2_1) -- (la_2_2);
\draw[violet] (lf_1_1) -- (lf_1_2);
\filldraw[fill=cyan] (d_1_2) circle[radius=1pt];
\filldraw[fill=red] (p_1) circle[radius=1pt];
\filldraw[fill=red] (p_2) circle[radius=1pt];
\filldraw[fill=red] (p_3) circle[radius=1pt];
\filldraw[fill=red] (p_4) circle[radius=1pt];
\draw (0, 0) -- (1, 0) -- (1, 1) -- (0, 1) -- cycle;
\node[below] at (.5, 0) {\scriptsize$(4{,}1{,}1{,}2)$};
\node[above] at (.5, 1) {$m = 3$};
\end{tikzpicture}\vspace{-4pt}
&
\begin{tikzpicture}[scale=1.4]
\coordinate (p_1) at (0.275434, 0.814813);
\coordinate (p_2) at (0.886218, 0.58854);
\coordinate (p_3) at (0.608196, 0.162511);
\coordinate (p_4) at (0.114101, 0.340228);
\coordinate (d_1_2) at (${0.5}*(p_1) + {0.5}*(p_2)$);

\coordinate (_la_1) at (-0.7491417217476636, 0.66240975289993);
\coordinate (la_1_1) at ($(p_3) + {-0.24533304240849615}*(_la_1)$);
\coordinate (la_1_2) at ($(p_3) + {0.8118570656846437}*(_la_1)$);
\coordinate (_la_2) at (-0.786249542106276, -0.6179091013552651);
\coordinate (la_2_1) at ($(p_3) + {-0.4983201630240703}*(_la_2)$);
\coordinate (la_2_2) at ($(p_3) + {0.26300146679109154}*(_la_2)$);
\draw[dotted, gray] (p_1) -- (p_2);
\draw[yellow] (la_1_1) -- (la_1_2);
\draw[yellow] (la_2_1) -- (la_2_2);
\filldraw[fill=cyan] (d_1_2) circle[radius=1pt];
\filldraw[fill=red] (p_1) circle[radius=1pt];
\filldraw[fill=red] (p_2) circle[radius=1pt];
\filldraw[fill=red] (p_3) circle[radius=1pt];
\filldraw[fill=red] (p_4) circle[radius=1pt];
\draw (0, 0) -- (1, 0) -- (1, 1) -- (0, 1) -- cycle;
\node[below] at (.5, 0) {\scriptsize$(4{,}1{,}0{,}2)$};
\node[above] at (.5, 1) {\vphantom{$m = 3$}};
\end{tikzpicture}\vspace{-4pt}
&
1 point on free line with $0$ in one coordinate
&
\begin{tikzpicture}[scale=1.4]
\coordinate (p_1) at (0.19, 0.19);
\coordinate (p_2) at (0.82, 0.16);
\coordinate (p_3) at (0.45, 0.87);
\coordinate (d_1_2) at (${0.5}*(p_1) + {0.5}*(p_2)$);
\coordinate (lf_1_1) at (0.1434024680608862, 0.0);
\coordinate (lf_1_2) at (1.0, 0.6844032375328413);
\coordinate (lf_2_1) at (0.6781903743095253, 1.0);
\coordinate (lf_2_2) at (0.3763339794244607, 0.0);
\coordinate (_la_1) at (-0.09597778921088826, -0.9953834758414418);
\coordinate (la_1_1) at ($(p_3) + {-0.1306029315888585}*(_la_1)$);
\coordinate (la_1_2) at ($(p_3) + {0.8740350037100529}*(_la_1)$);
\draw[dotted, gray] (p_1) -- (p_2);
\draw[yellow] (la_1_1) -- (la_1_2);
\draw[violet] (lf_1_1) -- (lf_1_2);
\draw[violet] (lf_2_1) -- (lf_2_2);
\filldraw[fill=cyan] (d_1_2) circle[radius=1pt];
\filldraw[fill=red] (p_1) circle[radius=1pt];
\filldraw[fill=red] (p_2) circle[radius=1pt];
\filldraw[fill=red] (p_3) circle[radius=1pt];
\draw (0, 0) -- (1, 0) -- (1, 1) -- (0, 1) -- cycle;
\node[below] at (.5, 0) {\scriptsize$(3{,}1{,}2{,}1)$};
\node[above] at (.5, 1) {$m = 5$};
\end{tikzpicture}\vspace{-4pt}
&
\begin{tikzpicture}[scale=1.4]
\coordinate (p_1) at (0.19, 0.19);
\coordinate (p_2) at (0.82, 0.16);
\coordinate (p_3) at (0.45, 0.87);
\coordinate (d_1_2) at (${0.5}*(p_1) + {0.5}*(p_2)$);
\filldraw[fill=cyan] (d_1_2) circle[radius=1pt];
\filldraw[fill=red] (p_1) circle[radius=1pt];
\filldraw[fill=red] (p_2) circle[radius=1pt];
\filldraw[fill=red] (p_3) circle[radius=1pt];
\draw (0, 0) -- (1, 0) -- (1, 1) -- (0, 1) -- cycle;
\node[below] at (.5, 0) {\scriptsize$(3{,}1{,}0{,}0)$};
\node[above] at (.5, 1) {\vphantom{$m = 5$}};
\end{tikzpicture}\vspace{-4pt}
&
1 point on each free line with $0$ in one coordinate
\\
\hline
\begin{tikzpicture}[scale=1.4]
\coordinate (p_1) at (0.275434, 0.814813);
\coordinate (p_2) at (0.886218, 0.58854);
\coordinate (p_3) at (0.608196, 0.162511);
\coordinate (p_4) at (0.114101, 0.340228);
\coordinate (lf_1_1) at (0.6095855057649198, 0.0);
\coordinate (lf_1_2) at (0.0, 0.7160518693088268);
\coordinate (_la_1) at (0.6156411091665215, 0.7880266649703012);
\coordinate (la_1_1) at ($(p_1) + {-0.44739377520272017}*(_la_1)$);
\coordinate (la_1_2) at ($(p_1) + {0.235000931100446}*(_la_1)$);
\coordinate (_la_2) at (-0.20904233124028718, 0.9779065925484018);
\coordinate (la_2_1) at ($(p_1) + {-0.8332217066628176}*(_la_2)$);
\coordinate (la_2_2) at ($(p_1) + {0.18937084728860143}*(_la_2)$);
\coordinate (_la_3) at (-0.8473498790649482, 0.5310350105676815);
\coordinate (la_3_1) at ($(p_1) + {-0.8550965992932691}*(_la_3)$);
\coordinate (la_3_2) at ($(p_1) + {0.3250534481741377}*(_la_3)$);
\coordinate (_la_4) at (-0.9108389669614937, -0.4127618880959325);
\coordinate (la_4_1) at ($(p_1) + {-0.44865334068090984}*(_la_4)$);
\coordinate (la_4_2) at ($(p_1) + {0.30239593384858354}*(_la_4)$);
\draw[yellow] (la_1_1) -- (la_1_2);
\draw[yellow] (la_2_1) -- (la_2_2);
\draw[yellow] (la_3_1) -- (la_3_2);
\draw[yellow] (la_4_1) -- (la_4_2);
\draw[violet] (lf_1_1) -- (lf_1_2);
\filldraw[fill=red] (p_1) circle[radius=1pt];
\filldraw[fill=red] (p_2) circle[radius=1pt];
\filldraw[fill=red] (p_3) circle[radius=1pt];
\filldraw[fill=red] (p_4) circle[radius=1pt];
\draw (0, 0) -- (1, 0) -- (1, 1) -- (0, 1) -- cycle;
\node[below] at (.5, 0) {\scriptsize$(4{,}0{,}1{,}4)$};
\node[above] at (.5, 1) {$m = 3$};
\end{tikzpicture}\vspace{-4pt}
 & \begin{tikzpicture}[scale=1.4]
\coordinate (p_1) at (0.275434, 0.814813);
\coordinate (p_2) at (0.886218, 0.58854);
\coordinate (p_3) at (0.608196, 0.162511);
\coordinate (p_4) at (0.114101, 0.340228);
\coordinate (lf_1_1) at (0.6095855057649198, 0.0);
\coordinate (lf_1_2) at (0.0, 0.7160518693088268);
\coordinate (_la_1) at (0.6156411091665215, 0.7880266649703012);
\coordinate (la_1_1) at ($(p_1) + {-0.44739377520272017}*(_la_1)$);
\coordinate (la_1_2) at ($(p_1) + {0.235000931100446}*(_la_1)$);
\coordinate (_la_2) at (-0.20904233124028718, 0.9779065925484018);
\coordinate (la_2_1) at ($(p_1) + {-0.8332217066628176}*(_la_2)$);
\coordinate (la_2_2) at ($(p_1) + {0.18937084728860143}*(_la_2)$);
\coordinate (_la_3) at (-0.8473498790649482, 0.5310350105676815);
\coordinate (la_3_1) at ($(p_1) + {-0.8550965992932691}*(_la_3)$);
\coordinate (la_3_2) at ($(p_1) + {0.3250534481741377}*(_la_3)$);
\coordinate (_la_4) at (-0.9108389669614937, -0.4127618880959325);
\coordinate (la_4_1) at ($(p_1) + {-0.44865334068090984}*(_la_4)$);
\coordinate (la_4_2) at ($(p_1) + {0.30239593384858354}*(_la_4)$);
\draw[yellow] (la_1_1) -- (la_1_2);
\draw[yellow] (la_2_1) -- (la_2_2);
\draw[yellow] (la_3_1) -- (la_3_2);
\draw[yellow] (la_4_1) -- (la_4_2);

\filldraw[fill=red] (p_1) circle[radius=1pt];
\filldraw[fill=red] (p_2) circle[radius=1pt];
\filldraw[fill=red] (p_3) circle[radius=1pt];
\filldraw[fill=red] (p_4) circle[radius=1pt];
\draw (0, 0) -- (1, 0) -- (1, 1) -- (0, 1) -- cycle;
\node[below] at (.5, 0) {\scriptsize$(4{,}0{,}0{,}4)$};
\node[above] at (.5, 1) {\vphantom{$m = 3$}};
\end{tikzpicture}\vspace{-4pt} & 1 point on free line with $0$ in one coordinate
&
\begin{tikzpicture}[scale=1.4]
\coordinate (p_1) at (0.19, 0.19);
\coordinate (p_2) at (0.82, 0.16);
\coordinate (p_3) at (0.45, 0.87);
\coordinate (d_1_2) at (${0.5}*(p_1) + {0.5}*(p_2)$);
\coordinate (lf_1_1) at (1.0, 0.27345849240536807);
\coordinate (lf_1_2) at (0.6001768997016045, 1.0);
\coordinate (lf_2_1) at (0.0, 0.599659354786837);
\coordinate (lf_2_2) at (0.5697869007322278, 0.0);
\coordinate (_la_1) at (-0.5515977714128465, -0.8341102436563054);
\coordinate (la_1_1) at ($(p_1) + {-0.9710946558447495}*(_la_1)$);
\coordinate (la_1_2) at ($(p_1) + {0.2277876353216079}*(_la_1)$);
\draw[dotted, gray] (p_1) -- (p_2);
\draw[yellow] (la_1_1) -- (la_1_2);
\draw[violet] (lf_1_1) -- (lf_1_2);
\draw[violet] (lf_2_1) -- (lf_2_2);
\filldraw[fill=cyan] (d_1_2) circle[radius=1pt];
\filldraw[fill=red] (p_1) circle[radius=1pt];
\filldraw[fill=red] (p_2) circle[radius=1pt];
\filldraw[fill=red] (p_3) circle[radius=1pt];
\draw (0, 0) -- (1, 0) -- (1, 1) -- (0, 1) -- cycle;
\node[below] at (.5, 0) {\scriptsize$(3{,}1{,}2{,}1)$};
\node[above] at (.5, 1) {$m = 5$};
\end{tikzpicture}\vspace{-4pt} 
& \begin{tikzpicture}[scale=1.4]
\coordinate (p_1) at (0.19, 0.19);
\coordinate (p_2) at (0.82, 0.16);
\coordinate (p_3) at (0.45, 0.87);
\coordinate (d_1_2) at (${0.5}*(p_1) + {0.5}*(p_2)$);
\draw[dotted, gray] (p_1) -- (p_2);
\filldraw[fill=cyan] (d_1_2) circle[radius=1pt];
\filldraw[fill=red] (p_1) circle[radius=1pt];
\filldraw[fill=red] (p_2) circle[radius=1pt];
\filldraw[fill=red] (p_3) circle[radius=1pt];
\draw (0, 0) -- (1, 0) -- (1, 1) -- (0, 1) -- cycle;
\node[below] at (.5, 0) {\scriptsize$(3{,}1{,}0{,}0)$};
\node[above] at (.5, 1) {\vphantom{$m = 5$}};
\end{tikzpicture}\vspace{-4pt}  & 1 point on each free line with $0$ in one coordinate\\

\hline
 \begin{tikzpicture}[scale=1.4]
\coordinate (p_1) at (0.19, 0.19);
\coordinate (p_2) at (0.82, 0.16);
\coordinate (p_3) at (0.45, 0.87);
\coordinate (d_1_2) at (${0.5}*(p_1) + {0.5}*(p_2)$);
\coordinate (lf_1_1) at (0.0, 0.8406220952684439);
\coordinate (lf_1_2) at (1.0, 0.532741217482898);
\coordinate (lf_2_1) at (0.0, 0.3859072208661184);
\coordinate (lf_2_2) at (0.8298076883000585, 1.0);
\coordinate (lf_3_1) at (0.0, 0.35686298646938555);
\coordinate (lf_3_2) at (0.6848536202250521, 1.0);
\coordinate (_la_1) at (0.36776792554791676, -0.9299176054566243);
\coordinate (la_1_1) at ($(p_3) + {-0.13979733176055437}*(_la_1)$);
\coordinate (la_1_2) at ($(p_3) + {0.9355667587052484}*(_la_1)$);
\draw[dotted, gray] (p_1) -- (p_2);
\draw[yellow] (la_1_1) -- (la_1_2);
\draw[violet] (lf_1_1) -- (lf_1_2);
\draw[violet] (lf_2_1) -- (lf_2_2);
\draw[violet] (lf_3_1) -- (lf_3_2);
\filldraw[fill=cyan] (d_1_2) circle[radius=1pt];
\filldraw[fill=red] (p_1) circle[radius=1pt];
\filldraw[fill=red] (p_2) circle[radius=1pt];
\filldraw[fill=red] (p_3) circle[radius=1pt];
\draw (0, 0) -- (1, 0) -- (1, 1) -- (0, 1) -- cycle;
\node[below] at (.5, 0) {\scriptsize$(3{,}1{,}3{,}1)$};
\node[above] at (.5, 1) {$m = 3$};
\end{tikzpicture}\vspace{-4pt}
 & 
 \begin{tikzpicture}[scale=1.4]
\coordinate (p_1) at (0.19, 0.19);
\coordinate (p_2) at (0.82, 0.16);
\coordinate (p_3) at (0.45, 0.87);
\coordinate (d_1_2) at (${0.5}*(p_1) + {0.5}*(p_2)$);

\draw[dotted, gray] (p_1) -- (p_2);

\filldraw[fill=cyan] (d_1_2) circle[radius=1pt];
\filldraw[fill=red] (p_1) circle[radius=1pt];
\filldraw[fill=red] (p_2) circle[radius=1pt];
\filldraw[fill=red] (p_3) circle[radius=1pt];
\draw (0, 0) -- (1, 0) -- (1, 1) -- (0, 1) -- cycle;
\node[below] at (.5, 0) {\scriptsize$(3{,}1{,}0{,}0)$};
\node[above] at (.5, 1) {\vphantom{$m = 3$}};
\end{tikzpicture}\vspace{-4pt}
&
1 point on each free line with $0$ in one coordinate
&
\begin{tikzpicture}[scale=1.4]
\coordinate (p_1) at (0.19, 0.19);
\coordinate (p_2) at (0.82, 0.16);
\coordinate (p_3) at (0.45, 0.87);
\coordinate (d_1_2) at (${0.5}*(p_1) + {0.5}*(p_2)$);
\coordinate (lf_1_1) at (0.7335831431657862, 0.0);
\coordinate (lf_1_2) at (0.6647115407204371, 1.0);
\coordinate (_la_1) at (-0.812297910461451, -0.5832427493419533);
\coordinate (la_1_1) at ($(p_1) + {-0.9971710988888972}*(_la_1)$);
\coordinate (la_1_2) at ($(p_1) + {0.23390433183813636}*(_la_1)$);
\coordinate (_la_2) at (-0.7506968045226046, 0.6606468857714767);
\coordinate (la_2_1) at ($(p_3) + {-0.7326526457639114}*(_la_2)$);
\coordinate (la_2_2) at ($(p_3) + {0.19677683010371155}*(_la_2)$);
\coordinate (_la_3) at (-0.7710459831373743, -0.6367794688019707);
\coordinate (la_3_1) at ($(p_3) + {-0.20415231075929702}*(_la_3)$);
\coordinate (la_3_2) at ($(p_3) + {0.5836227797581629}*(_la_3)$);
\draw[dotted, gray] (p_1) -- (p_2);
\draw[yellow] (la_1_1) -- (la_1_2);
\draw[yellow] (la_2_1) -- (la_2_2);
\draw[yellow] (la_3_1) -- (la_3_2);
\draw[violet] (lf_1_1) -- (lf_1_2);
\filldraw[fill=cyan] (d_1_2) circle[radius=1pt];
\filldraw[fill=red] (p_1) circle[radius=1pt];
\filldraw[fill=red] (p_2) circle[radius=1pt];
\filldraw[fill=red] (p_3) circle[radius=1pt];
\draw (0, 0) -- (1, 0) -- (1, 1) -- (0, 1) -- cycle;
\node[below] at (.5, 0) {\scriptsize$(3{,}1{,}1{,}3)$};
\node[above] at (.5, 1) {$m = 5$};
\end{tikzpicture}\vspace{-4pt}
 & \begin{tikzpicture}[scale=1.4]
\coordinate (p_1) at (0.19, 0.19);
\coordinate (p_2) at (0.82, 0.16);
\coordinate (p_3) at (0.45, 0.87);
\coordinate (d_1_2) at (${0.5}*(p_1) + {0.5}*(p_2)$);
\coordinate (_la_1) at (-0.812297910461451, -0.5832427493419533);
\coordinate (la_1_1) at ($(p_1) + {-0.9971710988888972}*(_la_1)$);
\coordinate (la_1_2) at ($(p_1) + {0.23390433183813636}*(_la_1)$);

\draw[dotted, gray] (p_1) -- (p_2);
\draw[yellow] (la_1_1) -- (la_1_2);

\filldraw[fill=cyan] (d_1_2) circle[radius=1pt];
\filldraw[fill=red] (p_1) circle[radius=1pt];
\filldraw[fill=red] (p_2) circle[radius=1pt];
\filldraw[fill=red] (p_3) circle[radius=1pt];
\draw (0, 0) -- (1, 0) -- (1, 1) -- (0, 1) -- cycle;
\node[below] at (.5, 0) {\scriptsize$(3{,}1{,}0{,}1)$};
\node[above] at (.5, 1) {\vphantom{$m = 5$}};
\end{tikzpicture}\vspace{-4pt} & 1 point on each unused (free and adjacent) line with $0$ in one coordinate\\
\hline
 \begin{tikzpicture}[scale=1.4]
\coordinate (p_1) at (0.19, 0.19);
\coordinate (p_2) at (0.82, 0.16);
\coordinate (p_3) at (0.45, 0.87);
\coordinate (d_1_2) at (${0.5}*(p_1) + {0.5}*(p_2)$);
\coordinate (lf_1_1) at (0.62505884415472677, 1.0);
\coordinate (lf_1_2) at (0.33684182036123494, 0.0);
\coordinate (lf_2_1) at (1.0, 0.840951067089406);
\coordinate (lf_2_2) at (0.19975607460139494, 1.0);
\coordinate (lf_3_1) at (0.4646259163179465, 0.0);
\coordinate (lf_3_2) at (0.0, 0.8689414789882421);
\coordinate (_la_1) at (-0.7081326871692482, -0.706079384603785);
\coordinate (la_1_1) at ($(p_1) + {-1.143853425602997}*(_la_1)$);
\coordinate (la_1_2) at ($(p_1) + {0.268311297363666}*(_la_1)$);
\draw[dotted, gray] (p_1) -- (p_2);
\draw[yellow] (la_1_1) -- (la_1_2);
\draw[violet] (lf_1_1) -- (lf_1_2);
\draw[violet] (lf_2_1) -- (lf_2_2);
\draw[violet] (lf_3_1) -- (lf_3_2);
\filldraw[fill=cyan] (d_1_2) circle[radius=1pt];
\filldraw[fill=red] (p_1) circle[radius=1pt];
\filldraw[fill=red] (p_2) circle[radius=1pt];
\filldraw[fill=red] (p_3) circle[radius=1pt];
\draw (0, 0) -- (1, 0) -- (1, 1) -- (0, 1) -- cycle;
\node[below] at (.5, 0) {\scriptsize$(3{,}1{,}3{,}1)$};
\node[above] at (.5, 1) {$m = 3$};
\end{tikzpicture}\vspace{-4pt}
& 
 \begin{tikzpicture}[scale=1.4]
\coordinate (p_1) at (0.19, 0.19);
\coordinate (p_2) at (0.82, 0.16);
\coordinate (p_3) at (0.45, 0.87);
\coordinate (d_1_2) at (${0.5}*(p_1) + {0.5}*(p_2)$);

\draw[dotted, gray] (p_1) -- (p_2);

\filldraw[fill=cyan] (d_1_2) circle[radius=1pt];
\filldraw[fill=red] (p_1) circle[radius=1pt];
\filldraw[fill=red] (p_2) circle[radius=1pt];
\filldraw[fill=red] (p_3) circle[radius=1pt];
\draw (0, 0) -- (1, 0) -- (1, 1) -- (0, 1) -- cycle;
\node[below] at (.5, 0) {\scriptsize$(3{,}1{,}0{,}0)$};
\node[above] at (.5, 1) {\vphantom{$m = 3$}};
\end{tikzpicture}\vspace{-4pt}& 1 point on each free line with $0$ in one coordinate & 
\begin{tikzpicture}[scale=1.4]
\coordinate (p_1) at (0.19, 0.19);
\coordinate (p_2) at (0.82, 0.16);
\coordinate (p_3) at (0.45, 0.87);
\coordinate (lf_1_1) at (1.0, 0.331643020594088);
\coordinate (lf_1_2) at (0.0, 0.4358379123323587);
\coordinate (lf_2_1) at (1.0, 0.8264838821226361);
\coordinate (lf_2_2) at (0.0, 0.5375592811767641);
\coordinate (lf_3_1) at (1.0, 0.6869798522231816);
\coordinate (lf_3_2) at (0.0, 0.2288459080301494);
\draw[violet] (lf_1_1) -- (lf_1_2);
\draw[violet] (lf_2_1) -- (lf_2_2);
\draw[violet] (lf_3_1) -- (lf_3_2);
\filldraw[fill=red] (p_1) circle[radius=1pt];
\filldraw[fill=red] (p_2) circle[radius=1pt];
\filldraw[fill=red] (p_3) circle[radius=1pt];
\draw (0, 0) -- (1, 0) -- (1, 1) -- (0, 1) -- cycle;
\node[below] at (.5, 0) {\scriptsize$(3{,}0{,}3{,}0)$};
\node[above] at (.5, 1) {$m = 6$};
\end{tikzpicture}\vspace{-4pt}
& \begin{tikzpicture}[scale=1.4]
\coordinate (p_1) at (0.19, 0.19);
\coordinate (p_2) at (0.82, 0.16);
\coordinate (p_3) at (0.45, 0.87);

\filldraw[fill=red] (p_1) circle[radius=1pt];
\filldraw[fill=red] (p_2) circle[radius=1pt];
\filldraw[fill=red] (p_3) circle[radius=1pt];
\draw (0, 0) -- (1, 0) -- (1, 1) -- (0, 1) -- cycle;
\node[below] at (.5, 0) {\scriptsize$(3{,}0{,}0{,}0)$};
\node[above] at (.5, 1) {\vphantom{$m = 6$}};
\end{tikzpicture}\vspace{-4pt}& 1 point on each free line with $0$ in one coordinate\\

\hline
\begin{tikzpicture}[scale=1.4]
\coordinate (p_1) at (0.19, 0.19);
\coordinate (p_2) at (0.82, 0.16);
\coordinate (p_3) at (0.45, 0.87);
\coordinate (d_1_2) at (${0.5}*(p_1) + {0.5}*(p_2)$);
\coordinate (lf_1_1) at (0.0, 0.6991599945836099);
\coordinate (lf_1_2) at (0.41990149856001313, 0.0);
\coordinate (_la_1) at (0.6882638792439697, 0.72546042795458);
\coordinate (la_1_1) at ($(p_1) + {-0.2619026382123983}*(_la_1)$);
\coordinate (la_1_2) at ($(p_1) + {1.1165322997475928}*(_la_1)$);
\coordinate (_la_2) at (0.06556145877108042, 0.997848533157116);
\coordinate (la_2_1) at ($(p_1) + {-0.19040966007020588}*(_la_2)$);
\coordinate (la_2_2) at ($(p_1) + {0.8117464455624567}*(_la_2)$);
\coordinate (_la_3) at (-0.5469974313113428, 0.8371342844124787);
\coordinate (la_3_1) at ($(p_1) + {-0.22696478156231129}*(_la_3)$);
\coordinate (la_3_2) at ($(p_1) + {0.34735080847546945}*(_la_3)$);
\coordinate (_la_4) at (-0.967669955622772, 0.2522198584271716);
\coordinate (la_4_1) at ($(p_1) + {-0.7533110246942052}*(_la_4)$);
\coordinate (la_4_2) at ($(p_1) + {0.19634793753384644}*(_la_4)$);
\coordinate (_la_5) at (-0.9863238439299868, -0.16481891546535204);
\coordinate (la_5_1) at ($(p_1) + {-0.8212312872540645}*(_la_5)$);
\coordinate (la_5_2) at ($(p_1) + {0.19263449947934844}*(_la_5)$);
\draw[dotted, gray] (p_1) -- (p_2);
\draw[yellow] (la_1_1) -- (la_1_2);
\draw[yellow] (la_2_1) -- (la_2_2);
\draw[yellow] (la_3_1) -- (la_3_2);
\draw[yellow] (la_4_1) -- (la_4_2);
\draw[yellow] (la_5_1) -- (la_5_2);
\draw[violet] (lf_1_1) -- (lf_1_2);
\filldraw[fill=cyan] (d_1_2) circle[radius=1pt];
\filldraw[fill=red] (p_1) circle[radius=1pt];
\filldraw[fill=red] (p_2) circle[radius=1pt];
\filldraw[fill=red] (p_3) circle[radius=1pt];
\draw (0, 0) -- (1, 0) -- (1, 1) -- (0, 1) -- cycle;
\node[below] at (.5, 0) {\scriptsize$(3{,}1{,}1{,}5)$};
\node[above] at (.5, 1) {$m = 3$};
\end{tikzpicture}\vspace{-4pt}
  & \begin{tikzpicture}[scale=1.4]
\coordinate (p_1) at (0.19, 0.19);
\coordinate (p_2) at (0.82, 0.16);
\coordinate (p_3) at (0.45, 0.87);
\coordinate (d_1_2) at (${0.5}*(p_1) + {0.5}*(p_2)$);
\coordinate (_la_1) at (0.6882638792439697, 0.72546042795458);
\coordinate (la_1_1) at ($(p_1) + {-0.2619026382123983}*(_la_1)$);
\coordinate (la_1_2) at ($(p_1) + {1.1165322997475928}*(_la_1)$);
\coordinate (_la_2) at (0.06556145877108042, 0.997848533157116);
\coordinate (la_2_1) at ($(p_1) + {-0.19040966007020588}*(_la_2)$);
\coordinate (la_2_2) at ($(p_1) + {0.8117464455624567}*(_la_2)$);
\coordinate (_la_3) at (-0.5469974313113428, 0.8371342844124787);
\coordinate (la_3_1) at ($(p_1) + {-0.22696478156231129}*(_la_3)$);
\coordinate (la_3_2) at ($(p_1) + {0.34735080847546945}*(_la_3)$);
\coordinate (_la_4) at (-0.967669955622772, 0.2522198584271716);
\coordinate (la_4_1) at ($(p_1) + {-0.7533110246942052}*(_la_4)$);
\coordinate (la_4_2) at ($(p_1) + {0.19634793753384644}*(_la_4)$);
\coordinate (_la_5) at (-0.9863238439299868, -0.16481891546535204);
\coordinate (la_5_1) at ($(p_1) + {-0.8212312872540645}*(_la_5)$);
\coordinate (la_5_2) at ($(p_1) + {0.19263449947934844}*(_la_5)$);
\draw[dotted, gray] (p_1) -- (p_2);
\draw[yellow] (la_1_1) -- (la_1_2);
\draw[yellow] (la_2_1) -- (la_2_2);
\draw[yellow] (la_3_1) -- (la_3_2);
\draw[yellow] (la_4_1) -- (la_4_2);
\draw[yellow] (la_5_1) -- (la_5_2);
\filldraw[fill=cyan] (d_1_2) circle[radius=1pt];
\filldraw[fill=red] (p_1) circle[radius=1pt];
\filldraw[fill=red] (p_2) circle[radius=1pt];
\filldraw[fill=red] (p_3) circle[radius=1pt];
\draw (0, 0) -- (1, 0) -- (1, 1) -- (0, 1) -- cycle;
\node[below] at (.5, 0) {\scriptsize$(3{,}1{,}0{,}5)$};
\node[above] at (.5, 1) {\vphantom{$m = 3$}};
\end{tikzpicture}\vspace{-4pt} & 1 point on free line with $0$ in one coordinate &
\begin{tikzpicture}[scale=1.4]
\coordinate (p_1) at (0.19, 0.19);
\coordinate (p_2) at (0.82, 0.16);
\coordinate (p_3) at (0.5, 0.87);
\coordinate (lf_1_1) at (0.2709510049703183, 0.0);
\coordinate (lf_1_2) at (1.0, 0.3609671134204344);
\coordinate (_la_1) at (0.28855878430575954, 0.9574621809763465);
\coordinate (la_1_1) at ($(p_1) + {-0.19844125833383056}*(_la_1)$);
\coordinate (la_1_2) at ($(p_1) + {0.845986417107383}*(_la_1)$);k\coordinate (_la_2) at (-0.7924278488582376, 0.6099656583398004);
\coordinate (la_2_1) at ($(p_1) + {-0.31149294620477563}*(_la_2)$);
\coordinate (la_2_2) at ($(p_1) + {0.23976946326881338}*(_la_2)$);
\coordinate (_la_3) at (-0.8633934166464874, -0.5045312756326461);
\coordinate (la_3_1) at ($(p_1) + {-0.9381586474751301}*(_la_3)$);
\coordinate (la_3_2) at ($(p_1) + {0.2200619049633021}*(_la_3)$);
\coordinate (_la_4) at (0.05817677888625933, -0.9983062968840872);
\coordinate (la_4_1) at ($(p_2) + {-0.8414251243549272}*(_la_4)$);
\coordinate (la_4_2) at ($(p_2) + {0.16027145225808137}*(_la_4)$);
\draw[yellow] (la_1_1) -- (la_1_2);
\draw[yellow] (la_2_1) -- (la_2_2);
\draw[yellow] (la_3_1) -- (la_3_2);
\draw[yellow] (la_4_1) -- (la_4_2);
\draw[violet] (lf_1_1) -- (lf_1_2);
\filldraw[fill=red] (p_1) circle[radius=1pt];
\filldraw[fill=red] (p_2) circle[radius=1pt];
\filldraw[fill=red] (p_3) circle[radius=1pt];
\draw (0, 0) -- (1, 0) -- (1, 1) -- (0, 1) -- cycle;
\node[below] at (.5, 0) {\scriptsize$(3{,}0{,}1{,}4)$};
\node[above] at (.5, 1) {$m = 6$};
\end{tikzpicture}\vspace{-4pt}
 & \begin{tikzpicture}[scale=1.4]
\coordinate (p_1) at (0.19, 0.19);
\coordinate (p_2) at (0.82, 0.16);
\coordinate (p_3) at (0.5, 0.87);
\coordinate (_la_4) at (0.05817677888625933, -0.9983062968840872);
\coordinate (la_4_1) at ($(p_2) + {-0.8414251243549272}*(_la_4)$);
\coordinate (la_4_2) at ($(p_2) + {0.16027145225808137}*(_la_4)$);
\draw[yellow] (la_4_1) -- (la_4_2);

\filldraw[fill=red] (p_1) circle[radius=1pt];
\filldraw[fill=red] (p_2) circle[radius=1pt];
\filldraw[fill=red] (p_3) circle[radius=1pt];
\draw (0, 0) -- (1, 0) -- (1, 1) -- (0, 1) -- cycle;
\node[below] at (.5, 0) {\scriptsize$(3{,}0{,}0{,}1)$};
\node[above] at (.5, 1) {\vphantom{$m = 6$}};
\end{tikzpicture}\vspace{-4pt} & 1 point on each unused (free and adjacent) line with $0$ in one coordinate\\

\hline
\begin{tikzpicture}[scale=1.4]
\coordinate (p_1) at (0.275434, 0.814813);
\coordinate (p_2) at (0.886218, 0.58854);
\coordinate (p_3) at (0.608196, 0.162511);
\coordinate (p_4) at (0.114101, 0.340228);
\coordinate (d_1_2) at (${0.5}*(p_1) + {0.5}*(p_2)$);
\coordinate (lf_1_1) at (0.621051087803388, 1.0);
\coordinate (lf_1_2) at (0.0, 0.3978521351287555);
\coordinate (_la_1) at (-0.7994199556831826, -0.6007726146018129);
\coordinate (la_1_1) at ($(p_3) + {-0.4901103571591044}*(_la_1)$);
\coordinate (la_1_2) at ($(p_3) + {0.2705033419469543}*(_la_1)$);
\draw[dotted, gray] (p_1) -- (p_2);
\draw[yellow] (la_1_1) -- (la_1_2);
\draw[violet] (lf_1_1) -- (lf_1_2);
\filldraw[fill=cyan] (d_1_2) circle[radius=1pt];
\filldraw[fill=red] (p_1) circle[radius=1pt];
\filldraw[fill=red] (p_2) circle[radius=1pt];
\filldraw[fill=red] (p_3) circle[radius=1pt];
\filldraw[fill=red] (p_4) circle[radius=1pt];
\draw (0, 0) -- (1, 0) -- (1, 1) -- (0, 1) -- cycle;
\node[below] at (.5, 0) {\scriptsize$(4{,}1{,}1{,}1)$};
\node[above] at (.5, 1) {$m = 4$};
\end{tikzpicture}\vspace{-4pt}
 & \begin{tikzpicture}[scale=1.4]
\coordinate (p_1) at (0.275434, 0.814813);
\coordinate (p_2) at (0.886218, 0.58854);
\coordinate (p_3) at (0.608196, 0.162511);
\coordinate (p_4) at (0.114101, 0.340228);
\coordinate (d_1_2) at (${0.5}*(p_1) + {0.5}*(p_2)$);

\draw[dotted, gray] (p_1) -- (p_2);
\filldraw[fill=cyan] (d_1_2) circle[radius=1pt];
\filldraw[fill=red] (p_1) circle[radius=1pt];
\filldraw[fill=red] (p_2) circle[radius=1pt];
\filldraw[fill=red] (p_3) circle[radius=1pt];
\filldraw[fill=red] (p_4) circle[radius=1pt];
\draw (0, 0) -- (1, 0) -- (1, 1) -- (0, 1) -- cycle;
\node[below] at (.5, 0) {\scriptsize$(4{,}1{,}0{,}0)$};
\node[above] at (.5, 1) {\vphantom{$m = 4$}};
\end{tikzpicture}\vspace{-4pt} & 1 point on each unused (free and adjacent) line with $0$ in one coordinate & 
\begin{tikzpicture}[scale=1.4]
\coordinate (p_1) at (0.2, 0.2);
\coordinate (p_2) at (0.8, 0.8);
\coordinate (d_1_2) at (${0.5}*(p_1) + {0.5}*(p_2)$);
\coordinate (lf_1_1) at (0.6845372775530156, 1.0);
\coordinate (lf_1_2) at (0.0, 0.35533961687210924);
\coordinate (lf_2_1) at (1.0, 0.1388291109261966);
\coordinate (lf_2_2) at (0.5080434215673772, 1.0);
\coordinate (lf_3_1) at (0.11724578779266981, 1.0);
\coordinate (lf_3_2) at (1.0, 0.6384987560272661);
\coordinate (_la_1) at (-0.3627173627365738, -0.9318991977512507);
\coordinate (la_1_1) at ($(p_1) + {-0.8584619473119687}*(_la_1)$);
\coordinate (la_1_2) at ($(p_1) + {0.21461548682799217}*(_la_1)$);
\draw[dotted, gray] (p_1) -- (p_2);
\draw[yellow] (la_1_1) -- (la_1_2);
\draw[violet] (lf_1_1) -- (lf_1_2);
\draw[violet] (lf_2_1) -- (lf_2_2);
\draw[violet] (lf_3_1) -- (lf_3_2);
\filldraw[fill=cyan] (d_1_2) circle[radius=1pt];
\filldraw[fill=red] (p_1) circle[radius=1pt];
\filldraw[fill=red] (p_2) circle[radius=1pt];
\draw (0, 0) -- (1, 0) -- (1, 1) -- (0, 1) -- cycle;
\node[below] at (.5, 0) {\scriptsize$(2{,}1{,}3{,}1)$};
\node[above] at (.5, 1) {$m = 6$};
\end{tikzpicture}\vspace{-4pt} & \begin{tikzpicture}[scale=1.4]
\coordinate (p_1) at (0.2, 0.2);
\coordinate (p_2) at (0.8, 0.8);
\coordinate (d_1_2) at (${0.5}*(p_1) + {0.5}*(p_2)$);
\coordinate (_la_1) at (-0.3627173627365738, -0.9318991977512507);
\coordinate (la_1_1) at ($(p_1) + {-0.8584619473119687}*(_la_1)$);
\coordinate (la_1_2) at ($(p_1) + {0.21461548682799217}*(_la_1)$);
\draw[dotted, gray] (p_1) -- (p_2);
\draw[yellow] (la_1_1) -- (la_1_2);
\filldraw[fill=cyan] (d_1_2) circle[radius=1pt];
\filldraw[fill=red] (p_1) circle[radius=1pt];
\filldraw[fill=red] (p_2) circle[radius=1pt];
\draw (0, 0) -- (1, 0) -- (1, 1) -- (0, 1) -- cycle;
\node[below] at (.5, 0) {\scriptsize$(2{,}1{,}0{,}1)$};
\node[above] at (.5, 1) {\vphantom{$m = 6$}};
\end{tikzpicture}\vspace{-4pt} & 1 point on each free line with $0$ in one coordinate\\

\hline

\begin{tikzpicture}[scale=1.4]
\coordinate (p_1) at (0.19, 0.19);
\coordinate (p_2) at (0.82, 0.16);
\coordinate (p_3) at (0.45, 0.87);
\coordinate (lf_1_1) at (0.8206719235797814, 0.0);
\coordinate (lf_1_2) at (0.0, 0.7746536215971568);
\coordinate (_la_1) at (0.4916937870594068, 0.8707681779711399);
\coordinate (la_1_1) at ($(p_1) + {-0.2181981436697578}*(_la_1)$);
\coordinate (la_1_2) at ($(p_1) + {0.9302131388026517}*(_la_1)$);
\coordinate (_la_2) at (-0.17704243589152957, 0.9842032187986349);
\coordinate (la_2_1) at ($(p_1) + {-0.1930495616869888}*(_la_2)$);
\coordinate (la_2_2) at ($(p_1) + {0.8230007629813734}*(_la_2)$);
\coordinate (_la_3) at (-0.8854405183204138, 0.46475271759988335);
\coordinate (la_3_1) at ($(p_1) + {-0.4088195567337715}*(_la_3)$);
\coordinate (la_3_2) at ($(p_1) + {0.21458245479934634}*(_la_3)$);
\coordinate (_la_4) at (-0.9467113443733641, -0.32208326630667655);
\coordinate (la_4_1) at ($(p_1) + {-0.8555934232901219}*(_la_4)$);
\coordinate (la_4_2) at ($(p_1) + {0.20069475361126313}*(_la_4)$);
\coordinate (_la_5) at (-0.9252448372514024, -0.3793705195974325);
\coordinate (la_5_1) at ($(p_2) + {-0.19454310119116228}*(_la_5)$);
\coordinate (la_5_2) at ($(p_2) + {0.42175127411002666}*(_la_5)$);
\draw[yellow] (la_1_1) -- (la_1_2);
\draw[yellow] (la_2_1) -- (la_2_2);
\draw[yellow] (la_3_1) -- (la_3_2);
\draw[yellow] (la_4_1) -- (la_4_2);
\draw[yellow] (la_5_1) -- (la_5_2);
\draw[violet] (lf_1_1) -- (lf_1_2);
\filldraw[fill=red] (p_1) circle[radius=1pt];
\filldraw[fill=red] (p_2) circle[radius=1pt];
\filldraw[fill=red] (p_3) circle[radius=1pt];
\draw (0, 0) -- (1, 0) -- (1, 1) -- (0, 1) -- cycle;
\node[below] at (.5, 0) {\scriptsize$(3{,}0{,}1{,}5)$};
\node[above] at (.5, 1) {$m = 4$};
\end{tikzpicture}\vspace{-4pt}
 & \begin{tikzpicture}[scale=1.4]
\coordinate (p_1) at (0.19, 0.19);
\coordinate (p_2) at (0.82, 0.16);
\coordinate (p_3) at (0.45, 0.87);
\coordinate (_la_5) at (-0.9252448372514024, -0.3793705195974325);
\coordinate (la_5_1) at ($(p_2) + {-0.19454310119116228}*(_la_5)$);
\coordinate (la_5_2) at ($(p_2) + {0.42175127411002666}*(_la_5)$);

\draw[yellow] (la_5_1) -- (la_5_2);

\filldraw[fill=red] (p_1) circle[radius=1pt];
\filldraw[fill=red] (p_2) circle[radius=1pt];
\filldraw[fill=red] (p_3) circle[radius=1pt];
\draw (0, 0) -- (1, 0) -- (1, 1) -- (0, 1) -- cycle;
\node[below] at (.5, 0) {\scriptsize$(3{,}0{,}0{,}1)$};
\node[above] at (.5, 1) {\vphantom{$m = 4$}};
\end{tikzpicture}\vspace{-4pt} & 1 point on each unused (free and adjacent) line with $0$ in one coordinate &
\begin{tikzpicture}[scale=1.4]
\coordinate (p_1) at (0.2, 0.2);
\coordinate (p_2) at (0.8, 0.8);
\coordinate (lf_1_1) at (0.8388600344637012, 0.0);
\coordinate (lf_1_2) at (0.878830250518326, 1.0);
\coordinate (lf_2_1) at (0.3911523950934258, 1.0);
\coordinate (lf_2_2) at (0.0, 0.25376119112246076);
\coordinate (_la_1) at (0.48759256411538765, 0.873071298016022);
\coordinate (la_1_1) at ($(p_1) + {-0.22907636576128718}*(_la_1)$);
\coordinate (la_1_2) at ($(p_1) + {0.9163054630451487}*(_la_1)$);
\coordinate (_la_2) at (-0.3626570288226295, 0.9319226788985997);
\coordinate (la_2_1) at ($(p_1) + {-0.2146100792786496}*(_la_2)$);
\coordinate (la_2_2) at ($(p_1) + {0.55148524392124}*(_la_2)$);
\coordinate (_la_3) at (-0.8311220477817329, 0.5560900481856322);
\coordinate (la_3_1) at ($(p_1) + {-0.35965398167535023}*(_la_3)$);
\coordinate (la_3_2) at ($(p_1) + {0.24063854464431617}*(_la_3)$);
\coordinate (_la_4) at (-0.989144412278352, -0.14694669665737178);
\coordinate (la_4_1) at ($(p_1) + {-0.8087797798476312}*(_la_4)$);
\coordinate (la_4_2) at ($(p_1) + {0.2021949449619078}*(_la_4)$);
\draw[yellow] (la_1_1) -- (la_1_2);
\draw[yellow] (la_2_1) -- (la_2_2);
\draw[yellow] (la_3_1) -- (la_3_2);
\draw[yellow] (la_4_1) -- (la_4_2);
\draw[violet] (lf_1_1) -- (lf_1_2);
\draw[violet] (lf_2_1) -- (lf_2_2);
\filldraw[fill=red] (p_1) circle[radius=1pt];
\filldraw[fill=red] (p_2) circle[radius=1pt];
\draw (0, 0) -- (1, 0) -- (1, 1) -- (0, 1) -- cycle;
\node[below] at (.5, 0) {\scriptsize$(2{,}0{,}2{,}4)$};
\node[above] at (.5, 1) {$m = 7$};
\end{tikzpicture}\vspace{-4pt}
 &  
\begin{tikzpicture}[scale=1.4]
\coordinate (p_1) at (0.2, 0.2);
\coordinate (p_2) at (0.8, 0.8);
\coordinate (lf_2_1) at (0.3911523950934258, 1.0);
\coordinate (lf_2_2) at (0.0, 0.25376119112246076);

\draw[violet] (lf_2_1) -- (lf_2_2);
\filldraw[fill=red] (p_1) circle[radius=1pt];
\filldraw[fill=red] (p_2) circle[radius=1pt];
\draw (0, 0) -- (1, 0) -- (1, 1) -- (0, 1) -- cycle;
\node[below] at (.5, 0) {\scriptsize$(2{,}0{,}1{,}0)$};
\node[above] at (.5, 1) {\vphantom{$m = 7$}};
\end{tikzpicture}\vspace{-4pt} & 1 point on each unused (free and adjacent) line with $0$ in one coordinate\\

\hline

\begin{tikzpicture}[scale=1.4]
\coordinate (p_1) at (0.275434, 0.814813);
\coordinate (p_2) at (0.886218, 0.58854);
\coordinate (p_3) at (0.608196, 0.162511);
\coordinate (p_4) at (0.114101, 0.340228);
\coordinate (lf_1_1) at (0.0, 0.337163547097473);
\coordinate (lf_1_2) at (0.8075788055878934, 1.0);
\coordinate (_la_1) at (-0.3773845533250738, 0.9260566391488344);
\coordinate (la_1_1) at ($(p_1) + {-0.8798738279646895}*(_la_1)$);
\coordinate (la_1_2) at ($(p_1) + {0.1999737296524441}*(_la_1)$);
\coordinate (_la_2) at (-0.6526129669181993, -0.7576914381265143);
\coordinate (la_2_1) at ($(p_1) + {-0.2444095190753346}*(_la_2)$);
\coordinate (la_2_2) at ($(p_1) + {0.42204800388914715}*(_la_2)$);
\draw[yellow] (la_1_1) -- (la_1_2);
\draw[yellow] (la_2_1) -- (la_2_2);
\draw[violet] (lf_1_1) -- (lf_1_2);
\filldraw[fill=red] (p_1) circle[radius=1pt];
\filldraw[fill=red] (p_2) circle[radius=1pt];
\filldraw[fill=red] (p_3) circle[radius=1pt];
\filldraw[fill=red] (p_4) circle[radius=1pt];
\draw (0, 0) -- (1, 0) -- (1, 1) -- (0, 1) -- cycle;
\node[below] at (.5, 0) {\scriptsize$(4{,}0{,}1{,}2)$};
\node[above] at (.5, 1) {$m = 5$};
\end{tikzpicture}\vspace{-4pt}
& 
\begin{tikzpicture}[scale=1.4]
\coordinate (p_1) at (0.275434, 0.814813);
\coordinate (p_2) at (0.886218, 0.58854);
\coordinate (p_3) at (0.608196, 0.162511);
\coordinate (p_4) at (0.114101, 0.340228);

\filldraw[fill=red] (p_1) circle[radius=1pt];
\filldraw[fill=red] (p_2) circle[radius=1pt];
\filldraw[fill=red] (p_3) circle[radius=1pt];
\filldraw[fill=red] (p_4) circle[radius=1pt];
\draw (0, 0) -- (1, 0) -- (1, 1) -- (0, 1) -- cycle;
\node[below] at (.5, 0) {\scriptsize$(4{,}0{,}0{,}0)$};
\node[above] at (.5, 1) {\vphantom{$m = 5$}};
\end{tikzpicture}\vspace{-4pt}
 & 1 point on each unused (free and adjacent) line with $0$ in one coordinate & 
\begin{tikzpicture}[scale=1.4]
\coordinate (p_1) at (0.2, 0.2);
\coordinate (p_2) at (0.8, 0.8);
\coordinate (lf_1_1) at (1.0, 0.6150945143273124);
\coordinate (lf_1_2) at (0.21441782900093448, 1.0);
\coordinate (_la_1) at (0.4052935736167188, 0.9141865888236326);
\coordinate (la_1_1) at ($(p_1) + {-0.21877371911281077}*(_la_1)$);
\coordinate (la_1_2) at ($(p_1) + {0.8750948764512431}*(_la_1)$);
\coordinate (_la_2) at (-0.2373759962398847, 0.9714178485127408);
\coordinate (la_2_1) at ($(p_1) + {-0.2058846255565551}*(_la_2)$);
\coordinate (la_2_2) at ($(p_1) + {0.8235385022262204}*(_la_2)$);
\coordinate (_la_3) at (-0.8396614318161137, 0.5431101913244804);
\coordinate (la_3_1) at ($(p_1) + {-0.36824939615340474}*(_la_3)$);
\coordinate (la_3_2) at ($(p_1) + {0.2381912428291694}*(_la_3)$);
\coordinate (_la_4) at (-0.9691949434273819, -0.24629486725223107);
\coordinate (la_4_1) at ($(p_1) + {-0.8254273357751386}*(_la_4)$);
\coordinate (la_4_2) at ($(p_1) + {0.20635683394378465}*(_la_4)$);
\coordinate (_la_5) at (-0.4959981875401072, 0.8683235560301982);
\coordinate (la_5_1) at ($(p_2) + {-0.40322727990579127}*(_la_5)$);
\coordinate (la_5_2) at ($(p_2) + {0.2303288890541678}*(_la_5)$);
\coordinate (_la_6) at (-0.9174641413228231, -0.3978184879901574);
\coordinate (la_6_1) at ($(p_2) + {-0.21799217102004106}*(_la_6)$);
\coordinate (la_6_2) at ($(p_2) + {0.8719686840801645}*(_la_6)$);
\draw[yellow] (la_1_1) -- (la_1_2);
\draw[yellow] (la_2_1) -- (la_2_2);
\draw[yellow] (la_3_1) -- (la_3_2);
\draw[yellow] (la_4_1) -- (la_4_2);
\draw[yellow] (la_5_1) -- (la_5_2);
\draw[yellow] (la_6_1) -- (la_6_2);
\draw[violet] (lf_1_1) -- (lf_1_2);
\filldraw[fill=red] (p_1) circle[radius=1pt];
\filldraw[fill=red] (p_2) circle[radius=1pt];
\draw (0, 0) -- (1, 0) -- (1, 1) -- (0, 1) -- cycle;
\node[below] at (.5, 0) {\scriptsize$(2{,}0{,}1{,}6)$};
\node[above] at (.5, 1) {$m = 7$};
\end{tikzpicture}\vspace{-4pt}
 & \begin{tikzpicture}[scale=1.4]
\coordinate (p_1) at (0.2, 0.2);
\coordinate (p_2) at (0.8, 0.8);
\coordinate (_la_5) at (-0.4959981875401072, 0.8683235560301982);
\coordinate (la_5_1) at ($(p_2) + {-0.40322727990579127}*(_la_5)$);
\coordinate (la_5_2) at ($(p_2) + {0.2303288890541678}*(_la_5)$);
\coordinate (_la_6) at (-0.9174641413228231, -0.3978184879901574);
\coordinate (la_6_1) at ($(p_2) + {-0.21799217102004106}*(_la_6)$);
\coordinate (la_6_2) at ($(p_2) + {0.8719686840801645}*(_la_6)$);
\draw[yellow] (la_5_1) -- (la_5_2);
\draw[yellow] (la_6_1) -- (la_6_2);

\filldraw[fill=red] (p_1) circle[radius=1pt];
\filldraw[fill=red] (p_2) circle[radius=1pt];
\draw (0, 0) -- (1, 0) -- (1, 1) -- (0, 1) -- cycle;
\node[below] at (.5, 0) {\scriptsize$(2{,}0{,}0{,}2)$};
\node[above] at (.5, 1) {\vphantom{$m = 7$}};
\end{tikzpicture}\vspace{-4pt}  & 1 point on each unused (free and adjacent) line with $0$ in one coordinate\\
\hline
\end{tabular}
\centering
\captionsetup{justification=centering}
\caption{List of non-minimal subproblems with overconstrained subsystems after elimination of the given subproblem.\\
More details can be found in Example \ref{ex:(34112)} which is the first entry of this table.}

\label{tab:elim}
\end{table*}

Among the three remaining non-minimal  balanced PLPs, there is one further case which can be solved using this strategy only requiring a minor adaptation of the argument. Instead of considering a minimal subproblem, we consider a subproblem with a two-dimensional solution set. This, however, does not affect the possibility of splitting the system of equations into one overconstrained subproblem and an underconstrained subproblem. 

\begin{example}\label{ex:53113}
    We consider the balanced PLP where 5 cameras observe 3 free points, 1 dependent point, 1 free line and 3 lines attached to a collinear point. We study the sub-PLP arising from omitting the free line, as shown on the right:
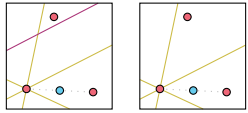
\begin{gather*}
\begin{tikzpicture}[scale=1.4]
\coordinate (p_1) at (0.19, 0.19);
\coordinate (p_2) at (0.82, 0.16);
\coordinate (p_3) at (0.45, 0.87);
\coordinate (d_1_2) at (${0.5}*(p_1) + {0.5}*(p_2)$);
\coordinate (lf_1_1) at (0.0, 0.5515451401310275);
\coordinate (lf_1_2) at (0.8678485811005808, 1.0);
\coordinate (_la_1) at (0.20548094590677488, 0.9786611164592456);
\coordinate (la_1_1) at ($(p_1) + {-0.19414279039450544}*(_la_1)$);
\coordinate (la_1_2) at ($(p_1) + {0.8276613695765759}*(_la_1)$);
\coordinate (_la_2) at (-0.9067983119161803, 0.42156472991222327);
\coordinate (la_2_1) at ($(p_1) + {-0.45070184130337737}*(_la_2)$);
\coordinate (la_2_2) at ($(p_1) + {0.20952840064126918}*(_la_2)$);
\coordinate (_la_3) at (-0.890494226934138, -0.4549945404034776);
\coordinate (la_3_1) at ($(p_1) + {-0.9096072444947008}*(_la_3)$);
\coordinate (la_3_2) at ($(p_1) + {0.2133646622888804}*(_la_3)$);
\draw[dotted, gray] (p_1) -- (p_2);
\draw[yellow] (la_1_1) -- (la_1_2);
\draw[yellow] (la_2_1) -- (la_2_2);
\draw[yellow] (la_3_1) -- (la_3_2);
\draw[violet] (lf_1_1) -- (lf_1_2);
\filldraw[fill=cyan] (d_1_2) circle[radius=1pt];
\filldraw[fill=red] (p_1) circle[radius=1pt];
\filldraw[fill=red] (p_2) circle[radius=1pt];
\filldraw[fill=red] (p_3) circle[radius=1pt];
\draw (0, 0) -- (1, 0) -- (1, 1) -- (0, 1) -- cycle;
\end{tikzpicture}
\quad
\begin{tikzpicture}[scale=1.4]
\coordinate (p_1) at (0.19, 0.19);
\coordinate (p_2) at (0.82, 0.16);
\coordinate (p_3) at (0.45, 0.87);
\coordinate (d_1_2) at (${0.5}*(p_1) + {0.5}*(p_2)$);
\coordinate (_la_1) at (0.20548094590677488, 0.9786611164592456);
\coordinate (la_1_1) at ($(p_1) + {-0.19414279039450544}*(_la_1)$);
\coordinate (la_1_2) at ($(p_1) + {0.8276613695765759}*(_la_1)$);
\coordinate (_la_2) at (-0.9067983119161803, 0.42156472991222327);
\coordinate (la_2_1) at ($(p_1) + {-0.45070184130337737}*(_la_2)$);
\coordinate (la_2_2) at ($(p_1) + {0.20952840064126918}*(_la_2)$);
\coordinate (_la_3) at (-0.890494226934138, -0.4549945404034776);
\coordinate (la_3_1) at ($(p_1) + {-0.9096072444947008}*(_la_3)$);
\coordinate (la_3_2) at ($(p_1) + {0.2133646622888804}*(_la_3)$);
\draw[dotted, gray] (p_1) -- (p_2);
\draw[yellow] (la_1_1) -- (la_1_2);
\draw[yellow] (la_2_1) -- (la_2_2);
\draw[yellow] (la_3_1) -- (la_3_2);
\filldraw[fill=cyan] (d_1_2) circle[radius=1pt];
\filldraw[fill=red] (p_1) circle[radius=1pt];
\filldraw[fill=red] (p_2) circle[radius=1pt];
\filldraw[fill=red] (p_3) circle[radius=1pt];
\draw (0, 0) -- (1, 0) -- (1, 1) -- (0, 1) -- cycle;
\end{tikzpicture}
\end{gather*}

We choose the following normalized point-line variety $\XX':=\lbrace (e_1,e_2,e_3,e_1+e_2,\overline{e_1e_4},\overline{e_1e})\rbrace\times \lbrace \overline{e_1Q}\mid Q=\lambda e_2+\mu e_3+\nu e_4, (\lambda:\mu:\nu) \in \PP^2\rbrace$, with stabilizer and reduced camera variety
\begin{gather}
  \Stab(\XX')
=
  \biggl\{
    \begin{bsmallmatrix}
      1 & 0 &0 &\lambda\\
      0 & 1 &0 & 0\\
      0& 0& 1& 0\\
      0& 0& 0& 1
    \end{bsmallmatrix}
  \biggr\}
\tx{,}\;\;
  \CC'
:=
  \Bigl\{
    \begin{bsmallmatrix}
      c_1 & c_2 &1 &0\\
      c_3 & c_4 &c_5 & c_6\\
      c_7& c_8& c_9& c_{10}
    \end{bsmallmatrix}
  \Bigr\}
\tx{.}
\end{gather}
We assume for contradiction that the original PLP is minimal.
Then, by Proposition \ref{Schlüssellemma},
the reduced joint camera map has a full-dimensional image. Thus, the Fiber-Dimension theorem tells us that its generic fiber has dimension $m \cdot \dim \mathcal{C}' + \dim \mathcal{X}' - m \cdot \dim \mathcal{Y} =  5\cdot 10+2-5\cdot 10 =2$. To consider the  free line from the original PLP, we  reintroduce the stabilizer, which results in cameras of the form \begin{gather}
    \begin{bmatrix}
        c_1^i & c_2^i &1 &c_1^i\lambda^i\\
        c_3^i & c_4^i &c_5^i & c_3^i\lambda^i+c_6^i\\
        c_7^i& c_8^i& c_9^i& c_7^i\lambda^i+c_{10}^i
    \end{bmatrix}, 
\end{gather}
where $\lambda^1=1$ to fix the $\PGL{4}$-action. The only difference to Example \ref{ex:(34112)} is that we do not have a finite set of these $c_j^i$ but that they form a 2-dimensional set. Now we consider the points $\begin{bmatrix}
    x_1^1 &
    x_2^1&
    1&
    0
\end{bmatrix}^\top$ and $\begin{bmatrix}
    x_1^2 &
    x_2^2&
    0&
    1
\end{bmatrix}^\top$ on the generic free line with coefficient vector $\begin{bmatrix}
    y_1^i&
    y_2^i&
    1
\end{bmatrix}^\top$ on the $i$-th image. They give rise to the remaining 10 equations: $$
\begin{bmatrix}
    y_1^i&
    y_2^i&
    1
\end{bmatrix} C_i  \begin{bmatrix}
    x_1^j &
    x_2^j&
    0&
    1
\end{bmatrix}^\top = 0
$$ 
for $i$ from $1$ to $5$ and $j=1,2$. The variables $\lambda^i$ for $i$ from $2$ to $5$ and $x_1^2,x_2^2$ do not appear within the first five constraints, and thus these variables are underconstrained.
Hence, if the original PLP has a solution, it must have infinitely many, which contradicts the minimality assumption.
\hfill $\diamondsuit$
\end{example}

For the last two balanced problems, we need to take a more detailed look in order to understand the structure of their equation systems and identify the overconstrained and underconstrained subsystems.
 Fortunately, the equations in both cases are very similar and  we can take care of them simultaneously. 

\begin{example}
The first scenario consists of 4 cameras observing 2 free points, 1 dependent point, 1 free line, and 6 adjacent lines, 5 of which are attached to a single point. In the other case, we have 6 cameras observing 2 free points, 1 dependent point, 1 free line and 5 adjacent lines, 4 of which are attached to a single point. In either case, we consider the subproblem that arises from omitting the free line and the single attached line.
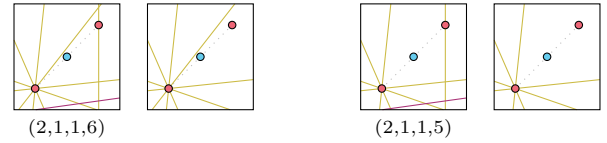
\begin{gather*}
\begin{tikzpicture}[scale=1.4]
\coordinate (p_1) at (0.2, 0.2);
\coordinate (p_2) at (0.8, 0.8);
\coordinate (d_1_2) at (${0.5}*(p_1) + {0.5}*(p_2)$);
\coordinate (lf_1_1) at (0.20585596352080798, 0.0);
\coordinate (lf_1_2) at (1.0, 0.11087046242297226);
\coordinate (_la_1) at (0.6158472936696543, 0.7878655411171138);
\coordinate (la_1_1) at ($(p_1) + {-0.2538504218834348}*(_la_1)$);
\coordinate (la_1_2) at ($(p_1) + {1.0154016875337393}*(_la_1)$);
\coordinate (_la_2) at (0.10571287279793859, 0.9943966957531621);
\coordinate (la_2_1) at ($(p_1) + {-0.201126975636739}*(_la_2)$);
\coordinate (la_2_2) at ($(p_1) + {0.804507902546956}*(_la_2)$);
\coordinate (_la_3) at (-0.4000361921353165, 0.9164993425976236);
\coordinate (la_3_1) at ($(p_1) + {-0.2182216513468982}*(_la_3)$);
\coordinate (la_3_2) at ($(p_1) + {0.49995476392382987}*(_la_3)$);
\coordinate (_la_4) at (-0.9456119088599442, 0.32529696866440755);
\coordinate (la_4_1) at ($(p_1) + {-0.6148228211936703}*(_la_4)$);
\coordinate (la_4_2) at ($(p_1) + {0.2115032584997005}*(_la_4)$);
\coordinate (_la_5) at (-0.9944945656066786, -0.10478816239816208);
\coordinate (la_5_1) at ($(p_1) + {-0.8044287295948875}*(_la_5)$);
\coordinate (la_5_2) at ($(p_1) + {0.20110718239872188}*(_la_5)$);
\coordinate (_la_6) at (-0.6627433643687954, -0.7488466017717715);
\coordinate (la_6_1) at ($(0.8, 0.0)$);
\coordinate (la_6_2) at ($(0.8, 1.0)$);
\draw[dotted, gray] (p_1) -- (p_2);
\draw[yellow] (la_1_1) -- (la_1_2);
\draw[yellow] (la_2_1) -- (la_2_2);
\draw[yellow] (la_3_1) -- (la_3_2);
\draw[yellow] (la_4_1) -- (la_4_2);
\draw[yellow] (la_5_1) -- (la_5_2);
\draw[yellow] (la_6_1) -- (la_6_2);
\draw[violet] (lf_1_1) -- (lf_1_2);
\filldraw[fill=cyan] (d_1_2) circle[radius=1pt];
\filldraw[fill=red] (p_1) circle[radius=1pt];
\filldraw[fill=red] (p_2) circle[radius=1pt];
\draw (0, 0) -- (1, 0) -- (1, 1) -- (0, 1) -- cycle;
\node[below] at (.5, 0) {\scriptsize$(2{,}1{,}1{,}6)$};
\end{tikzpicture}
\quad
\begin{tikzpicture}[scale=1.4]
\coordinate (p_1) at (0.2, 0.2);
\coordinate (p_2) at (0.8, 0.8);
\coordinate (d_1_2) at (${0.5}*(p_1) + {0.5}*(p_2)$);
\coordinate (_la_1) at (0.6158472936696543, 0.7878655411171138);
\coordinate (la_1_1) at ($(p_1) + {-0.2538504218834348}*(_la_1)$);
\coordinate (la_1_2) at ($(p_1) + {1.0154016875337393}*(_la_1)$);
\coordinate (_la_2) at (0.10571287279793859, 0.9943966957531621);
\coordinate (la_2_1) at ($(p_1) + {-0.201126975636739}*(_la_2)$);
\coordinate (la_2_2) at ($(p_1) + {0.804507902546956}*(_la_2)$);
\coordinate (_la_3) at (-0.4000361921353165, 0.9164993425976236);
\coordinate (la_3_1) at ($(p_1) + {-0.2182216513468982}*(_la_3)$);
\coordinate (la_3_2) at ($(p_1) + {0.49995476392382987}*(_la_3)$);
\coordinate (_la_4) at (-0.9456119088599442, 0.32529696866440755);
\coordinate (la_4_1) at ($(p_1) + {-0.6148228211936703}*(_la_4)$);
\coordinate (la_4_2) at ($(p_1) + {0.2115032584997005}*(_la_4)$);
\coordinate (_la_5) at (-0.9944945656066786, -0.10478816239816208);
\coordinate (la_5_1) at ($(p_1) + {-0.8044287295948875}*(_la_5)$);
\coordinate (la_5_2) at ($(p_1) + {0.20110718239872188}*(_la_5)$);
\draw[dotted, gray] (p_1) -- (p_2);
\draw[yellow] (la_1_1) -- (la_1_2);
\draw[yellow] (la_2_1) -- (la_2_2);
\draw[yellow] (la_3_1) -- (la_3_2);
\draw[yellow] (la_4_1) -- (la_4_2);
\draw[yellow] (la_5_1) -- (la_5_2);
\filldraw[fill=cyan] (d_1_2) circle[radius=1pt];
\filldraw[fill=red] (p_1) circle[radius=1pt];
\filldraw[fill=red] (p_2) circle[radius=1pt];
\draw (0, 0) -- (1, 0) -- (1, 1) -- (0, 1) -- cycle;
\node[below] at (.5, 0) {\vphantom{\scriptsize$(2{,}1{,}1{,}6)$}};
\end{tikzpicture}
\qquad\qquad
\begin{tikzpicture}[scale=1.4]
\coordinate (p_1) at (0.2, 0.2);
\coordinate (p_2) at (0.8, 0.8);
\coordinate (d_1_2) at (${0.5}*(p_1) + {0.5}*(p_2)$);
\coordinate (lf_1_1) at (0.20585596352080798, 0.0);
\coordinate (lf_1_2) at (1.0, 0.11087046242297226);
\coordinate (_la_2) at (0.10571287279793859, 0.9943966957531621);
\coordinate (la_2_1) at ($(p_1) + {-0.201126975636739}*(_la_2)$);
\coordinate (la_2_2) at ($(p_1) + {0.804507902546956}*(_la_2)$);
\coordinate (_la_3) at (-0.4000361921353165, 0.9164993425976236);
\coordinate (la_3_1) at ($(p_1) + {-0.2182216513468982}*(_la_3)$);
\coordinate (la_3_2) at ($(p_1) + {0.49995476392382987}*(_la_3)$);
\coordinate (_la_4) at (-0.9456119088599442, 0.32529696866440755);
\coordinate (la_4_1) at ($(p_1) + {-0.6148228211936703}*(_la_4)$);
\coordinate (la_4_2) at ($(p_1) + {0.2115032584997005}*(_la_4)$);
\coordinate (_la_5) at (-0.9944945656066786, -0.10478816239816208);
\coordinate (la_5_1) at ($(p_1) + {-0.8044287295948875}*(_la_5)$);
\coordinate (la_5_2) at ($(p_1) + {0.20110718239872188}*(_la_5)$);
\coordinate (_la_6) at (-0.6627433643687954, -0.7488466017717715);
\coordinate (la_6_1) at ($(0.8, 0.0)$);
\coordinate (la_6_2) at ($(0.8, 1.0)$);
\draw[dotted, gray] (p_1) -- (p_2);
\draw[yellow] (la_2_1) -- (la_2_2);
\draw[yellow] (la_3_1) -- (la_3_2);
\draw[yellow] (la_4_1) -- (la_4_2);
\draw[yellow] (la_5_1) -- (la_5_2);
\draw[yellow] (la_6_1) -- (la_6_2);
\draw[violet] (lf_1_1) -- (lf_1_2);
\filldraw[fill=cyan] (d_1_2) circle[radius=1pt];
\filldraw[fill=red] (p_1) circle[radius=1pt];
\filldraw[fill=red] (p_2) circle[radius=1pt];
\draw (0, 0) -- (1, 0) -- (1, 1) -- (0, 1) -- cycle;
\node[below] at (.5, 0) {\scriptsize$(2{,}1{,}1{,}5)$};
\end{tikzpicture}  
\quad
\begin{tikzpicture}[scale=1.4]
\coordinate (p_1) at (0.2, 0.2);
\coordinate (p_2) at (0.8, 0.8);
\coordinate (d_1_2) at (${0.5}*(p_1) + {0.5}*(p_2)$);
\coordinate (_la_2) at (0.10571287279793859, 0.9943966957531621);
\coordinate (la_2_1) at ($(p_1) + {-0.201126975636739}*(_la_2)$);
\coordinate (la_2_2) at ($(p_1) + {0.804507902546956}*(_la_2)$);
\coordinate (_la_3) at (-0.4000361921353165, 0.9164993425976236);
\coordinate (la_3_1) at ($(p_1) + {-0.2182216513468982}*(_la_3)$);
\coordinate (la_3_2) at ($(p_1) + {0.49995476392382987}*(_la_3)$);
\coordinate (_la_4) at (-0.9456119088599442, 0.32529696866440755);
\coordinate (la_4_1) at ($(p_1) + {-0.6148228211936703}*(_la_4)$);
\coordinate (la_4_2) at ($(p_1) + {0.2115032584997005}*(_la_4)$);
\coordinate (_la_5) at (-0.9944945656066786, -0.10478816239816208);
\coordinate (la_5_1) at ($(p_1) + {-0.8044287295948875}*(_la_5)$);
\coordinate (la_5_2) at ($(p_1) + {0.20110718239872188}*(_la_5)$);
\draw[dotted, gray] (p_1) -- (p_2);
\draw[yellow] (la_2_1) -- (la_2_2);
\draw[yellow] (la_3_1) -- (la_3_2);
\draw[yellow] (la_4_1) -- (la_4_2);
\draw[yellow] (la_5_1) -- (la_5_2);
\filldraw[fill=cyan] (d_1_2) circle[radius=1pt];
\filldraw[fill=red] (p_1) circle[radius=1pt];
\filldraw[fill=red] (p_2) circle[radius=1pt];
\draw (0, 0) -- (1, 0) -- (1, 1) -- (0, 1) -- cycle;
\node[below] at (.5, 0) {\vphantom{\scriptsize$(2{,}1{,}1{,}5)$}};
\end{tikzpicture}
\end{gather*}

We will focus on the smaller example with 4 cameras, but all computations work  the same for the 6-camera case. We will point it out whenever a difference occurs. 

We choose the normalized point-line variety to be $\XX':=\lbrace e_1,e_2,e_1+e_2, \overline{e_1e_3},\overline{e_1e_4},\overline{e_1e}\rbrace\times  \lbrace (\overline{e_1Q_1},\overline{e_1Q_2})\mid Q_i=\lambda_i e_2+\mu_i e_3+\nu_i e_4, (\lambda_i:\mu_i:\nu_i) \in \PP^2\rbrace$ (in the 6-camera case,  omit $Q_2$), with stabilizer and reduced camera variety
\begin{gather}
  \Stab(\XX')
=
  \biggl\{
    \begin{bsmallmatrix}
      1 & 0 &\lambda &\mu\\
      0 & 1 &0 & 0\\
      0& 0& 1& 0\\
      0& 0& 0& 1
    \end{bsmallmatrix}
  \biggr\}
\tx{,}\;\;
  \CC'
:=
  \Bigl\{
    \begin{bsmallmatrix}
      c_1 & 1 &0 &0\\
      c_2 & c_3 &c_4 & c_5\\
      c_6& c_7& c_8& c_9
    \end{bsmallmatrix}
  \Bigr\}
\tx{.}
\end{gather}
We assume again for contradiction that the original PLP is minimal.
Then, for 4 cameras the reduced subproblem is balanced and minimal, whereas in case of 6 cameras the fiber of the reduced joint camera map is 2-dimensional. Reintroducing the parameters from the stabilizer, we obtain cameras of the form\begin{gather}
\begin{bmatrix}
       c_1^i & 1 &c_1^i\lambda^i &c_1^i\mu^i\\
        c_2^i & c_3^i &c_2^i\lambda^i+c_4^i &c_2^i\mu^i+ c_5^i\\
        c_6^i& c_7^i& c_6^i\lambda^i+c_8^i&c_6^i\mu^i +c_9^i
\end{bmatrix}
\end{gather}
and we can assume $\lambda^1=\mu^1=0$ to fix the $\PGL{4}$-action. Let the generic free line be spanned by the two points $\begin{bmatrix}
    x_1 &
    x_2&
    1&
    0
\end{bmatrix}^\top$ and $\begin{bmatrix}
    x_3 &
    x_4&
    0&
    1
\end{bmatrix}^\top$. The last adjacent line was attached to the second point which is now $e_2$. Thus, it can be represented by one further point of the form $\begin{bmatrix}
    x_5 &
    0&
    1&
    x_6
\end{bmatrix}^\top$. As there is no immediate over- or underconstrained subsystem, we study the equations directly. Let the lines on the $i$-th image be represented by the coefficient vectors $\begin{bmatrix}
    y_1^i &
    y_2^i&
    1
\end{bmatrix}^\top$ for the free line and $\begin{bmatrix}
    1 &
    z_1^i&
    z_2^i
\end{bmatrix}^\top$ for the adjacent line. Define \begin{itemize}
    \item $\alpha_i:=c_1^iy_1^i+c_2^iy_2^i+c_6^i$
    \item $\beta_i:=y_1^i+c_3^iy_2^i+c_7^i$
    \item $\gamma_i:=c_4^iy_2^i+c_8^i$
    \item $\delta_i:=c_5^iy_2^i+c_9^i$
    \item $\varphi_i:=c_1^i+c_2^iz_1^i+c_6^iz_2^i$
    \item $\psi_i:=z_1^ic_5^i+z_2^ic_9^i$
    \item $\varepsilon_i:=z_1^ic_4^i+z_2^ic_8^i$
\end{itemize}
Now we can write the constraints coming from the lines as\begin{gather}
0=\alpha_ix_1+\beta_ix_2+\alpha_i\lambda^i+\gamma_i \label{eq:lf1}\\
    0=\alpha_ix_3+\beta_ix_4+\alpha_i\mu^i+\delta_i\label{eq:lf2}\\
    0=\varphi_ix_5+\psi_ix_6+\varphi_i\lambda^i+\varphi_i\mu^ix_6+\varepsilon_i\label{eq:la}
\end{gather}
where $i$ reaches from $1$ to $4$ resp. $6$. The main idea is now to use the linearity of these constraints together with the plethora of common coefficients to find the desired subsystem. Using $\lambda^1=\mu^1=0$, we can eliminate the variables $x_2,x_4$ and $x_5$. Next, Equations \eqref{eq:lf1} and \eqref{eq:lf2} for $i>1$ allow to write $\lambda^i$ and $\mu^i$ linearly in terms of $x_1$ and $x_3$, respectively. Inserting this into the $i-1$ last Equations \eqref{eq:la}, we obtain equations of the form $0=a_i'x_6+c_i'x_1+c_i'x_3x_6+b_i'$. Rescaling each of these equations by the inverse of $c_i'$ yields $3$ resp. $5$ equations of the form \begin{gather}
    0=a_ix_6+x_1+x_3x_6+b_i. \label{eq:final}
\end{gather}
Taking differences of these yields $2$ resp. $4$ equations constraining $x_6$ (and the 2-dimensional fiber). But in either case, this corresponds to an overconstrained subsystem, and assuming a solution to this yields a one-dimensional set of possible solutions for $x_1$ and $x_3$ in Equation \eqref{eq:final}. 
As in Example \ref{ex:53113}, this contradicts the minimality of the original PLP.
\hfill $\diamondsuit$
\end{example}

Together with the criteria derived in Section \ref{appendix:nonMinimal}, the $19$ cases discussed in this section formally prove the completeness of the list of minimal problems given in Theorem \ref{thm:main}.

\section{Efficient factorizations}\label{sm:factorisations}
In Remark \ref{rem:positiveEffect}, we explained how to use the stabilizer technique to identify a subproblem and factorize a given minimal problem into two parts, both of which can be easier to solve than the original problem. In Example \ref{ex:positiveEffect}, we have seen that the degree of the problem can factorize by this method. This degree factorization is a good indicator for making a problem easier, but it is not the only way. In the following very similar example (in fact, it is the 4-camera analog to Example \ref{ex:positiveEffect}), we factorize a minimal problem without factorizing its degree and obtain a drastic computational speed-up since the obtained subproblems have much less unknowns.

\begin{example}\label{ex:effFact}
    We consider the scenario where 4 cameras observe 6 lines attached to a single point and additionally three free lines. This is the first example with four cameras reported in Section \ref{sm:minimal}. Computing the degree of this problem using Gröbner bases turned out to be particularly challenging, although it is has only degree 2. Both the naive implementation of the equations and the implementation eliminating the variables of world lines (which was the method of choice in \cite{plmp})  have not terminated within a reasonable amount of time. However, it turned out that factorizing the computation into solving at first the subproblem arising from omitting the free lines and then adding them afterwards  returns the degree immediately. We use basically the same normalization as in Example
    \ref{ex:running1} and Example \ref{ex:positiveEffect}:
    $\XX' := \{ (e_1, \overline{e_1 e_2}, \overline{e_1 e_3}, \overline{e_1 e_4}, \overline{e_1 e}, \overline{e_1 Q_1}, \overline{e_1 Q_2}) \mid Q_i \neq e_1\rbrace$ and also the stabilizers and reduced cameras stay the same.
    This reduced subproblem is minimal of degree 2. After solving it, one has to  add the remaining 3 free lines again and thus needs to reintroduce the stabilizer exactly as in Equation \eqref{ReintroStab}. The $\PGL{4}$-action can be fixed afterward by setting $\lambda_j^1=1$. 
    This remainder problem in the free lines and the missing camera parameters has degree one. An explicit implementation of this example can be found in the attached code\footref{code} (lines 1805-1953). \hfill $\diamondsuit$
\end{example}

Analogously, one can find for many more minimal problems minimal subproblems that allow for a tremendous speedup in computation time.

\section{Camera registration}
\label{sec:singleView}

The two PLPs from Theorem \ref{thm:main} \ref{item:thm:main:b} that are minimal for arbitrarily many cameras are in fact also minimal for a single view. 
Moreover, there are a few other minimal problems for a single projective camera:

\begin{prop}\label{prop:singleView}
    All minimal PLPs for a single projective camera are depicted in Table \ref{tab:m1}. The five right-most PLPs are infinite families as they admit an arbitrary number of (non-depicted) dependent points. All PLPs have degree $1$.
\end{prop}
To prove this, we first classify the balanced problems. For a single camera, the corresponding dimension equation becomes \begin{gather} \nonumber
    11+3p^f+p^d+4\ell^f+2\ell^a-15=2p^f+p^d+2\ell^f+\ell^a \\\iff p^f+2\ell^f+\ell^a=4. \label{eq:singleViewBalanced}
\end{gather}
Each of the appearing parameters above is upper bounded by 4, and so we can perform a simple case distinction to obtain the $8$ possible cases without dependent points shown in Table \ref{tab:m1}.
\begin{table*}[h!]
\setlength\tabcolsep{0pt}%
\centering
\renewcommand\arraystretch{0.0}%
\begin{tabular}{|@{\hspace{5pt}}c@{\hspace{5pt}}|*{8}{>{\centering\arraybackslash}p{5.10em}}|}
\hline
  $m$ & \multicolumn{8}{c|}{$(p^f,p^d,l^f,l^a)$, algebraic degree\vphantom{\raisebox{-4pt}{a}\rule{0pt}{10pt}}}
\\\hline
\raisebox{2.7em}{$1$}
& \input{fig_m1_pf0_pd0_lf2_la0}\rule{0pt}{5.6em}
&\input{fig_m1_pf1_pd0_lf1_la1-1}
& \input{fig_m1_pf1_pd0_lf0_la3-3}
&\input{fig_m1_pf2_pd0_lf1_la0}
&\input{fig_m1_pf2_pd0_lf0_la2-2-0}
&\input{fig_m1_pf2_pd0_lf0_la2-1-1}
&\input{fig_m1_pf3_pd0_lf0_la1-1-0-0}
&\input{fig_m1_pf4_pd0_lf0_la0}\\
\hline
\end{tabular}
\caption{Single-view minimal PLPs with their associated degree. }
 \label{tab:m1}
\end{table*}

Since dependent points do not appear in the equation, there is an infinite family of possible problems for each of these PLPs with at least two free points. However, we still have to see that these are indeed all problems, i.e., that we can assume without loss of generality that adjacent lines are attached to free points.
\begin{lemma}  \label{lem:2or3freePts}
\begin{enumerate}
\item In the case of two free points, any two distinct points allow to define the other points of the arrangement as dependent points.
    \item In the case of three free points, it holds that for any point of the arrangement we can find two further points such that these three points allow to define any other point of the arrangement as a dependent point.   
\end{enumerate}
\end{lemma}
\begin{proof}
    The first statement follows directly from the fact that all points lie on the line spanned by the two free points and any two distinct points on a line define it uniquely.
    
    For the other statement let $X$ be any point. Assume first that $X$ lies on a line spanned by two of the three free points. In that case, we can take any other point on this line and the third free point since we can then construct all three of the original free points as dependent points and thus the statement follows. 
    Otherwise, we know that the line used to define $X$ does not contain two of the free points but the arrangement contains at least one such line since the definition of the first dependent point requires a line through two free points. Since all these lines sit inside the span of the three free points (i.e., a plane), they all intersect. Let $Y$ be the intersection point of the supporting line used to define $X$ and a line $L$ containing two free points. Since we assume that all intersections of lines are present in the arrangement, so is $Y$. Now we pick as a third point one of the two originally free points on $L$. Since $X$ does not lie on a line spanned by two free points, it is impossible that all dependent points lie on one line, and thus there is a supporting line containing the last free point. This line intersects now both the line used to define $X$ and $L$, and thus we have at least two points on it and can thus construct the last originally free point as a dependent point on that line.
\end{proof}

Using the previous lemma we see that we can assume without loss of generality that all lines are attached to free points since we can simply re-choose the free points in the arrangement to be the points where lines are attached. Using this, the following result allows to reduce to the case of no dependent point.
\begin{lemma}\label{lem:singleViewReduce}
    Let $\mathcal{P}$ be a minimal PLP of degree $d$ for a single camera. Then also the PLP $\mathcal{P}'$ arising from $\mathcal{P}$ by adding a single dependent point is minimal of degree~$d$.
\end{lemma}
\begin{proof}
    Consider the $d$ solutions  $(P_1,A_1),\ldots,(P_d,A_d)$ of the problem $\mathcal{P}$. We show that each of these admits a unique extension to $\mathcal{P}'$. 
    This proves the desired result since each solution to $\mathcal{P}'$ projects onto a solution of $\mathcal{P}$. 
    
    On the given image, the dependent point lies on a line between two points $x_1$ and $x_2$.
    Those correspond to world points $X_1$ and $X_2$ in the $i$-th solution $(P_i,A_i)$.
    For generic input images, the center of the camera $P_i$ does not lie on the line $\overline{X_1X_2}$. Thus, the camera $P_i$ yields an isomorphism between the lines $\overline{X_1X_2}$ and $\overline{x_1x_2}$.
    Hence, the dependent point on $\overline{x_1x_2}$ corresponds to a unique world point on $\overline{X_1X_2}$.
\end{proof}

Putting everything together, we can now show that Table~ \ref{tab:m1} indeed shows all single-view minimal PLPs.

\begin{proof}[\textbf{Proof of Proposition \ref{prop:singleView}}]
    Considering \eqref{eq:singleViewBalanced} together with Lemma~\ref{lem:2or3freePts} yields that Table~\ref{tab:m1} lists all balanced PLPs for a single projective view, up to arbitrarily many dependent points for the five right-most PLPs.
    Due to Lemma~\ref{lem:singleViewReduce}, it is sufficient to prove minimality (and compute the degree) for the balanced PLPs without any dependent points.
    This can be done analogously to the proof of Theorem \ref{thm:main} \ref{item:thm:main:b} in Section \ref{subsec:mainProof}.
    A list with the respective normalizations can be found in Table \ref{tab:m1norm}. From the chosen normalizations it follows directly that the reconstruction problem becomes a linear problem since all appearing equations are linear.
\end{proof}

\begin{table*}[htb]
    \centering
    \begin{tabular}{|>{\centering\arraybackslash}m{5em}|>{\centering\arraybackslash}m{7em}|m{8em}||>{\centering\arraybackslash}m{5em}|>{\centering\arraybackslash}m{7em}|m{8em}|}
\hline
  Problem & Arrangement& \multicolumn{1}{c||}{Camera} & Problem & Arrangement&  \multicolumn{1}{c|}{Camera}
\\
\hline

\input{fig_m1_pf0_pd0_lf2_la0}\rule{0pt}{5.6em}
&$\lbrace \overline{e_1e_2}\rbrace \times\lbrace \overline{e_3e_4}\rbrace $&
$\begin{bmatrix}
      c_1 & c_2 &c_3 &c_4\\
      1 & -1 &1 & -1\\
     1& 1& 1& 1
    \end{bmatrix}$
&

\input{fig_m1_pf1_pd0_lf0_la3-3}\rule{0pt}{5.6em}
&$\lbrace e_1\rbrace \times\lbrace \overline{e_1e_2}\rbrace \times\lbrace \overline{e_1e_3}\rbrace \times\lbrace \overline{e_1e_4}\rbrace $&
$\begin{bmatrix}
      c_1 & c_2 &c_3 &c_4\\
      c_5 & 1 &1 & -1\\
     1& 1& 1& 1
    \end{bmatrix}$
\\
\hline
\input{fig_m1_pf1_pd0_lf1_la1-1}\rule{0pt}{5.6em}& $\lbrace e_1\rbrace \times\lbrace\overline{e_1e_2}\rbrace \times\lbrace \overline{e_3e_4}\rbrace $&
$\begin{bmatrix}
      c_1 & c_2 &c_3 &c_4\\
      c_5 & 1 &1 & -1\\
     1& 1& 1& 1
    \end{bmatrix}$&\input{fig_m1_pf2_pd0_lf0_la2-1-1}\rule{0pt}{5.6em}
&$\lbrace e_1\rbrace \times\lbrace e_2\rbrace \times\lbrace \overline{e_1e_3}\rbrace \times\lbrace \overline{e_2e_4}\rbrace $ &$\begin{bmatrix}
      c_1 & c_2 &c_3 &c_4\\
      c_5 & c_6 &1 & -1\\
     1& 1& 1& 1
    \end{bmatrix}$
\\
\hline
\input{fig_m1_pf2_pd0_lf0_la2-2-0}\rule{0pt}{5.6em}&$\lbrace e_1\rbrace \times\lbrace e_2\rbrace \times\lbrace \overline{e_1e_3}\rbrace \times\lbrace \overline{e_1e_4}\rbrace $
& $\begin{bmatrix}
      c_1 & c_2 &c_3 &c_4\\
      c_5 & c_6 &1 & -1\\
     1& 1& 1& 1
    \end{bmatrix}$
    &
\input{fig_m1_pf2_pd0_lf1_la0}\rule{0pt}{5.6em}
&$\lbrace e_1\rbrace \times\lbrace e_2\rbrace \times\lbrace \overline{e_3e_4}\rbrace $
& $\begin{bmatrix}
      c_1 & c_2 &c_3 &c_4\\
      c_5 & c_6 &1 & -1\\
     1& 1& 1& 1
    \end{bmatrix}$\\
\hline
\input{fig_m1_pf3_pd0_lf0_la1-1-0-0}\rule{0pt}{5.6em}&
$\lbrace e_1\rbrace \times\lbrace e_2\rbrace \times\lbrace e_3\rbrace \times\lbrace \overline{e_1e_4}\rbrace $
&$\begin{bmatrix}
      c_1 & c_2 &c_3 &c_4\\
      c_5 & c_6 &c_7 & 1\\
     1& 1& 1& 1
    \end{bmatrix}$
    &
\input{fig_m1_pf4_pd0_lf0_la0}\rule{0pt}{5.6em}&
$\lbrace e_1\rbrace \times\lbrace e_2\rbrace \times\lbrace e_3\rbrace \times\lbrace e_4\rbrace $&
$\begin{bmatrix}
      c_1 & c_2 &c_3 &c_4\\
      c_5 & c_6 &c_7 & c_8\\
     1& 1& 1& 1
    \end{bmatrix}$
\\
\hline

\end{tabular}
\centering
\captionsetup{justification=centering}
\caption{List of normalizations for single-view PLPs }

\label{tab:m1norm}
\end{table*}

\begin{remark}
    The two minimal PLPs from Theorem \ref{thm:main} \ref{item:thm:main:b} that work for arbitrarily many views are special cases of the right-most PLP in Table~\ref{tab:m1}.
    Note that, differently from the normalizations in the proof above, the normalizations in the proof of Theorem \ref{thm:main} \ref{item:thm:main:b} did not affect the camera parameters (only the point arrangements).
    Therefore, the two PLPs from Theorem \ref{thm:main} \ref{item:thm:main:b} are minimal for an arbitrary number of cameras as each camera can be registered independently using the same world coordinate system.
    \hfill $\diamondsuit$
\end{remark}

\section{Minimal PLPs}\label{sm:minimal}
\ifdraft{\textcolor{red}{Tables will be generated when removing \texttt{draft} from class options.}}{\begin{table*}[h!]
\setlength\tabcolsep{0pt}%
\centering
\renewcommand\arraystretch{0.0}%
\begin{tabular}{|@{\hspace{5pt}}c@{\hspace{5pt}}|*{9}{>{\centering\arraybackslash}p{5.10em}}|}
\hline
  $m$ & \multicolumn{9}{c|}{$(p^f,p^d,l^f,l^a)$, algebraic degree\vphantom{\raisebox{-4pt}{a}\rule{0pt}{10pt}}}
\\\hline
& \input{figures/fig_m3_pf0_pd0_lf9_la0.tex}\rule{0pt}{5.6em}
& \input{figures/fig_m3_pf1_pd0_lf4_la7-7.tex}
& \input{figures/fig_m3_pf1_pd0_lf5_la5-5.tex}
& \input{figures/fig_m3_pf1_pd0_lf6_la3-3.tex}
& \input{figures/fig_m3_pf1_pd0_lf7_la1-1.tex}
& \input{figures/fig_m3_pf2_pd0_lf0_la12-6-6.tex}
& \input{figures/fig_m3_pf2_pd0_lf1_la10-6-4.tex}
& \input{figures/fig_m3_pf2_pd0_lf1_la10-5-5.tex}
& \input{figures/fig_m3_pf2_pd0_lf2_la8-6-2.tex}
\\
& \input{figures/fig_m3_pf2_pd0_lf2_la8-5-3.tex}
& \input{figures/fig_m3_pf2_pd0_lf2_la8-4-4.tex}
& \input{figures/fig_m3_pf2_pd0_lf3_la6-6-0.tex}
& \input{figures/fig_m3_pf2_pd0_lf3_la6-5-1.tex}
& \input{figures/fig_m3_pf2_pd0_lf3_la6-4-2.tex}
& \input{figures/fig_m3_pf2_pd0_lf3_la6-3-3.tex}
& \input{figures/fig_m3_pf2_pd0_lf4_la4-4-0.tex}
& \input{figures/fig_m3_pf2_pd0_lf4_la4-3-1.tex}
& \input{figures/fig_m3_pf2_pd0_lf4_la4-2-2.tex}
\\
& \input{figures/fig_m3_pf2_pd0_lf5_la2-2-0.tex}
& \input{figures/fig_m3_pf2_pd0_lf5_la2-1-1.tex}
& \input{figures/fig_m3_pf2_pd0_lf6_la0-0-0.tex}
& \input{figures/fig_m3_pf3_pd0_lf0_la9-5-4-0.tex}
& \input{figures/fig_m3_pf3_pd0_lf0_la9-5-3-1.tex}
& \input{figures/fig_m3_pf3_pd0_lf0_la9-5-2-2.tex}
& \input{figures/fig_m3_pf3_pd0_lf0_la9-4-4-1.tex}
& \input{figures/fig_m3_pf3_pd0_lf0_la9-4-3-2.tex}
& \input{figures/fig_m3_pf3_pd0_lf0_la9-3-3-3.tex}
\\
& \input{figures/fig_m3_pf3_pd0_lf1_la7-5-2-0.tex}
& \input{figures/fig_m3_pf3_pd0_lf1_la7-4-3-0.tex}
& \input{figures/fig_m3_pf3_pd0_lf1_la7-5-1-1.tex}
& \input{figures/fig_m3_pf3_pd0_lf1_la7-4-2-1.tex}
& \input{figures/fig_m3_pf3_pd0_lf1_la7-3-3-1.tex}
& \input{figures/fig_m3_pf3_pd0_lf1_la7-3-2-2.tex}
& \input{figures/fig_m3_pf3_pd0_lf2_la5-5-0-0.tex}
& \input{figures/fig_m3_pf3_pd0_lf2_la5-4-1-0.tex}
& \input{figures/fig_m3_pf3_pd0_lf2_la5-3-2-0.tex}
\\
& \input{figures/fig_m3_pf3_pd0_lf2_la5-3-1-1.tex}
& \input{figures/fig_m3_pf3_pd0_lf2_la5-2-2-1.tex}
& \input{figures/fig_m3_pf3_pd0_lf3_la3-3-0-0.tex}
& \input{figures/fig_m3_pf3_pd0_lf3_la3-2-1-0.tex}
& \input{figures/fig_m3_pf3_pd0_lf3_la3-1-1-1.tex}
& \input{figures/fig_m3_pf3_pd0_lf4_la1-1-0-0.tex}
& \input{figures/fig_m3_pf2_pd1-1.2_lf0_la10-6-4-0.tex}
& \input{figures/fig_m3_pf2_pd1-1.2_lf0_la10-5-5-0.tex}
& \input{figures/fig_m3_pf2_pd1-1.2_lf0_la10-6-3-1.tex}
\\
\raisebox{3em}{$3$}
& \input{figures/fig_m3_pf2_pd1-1.2_lf0_la10-6-2-2.tex}
& \input{figures/fig_m3_pf2_pd1-1.2_lf0_la10-5-4-1.tex}
& \input{figures/fig_m3_pf2_pd1-1.2_lf0_la10-5-3-2.tex}
& \input{figures/fig_m3_pf2_pd1-1.2_lf0_la10-4-4-2.tex}
& \input{figures/fig_m3_pf2_pd1-1.2_lf0_la10-4-3-3.tex}
& \input{figures/fig_m3_pf2_pd1-1.2_lf1_la8-6-2-0.tex}
& \input{figures/fig_m3_pf2_pd1-1.2_lf1_la8-5-3-0.tex}
& \input{figures/fig_m3_pf2_pd1-1.2_lf1_la8-4-4-0.tex}
& \input{figures/fig_m3_pf2_pd1-1.2_lf1_la8-6-1-1.tex}
\\
& \input{figures/fig_m3_pf2_pd1-1.2_lf1_la8-5-2-1.tex}
& \input{figures/fig_m3_pf2_pd1-1.2_lf1_la8-4-3-1.tex}
& \input{figures/fig_m3_pf2_pd1-1.2_lf1_la8-4-2-2.tex}
& \input{figures/fig_m3_pf2_pd1-1.2_lf1_la8-3-3-2.tex}
& \input{figures/fig_m3_pf2_pd1-1.2_lf2_la6-6-0-0.tex}
& \input{figures/fig_m3_pf2_pd1-1.2_lf2_la6-5-1-0.tex}
& \input{figures/fig_m3_pf2_pd1-1.2_lf2_la6-4-2-0.tex}
& \input{figures/fig_m3_pf2_pd1-1.2_lf2_la6-3-3-0.tex}
& \input{figures/fig_m3_pf2_pd1-1.2_lf2_la6-4-1-1.tex}
\\
& \input{figures/fig_m3_pf2_pd1-1.2_lf2_la6-3-2-1.tex}
& \input{figures/fig_m3_pf2_pd1-1.2_lf2_la6-2-2-2.tex}
& \input{figures/fig_m3_pf2_pd1-1.2_lf3_la4-4-0-0.tex}
& \input{figures/fig_m3_pf2_pd1-1.2_lf3_la4-3-1-0.tex}
& \input{figures/fig_m3_pf2_pd1-1.2_lf3_la4-2-2-0.tex}
& \input{figures/fig_m3_pf2_pd1-1.2_lf3_la4-2-1-1.tex}
& \input{figures/fig_m3_pf2_pd1-1.2_lf4_la2-2-0-0.tex}
& \input{figures/fig_m3_pf2_pd1-1.2_lf4_la2-1-1-0.tex}
& \input{figures/fig_m3_pf2_pd1-1.2_lf5_la0-0-0-0.tex}
\\
& \input{figures/fig_m3_pf4_pd0_lf0_la6-4-2-0-0.tex}
& \input{figures/fig_m3_pf4_pd0_lf0_la6-3-3-0-0.tex}
& \input{figures/fig_m3_pf4_pd0_lf0_la6-4-1-1-0.tex}
& \input{figures/fig_m3_pf4_pd0_lf0_la6-3-2-1-0.tex}
& \input{figures/fig_m3_pf4_pd0_lf0_la6-2-2-2-0.tex}
& \input{figures/fig_m3_pf4_pd0_lf0_la6-3-1-1-1.tex}
& \input{figures/fig_m3_pf4_pd0_lf0_la6-2-2-1-1.tex}
& \input{figures/fig_m3_pf4_pd0_lf1_la4-3-1-0-0.tex}
& \input{figures/fig_m3_pf4_pd0_lf1_la4-2-2-0-0.tex}
\\
& \input{figures/fig_m3_pf4_pd0_lf1_la4-2-1-1-0.tex}
& \input{figures/fig_m3_pf4_pd0_lf1_la4-1-1-1-1.tex}
& \input{figures/fig_m3_pf4_pd0_lf2_la2-2-0-0-0.tex}
& \input{figures/fig_m3_pf4_pd0_lf2_la2-1-1-0-0.tex}
& \input{figures/fig_m3_pf4_pd0_lf3_la0-0-0-0-0.tex}
& \input{figures/fig_m3_pf3_pd1-1.2_lf0_la7-4-0-3-0.tex}
& \input{figures/fig_m3_pf3_pd1-1.2_lf0_la7-3-1-3-0.tex}
& \input{figures/fig_m3_pf3_pd1-1.2_lf0_la7-2-2-3-0.tex}
& \input{figures/fig_m3_pf3_pd1-1.2_lf0_la7-2-1-3-1.tex}
\\
& \input{figures/fig_m3_pf3_pd1-1.2_lf0_la7-5-0-2-0.tex}
& \input{figures/fig_m3_pf3_pd1-1.2_lf0_la7-4-1-2-0.tex}
& \input{figures/fig_m3_pf3_pd1-1.2_lf0_la7-3-2-2-0.tex}
& \input{figures/fig_m3_pf3_pd1-1.2_lf0_la7-3-1-2-1.tex}
& \input{figures/fig_m3_pf3_pd1-1.2_lf0_la7-2-2-2-1.tex}
& \input{figures/fig_m3_pf3_pd1-1.2_lf0_la7-5-1-1-0.tex}
& \input{figures/fig_m3_pf3_pd1-1.2_lf0_la7-4-2-1-0.tex}
& \input{figures/fig_m3_pf3_pd1-1.2_lf0_la7-3-3-1-0.tex}
& \input{figures/fig_m3_pf3_pd1-1.2_lf0_la7-4-1-1-1.tex}
\\\hline
\end{tabular}%
\caption{Minimal problems with their associated degree.}%
\label{tab:min1}
\end{table*}

\begin{table*}[h!]
\setlength\tabcolsep{0pt}%
\centering
\renewcommand\arraystretch{0.0}%
\begin{tabular}{|@{\hspace{5pt}}c@{\hspace{5pt}}|*{9}{>{\centering\arraybackslash}p{5.10em}}|}
\hline
  $m$ & \multicolumn{9}{c|}{$(p^f,p^d,l^f,l^a)$, algebraic degree\vphantom{\raisebox{-4pt}{a}\rule{0pt}{10pt}}}
\\\hline
& \input{figures/fig_m3_pf3_pd1-1.2_lf0_la7-3-2-1-1.tex}\rule{0pt}{5.6em}
& \input{figures/fig_m3_pf3_pd1-1.2_lf0_la7-2-2-1-2.tex}
& \input{figures/fig_m3_pf3_pd1-1.2_lf0_la7-5-2-0-0.tex}
& \input{figures/fig_m3_pf3_pd1-1.2_lf0_la7-4-3-0-0.tex}
& \input{figures/fig_m3_pf3_pd1-1.2_lf0_la7-5-1-0-1.tex}
& \input{figures/fig_m3_pf3_pd1-1.2_lf0_la7-4-2-0-1.tex}
& \input{figures/fig_m3_pf3_pd1-1.2_lf0_la7-3-3-0-1.tex}
& \input{figures/fig_m3_pf3_pd1-1.2_lf0_la7-3-2-0-2.tex}
& \input{figures/fig_m3_pf3_pd1-1.2_lf1_la5-2-0-3-0.tex}
\\
& \input{figures/fig_m3_pf3_pd1-1.2_lf1_la5-1-1-3-0.tex}
& \input{figures/fig_m3_pf3_pd1-1.2_lf1_la5-3-0-2-0.tex}
& \input{figures/fig_m3_pf3_pd1-1.2_lf1_la5-2-1-2-0.tex}
& \input{figures/fig_m3_pf3_pd1-1.2_lf1_la5-1-1-2-1.tex}
& \input{figures/fig_m3_pf3_pd1-1.2_lf1_la5-4-0-1-0.tex}
& \input{figures/fig_m3_pf3_pd1-1.2_lf1_la5-3-1-1-0.tex}
& \input{figures/fig_m3_pf3_pd1-1.2_lf1_la5-2-2-1-0.tex}
& \input{figures/fig_m3_pf3_pd1-1.2_lf1_la5-2-1-1-1.tex}
& \input{figures/fig_m3_pf3_pd1-1.2_lf1_la5-4-1-0-0.tex}
\\
& \input{figures/fig_m3_pf3_pd1-1.2_lf1_la5-3-2-0-0.tex}
& \input{figures/fig_m3_pf3_pd1-1.2_lf1_la5-3-1-0-1.tex}
& \input{figures/fig_m3_pf3_pd1-1.2_lf1_la5-2-2-0-1.tex}
& \input{figures/fig_m3_pf3_pd1-1.2_lf2_la3-0-0-3-0.tex}
& \input{figures/fig_m3_pf3_pd1-1.2_lf2_la3-1-0-2-0.tex}
& \input{figures/fig_m3_pf3_pd1-1.2_lf2_la3-2-0-1-0.tex}
& \input{figures/fig_m3_pf3_pd1-1.2_lf2_la3-1-1-1-0.tex}
& \input{figures/fig_m3_pf3_pd1-1.2_lf2_la3-3-0-0-0.tex}
& \input{figures/fig_m3_pf3_pd1-1.2_lf2_la3-2-1-0-0.tex}
\\
& \input{figures/fig_m3_pf3_pd1-1.2_lf2_la3-1-1-0-1.tex}
& \input{figures/fig_m3_pf5_pd0_lf0_la3-3-0-0-0-0.tex}
& \input{figures/fig_m3_pf5_pd0_lf0_la3-2-1-0-0-0.tex}
& \input{figures/fig_m3_pf5_pd0_lf0_la3-1-1-1-0-0.tex}
& \input{figures/fig_m3_pf5_pd0_lf1_la1-1-0-0-0-0.tex}
& \input{figures/fig_m3_pf4_pd1-1.2_lf0_la4-0-0-2-2-0.tex}
& \input{figures/fig_m3_pf4_pd1-1.2_lf0_la4-1-0-2-1-0.tex}
& \input{figures/fig_m3_pf4_pd1-1.2_lf0_la4-2-0-2-0-0.tex}
& \input{figures/fig_m3_pf4_pd1-1.2_lf0_la4-2-0-1-1-0.tex}
\\
\raisebox{0em}{$3$}
& \input{figures/fig_m3_pf4_pd1-1.2_lf0_la4-1-1-2-0-0.tex}
& \input{figures/fig_m3_pf4_pd1-1.2_lf0_la4-1-1-1-1-0.tex}
& \input{figures/fig_m3_pf4_pd1-1.2_lf0_la4-3-0-1-0-0.tex}
& \input{figures/fig_m3_pf4_pd1-1.2_lf0_la4-2-1-1-0-0.tex}
& \input{figures/fig_m3_pf4_pd1-1.2_lf0_la4-1-1-1-0-1.tex}
& \input{figures/fig_m3_pf4_pd1-1.2_lf0_la4-4-0-0-0-0.tex}
& \input{figures/fig_m3_pf4_pd1-1.2_lf0_la4-3-1-0-0-0.tex}
& \input{figures/fig_m3_pf4_pd1-1.2_lf0_la4-2-2-0-0-0.tex}
& \input{figures/fig_m3_pf4_pd1-1.2_lf0_la4-2-1-0-0-1.tex}
\\
& \input{figures/fig_m3_pf4_pd1-1.2_lf1_la2-0-0-1-1-0.tex}
& \input{figures/fig_m3_pf4_pd1-1.2_lf1_la2-1-0-1-0-0.tex}
& \input{figures/fig_m3_pf4_pd1-1.2_lf1_la2-2-0-0-0-0.tex}
& \input{figures/fig_m3_pf4_pd1-1.2_lf1_la2-1-1-0-0-0.tex}
& \input{figures/fig_m3_pf4_pd1-1.2_lf2_la0-0-0-0-0-0.tex}
& \input{figures/fig_m3_pf3_pd2-1.2-1.3_lf0_la5-0-3-2-0-0.tex}
& \input{figures/fig_m3_pf3_pd2-1.2-1.3_lf0_la5-0-3-1-0-1.tex}
& \input{figures/fig_m3_pf3_pd2-1.2-1.3_lf0_la5-0-2-2-1-0.tex}
& \input{figures/fig_m3_pf3_pd2-1.2-1.3_lf0_la5-0-2-1-1-1.tex}
\\
& \input{figures/fig_m3_pf3_pd2-1.2-1.3_lf0_la5-0-3-1-1-0.tex}
& \input{figures/fig_m3_pf3_pd2-1.2-1.3_lf0_la5-0-2-1-2-0.tex}
& \input{figures/fig_m3_pf3_pd2-1.2-1.3_lf0_la5-0-3-0-2-0.tex}
& \input{figures/fig_m3_pf3_pd2-1.2-1.3_lf0_la5-1-2-2-0-0.tex}
& \input{figures/fig_m3_pf3_pd2-1.2-1.3_lf0_la5-1-2-1-0-1.tex}
& \input{figures/fig_m3_pf3_pd2-1.2-1.3_lf0_la5-1-1-1-1-1.tex}
& \input{figures/fig_m3_pf3_pd2-1.2-1.3_lf0_la5-1-3-1-0-0.tex}
& \input{figures/fig_m3_pf3_pd2-1.2-1.3_lf0_la5-1-2-1-1-0.tex}
& \input{figures/fig_m3_pf3_pd2-1.2-1.3_lf0_la5-1-3-0-1-0.tex}
\\
& \input{figures/fig_m3_pf3_pd2-1.2-1.3_lf0_la5-1-2-0-2-0.tex}
& \input{figures/fig_m3_pf3_pd2-1.2-1.3_lf0_la5-2-2-1-0-0.tex}
& \input{figures/fig_m3_pf3_pd2-1.2-1.3_lf0_la5-2-1-1-1-0.tex}
& \input{figures/fig_m3_pf3_pd2-1.2-1.3_lf0_la5-2-3-0-0-0.tex}
& \input{figures/fig_m3_pf3_pd2-1.2-1.3_lf0_la5-2-2-0-1-0.tex}
& \input{figures/fig_m3_pf3_pd2-1.2-1.3_lf0_la5-3-1-1-0-0.tex}
& \input{figures/fig_m3_pf3_pd2-1.2-1.3_lf0_la5-3-2-0-0-0.tex}
& \input{figures/fig_m3_pf3_pd2-1.2-1.3_lf0_la5-3-1-0-1-0.tex}
& \input{figures/fig_m3_pf6_pd0_lf0_la0-0-0-0-0-0-0.tex}
\\
& \input{figures/fig_m3_pf5_pd1-1.2_lf0_la1-0-0-1-0-0-0.tex}
& \input{figures/fig_m3_pf5_pd1-1.2_lf0_la1-1-0-0-0-0-0.tex}
& \input{figures/fig_m3_pf4_pd2-1.2-1.3_lf0_la2-0-1-1-0-0-0.tex}
& \input{figures/fig_m3_pf4_pd2-1.2-1.3_lf0_la2-0-2-0-0-0-0.tex}
& \input{figures/fig_m3_pf4_pd2-1.2-1.3_lf0_la2-0-1-0-0-1-0.tex}
& \input{figures/fig_m3_pf4_pd2-1.2-1.3_lf0_la2-1-1-0-0-0-0.tex}
& \input{figures/fig_m3_pf4_pd2-1.2-1.3_lf0_la2-2-0-0-0-0-0.tex}
& \input{figures/fig_m3_pf4_pd2-1.2-3.4_lf0_la2-1-0-1-0-0-0.tex}
& \input{figures/fig_m3_pf4_pd2-1.2-3.4_lf0_la2-2-0-0-0-0-0.tex}
\\
& \input{figures/fig_m3_pf4_pd2-1.2-3.4_lf0_la2-1-1-0-0-0-0.tex}
& \input{figures/fig_m3_pf4_pd2-1.2-3.4_lf1_la0-0-0-0-0-0-0.tex}
&
&
&
&
&
&
&
\\\hline
\end{tabular}%
\caption{Minimal problems with their associated degree.}%
\label{tab:min2}
\end{table*}

\begin{table*}[h!]
\setlength\tabcolsep{0pt}%
\centering
\renewcommand\arraystretch{0.0}%
\begin{tabular}{|@{\hspace{5pt}}c@{\hspace{5pt}}|*{9}{>{\centering\arraybackslash}p{5.10em}}|}
\hline
  $m$ & \multicolumn{9}{c|}{$(p^f,p^d,l^f,l^a)$, algebraic degree\vphantom{\raisebox{-4pt}{a}\rule{0pt}{10pt}}}
\\\hline
& \input{figures/fig_m4_pf1_pd0_lf3_la6-6.tex}\rule{0pt}{5.6em}
& \input{figures/fig_m4_pf1_pd0_lf4_la4-4.tex}
& \input{figures/fig_m4_pf1_pd0_lf5_la2-2.tex}
& \input{figures/fig_m4_pf1_pd0_lf6_la0-0.tex}
& \input{figures/fig_m4_pf3_pd0_lf0_la7-4-3-0.tex}
& \input{figures/fig_m4_pf3_pd0_lf0_la7-4-2-1.tex}
& \input{figures/fig_m4_pf3_pd0_lf0_la7-3-3-1.tex}
& \input{figures/fig_m4_pf3_pd0_lf0_la7-3-2-2.tex}
& \input{figures/fig_m4_pf3_pd0_lf1_la5-3-2-0.tex}
\\
& \input{figures/fig_m4_pf3_pd0_lf1_la5-3-1-1.tex}
& \input{figures/fig_m4_pf3_pd0_lf1_la5-2-2-1.tex}
& \input{figures/fig_m4_pf3_pd0_lf2_la3-3-0-0.tex}
& \input{figures/fig_m4_pf3_pd0_lf2_la3-2-1-0.tex}
& \input{figures/fig_m4_pf3_pd0_lf2_la3-1-1-1.tex}
& \input{figures/fig_m4_pf3_pd0_lf3_la1-1-0-0.tex}
& \input{figures/fig_m4_pf2_pd1-1.2_lf0_la8-5-3-0.tex}
& \input{figures/fig_m4_pf2_pd1-1.2_lf0_la8-4-4-0.tex}
& \input{figures/fig_m4_pf2_pd1-1.2_lf0_la8-5-2-1.tex}
\\
\raisebox{0em}{$4$}
& \input{figures/fig_m4_pf2_pd1-1.2_lf0_la8-4-3-1.tex}
& \input{figures/fig_m4_pf2_pd1-1.2_lf0_la8-4-2-2.tex}
& \input{figures/fig_m4_pf2_pd1-1.2_lf0_la8-3-3-2.tex}
& \input{figures/fig_m4_pf2_pd1-1.2_lf1_la6-4-2-0.tex}
& \input{figures/fig_m4_pf2_pd1-1.2_lf1_la6-3-3-0.tex}
& \input{figures/fig_m4_pf2_pd1-1.2_lf1_la6-4-1-1.tex}
& \input{figures/fig_m4_pf2_pd1-1.2_lf1_la6-3-2-1.tex}
& \input{figures/fig_m4_pf2_pd1-1.2_lf1_la6-2-2-2.tex}
& \input{figures/fig_m4_pf2_pd1-1.2_lf2_la4-4-0-0.tex}
\\
& \input{figures/fig_m4_pf2_pd1-1.2_lf2_la4-3-1-0.tex}
& \input{figures/fig_m4_pf2_pd1-1.2_lf2_la4-2-2-0.tex}
& \input{figures/fig_m4_pf2_pd1-1.2_lf2_la4-2-1-1.tex}
& \input{figures/fig_m4_pf2_pd1-1.2_lf3_la2-2-0-0.tex}
& \input{figures/fig_m4_pf2_pd1-1.2_lf3_la2-1-1-0.tex}
& \input{figures/fig_m4_pf2_pd1-1.2_lf4_la0-0-0-0.tex}
& \input{figures/fig_m4_pf5_pd0_lf0_la2-2-0-0-0-0.tex}
& \input{figures/fig_m4_pf5_pd0_lf0_la2-1-1-0-0-0.tex}
& \input{figures/fig_m4_pf5_pd0_lf1_la0-0-0-0-0-0.tex}
\\
& \input{figures/fig_m4_pf4_pd1-1.2_lf0_la3-1-0-1-1-0.tex}
& \input{figures/fig_m4_pf4_pd1-1.2_lf0_la3-2-0-1-0-0.tex}
& \input{figures/fig_m4_pf4_pd1-1.2_lf0_la3-1-1-1-0-0.tex}
& \input{figures/fig_m4_pf4_pd1-1.2_lf0_la3-3-0-0-0-0.tex}
& \input{figures/fig_m4_pf4_pd1-1.2_lf0_la3-2-1-0-0-0.tex}
& \input{figures/fig_m4_pf4_pd1-1.2_lf0_la3-1-1-0-0-1.tex}
& \input{figures/fig_m4_pf4_pd1-1.2_lf1_la1-1-0-0-0-0.tex}
& \input{figures/fig_m4_pf3_pd2-1.2-1.3_lf0_la4-0-2-2-0-0.tex}
& \input{figures/fig_m4_pf3_pd2-1.2-1.3_lf0_la4-0-2-1-0-1.tex}
\\
& \input{figures/fig_m4_pf3_pd2-1.2-1.3_lf0_la4-0-1-1-1-1.tex}
& \input{figures/fig_m4_pf3_pd2-1.2-1.3_lf0_la4-0-2-1-1-0.tex}
& \input{figures/fig_m4_pf3_pd2-1.2-1.3_lf0_la4-0-2-0-2-0.tex}
& \input{figures/fig_m4_pf3_pd2-1.2-1.3_lf0_la4-1-2-1-0-0.tex}
& \input{figures/fig_m4_pf3_pd2-1.2-1.3_lf0_la4-1-1-1-1-0.tex}
& \input{figures/fig_m4_pf3_pd2-1.2-1.3_lf0_la4-1-2-0-1-0.tex}
& \input{figures/fig_m4_pf3_pd2-1.2-1.3_lf0_la4-2-1-1-0-0.tex}
& \input{figures/fig_m4_pf3_pd2-1.2-1.3_lf0_la4-2-2-0-0-0.tex}
& \input{figures/fig_m4_pf3_pd2-1.2-1.3_lf0_la4-2-1-0-1-0.tex}
\\\hline
\end{tabular}%
\caption{Minimal problems with their associated degree.}%
\label{tab:min3}
\end{table*}

\begin{table*}[h!]
\setlength\tabcolsep{0pt}%
\centering
\renewcommand\arraystretch{0.0}%
\begin{tabular}{|@{\hspace{5pt}}c@{\hspace{5pt}}|*{9}{>{\centering\arraybackslash}p{5.10em}}|}
\hline
  $m$ & \multicolumn{9}{c|}{$(p^f,p^d,l^f,l^a)$, algebraic degree\vphantom{\raisebox{-4pt}{a}\rule{0pt}{10pt}}}
\\\hline
& \input{figures/fig_m5_pf1_pd0_lf3_la5-5.tex}\rule{0pt}{5.6em}
& \input{figures/fig_m5_pf1_pd0_lf4_la3-3.tex}
& \input{figures/fig_m5_pf1_pd0_lf5_la1-1.tex}
& \input{figures/fig_m5_pf4_pd0_lf0_la4-2-2-0-0.tex}
& \input{figures/fig_m5_pf4_pd0_lf0_la4-2-1-1-0.tex}
& \input{figures/fig_m5_pf4_pd0_lf0_la4-1-1-1-1.tex}
& \input{figures/fig_m5_pf4_pd0_lf1_la2-1-1-0-0.tex}
& \input{figures/fig_m5_pf4_pd0_lf2_la0-0-0-0-0.tex}
& \input{figures/fig_m5_pf3_pd1-1.2_lf0_la5-3-0-2-0.tex}
\\
\raisebox{3em}{$5$}
& \input{figures/fig_m5_pf3_pd1-1.2_lf0_la5-2-1-2-0.tex}
& \input{figures/fig_m5_pf3_pd1-1.2_lf0_la5-1-1-2-1.tex}
& \input{figures/fig_m5_pf3_pd1-1.2_lf0_la5-3-1-1-0.tex}
& \input{figures/fig_m5_pf3_pd1-1.2_lf0_la5-2-2-1-0.tex}
& \input{figures/fig_m5_pf3_pd1-1.2_lf0_la5-2-1-1-1.tex}
& \input{figures/fig_m5_pf3_pd1-1.2_lf0_la5-3-2-0-0.tex}
& \input{figures/fig_m5_pf3_pd1-1.2_lf0_la5-3-1-0-1.tex}
& \input{figures/fig_m5_pf3_pd1-1.2_lf0_la5-2-2-0-1.tex}
& \input{figures/fig_m5_pf3_pd1-1.2_lf1_la3-2-0-1-0.tex}
\\
& \input{figures/fig_m5_pf3_pd1-1.2_lf1_la3-1-1-1-0.tex}
& \input{figures/fig_m5_pf3_pd1-1.2_lf1_la3-2-1-0-0.tex}
& \input{figures/fig_m5_pf3_pd1-1.2_lf1_la3-1-1-0-1.tex}
&
&
&
&
&
&
\\
\hline
\raisebox{0em}{$6$}
& \input{figures/fig_m6_pf3_pd0_lf0_la6-3-3-0.tex}\rule{0pt}{5.6em}
& \input{figures/fig_m6_pf3_pd0_lf0_la6-3-2-1.tex}
& \input{figures/fig_m6_pf3_pd0_lf0_la6-2-2-2.tex}
& \input{figures/fig_m6_pf3_pd0_lf1_la4-2-2-0.tex}
& \input{figures/fig_m6_pf3_pd0_lf1_la4-2-1-1.tex}
& \input{figures/fig_m6_pf3_pd0_lf2_la2-2-0-0.tex}
& \input{figures/fig_m6_pf3_pd0_lf2_la2-1-1-0.tex}
& \input{figures/fig_m6_pf2_pd1-1.2_lf0_la7-4-3-0.tex}
& \input{figures/fig_m6_pf2_pd1-1.2_lf0_la7-4-2-1.tex}
\\
& \input{figures/fig_m6_pf2_pd1-1.2_lf0_la7-3-3-1.tex}
& \input{figures/fig_m6_pf2_pd1-1.2_lf0_la7-3-2-2.tex}
& \input{figures/fig_m6_pf2_pd1-1.2_lf1_la5-3-2-0.tex}
& \input{figures/fig_m6_pf2_pd1-1.2_lf1_la5-3-1-1.tex}
& \input{figures/fig_m6_pf2_pd1-1.2_lf1_la5-2-2-1.tex}
& \input{figures/fig_m6_pf2_pd1-1.2_lf2_la3-3-0-0.tex}
& \input{figures/fig_m6_pf2_pd1-1.2_lf2_la3-2-1-0.tex}
& \input{figures/fig_m6_pf2_pd1-1.2_lf2_la3-1-1-1.tex}
&
\\
\hline
\raisebox{2.7em}{$7$}
& \input{figures/fig_m7_pf2_pd0_lf0_la8-4-4.tex}\rule{0pt}{5.6em}
& \input{figures/fig_m7_pf2_pd0_lf1_la6-3-3.tex}
& \input{figures/fig_m7_pf2_pd0_lf2_la4-3-1.tex}
& \input{figures/fig_m7_pf2_pd0_lf2_la4-2-2.tex}
& \input{figures/fig_m7_pf2_pd0_lf3_la2-2-0.tex}
& \input{figures/fig_m7_pf2_pd0_lf3_la2-1-1.tex}
& \input{figures/fig_m7_pf2_pd0_lf4_la0-0-0.tex}
&
&
\\
\hline
\raisebox{2.7em}{$8$}
& \input{figures/fig_m8_pf1_pd0_lf3_la4-4.tex}\rule{0pt}{5.6em}
& \input{figures/fig_m8_pf1_pd0_lf4_la2-2.tex}
& \input{figures/fig_m8_pf1_pd0_lf5_la0-0.tex}
&
&
&
&
&
&
\\
\hline
\raisebox{2.7em}{$9$}
& \input{figures/fig_m9_pf0_pd0_lf6_la0.tex}\rule{0pt}{5.6em}
&
&
&
&
&
&
&
&
\\\hline
\end{tabular}%
\caption{Minimal problems with their associated degree.}%
\label{tab:min4}
\end{table*}}

The 285 minimal PLPs from Theorem~\ref{thm:main} \ref{item:thm:main:c}  
are depicted in Tables \ref{tab:min1} -- \ref{tab:min4}.

\end{document}